\documentclass[11pt]{article}

\usepackage[final]{acl}

\usepackage{times}
\usepackage{latexsym}
\usepackage[T1]{fontenc}
\usepackage[utf8]{inputenc}
\usepackage{microtype}
\usepackage{inconsolata}
\usepackage{graphicx}
\usepackage{xcolor}
\usepackage{subcaption}
\usepackage{amsmath}
\usepackage{amsthm}
\usepackage{amssymb}
\usepackage{mathtools}
\usepackage{booktabs}
\usepackage{placeins}
\usepackage{tabularx}
\usepackage{enumitem}
\usepackage[most]{tcolorbox}

\theoremstyle{plain}
\newtheorem{theorem}{Theorem}[section]
\newtheorem{lemma}[theorem]{Lemma}
\newtheorem{proposition}[theorem]{Proposition}
\newtheorem{corollary}[theorem]{Corollary}
\theoremstyle{definition}
\newtheorem{definition}[theorem]{Definition}
\newtheorem{assumption}[theorem]{Assumption}

\title{From Fallback to Frontline: \\When Can LLMs be Superior Annotators of Human Perspectives?}

\author{
  \textbf{Hasan Amin\textsuperscript{1}},
  \textbf{Harry Yizhou Tian\textsuperscript{1}},
  \textbf{Xiaoni Duan\textsuperscript{1}},
  \textbf{Chien-Ju Ho\textsuperscript{2}},
  \textbf{Rajiv Khanna\textsuperscript{1}},
  \textbf{Ming Yin\textsuperscript{1}}
\\
\\
  \textsuperscript{1}Purdue University,
  \textsuperscript{2}Washington University in St.\ Louis
\\
\small{\{hasanamin,tian253,duan79,rajivak,mingyin\}@purdue.edu},\,\,
\small{chienju.ho@wustl.edu}
}

\begin{document}
\maketitle
\begin{abstract}
Although large language models (LLMs) are increasingly used as annotators at scale, they are typically treated as a pragmatic fallback rather than a faithful estimator of human perspectives. This work challenges that presumption. By framing perspective-taking as the estimation of a latent group-level judgment, we characterize the conditions under which modern LLMs can outperform human annotators, including in-group humans, when predicting aggregate subgroup opinions on subjective tasks, and show that these conditions are common in practice. This advantage arises from structural properties of LLMs as estimators, including low variance and reduced coupling between representation and processing biases, rather than any claim of lived experience. Our analysis identifies clear regimes where LLMs act as statistically superior frontline estimators, as well as principled limits where human judgment remains essential. These findings reposition LLMs from a cost-saving compromise to a principled tool for estimating collective human perspectives.\footnote{Code and data, including LLM annotations, are available at \href{https://github.com/shasanamin/llm-perspective-taking}{\tt github.com/shasanamin/llm-perspective-taking}.}\\[3pt]
{\color{red} \textbf{Content warning:} This paper contains examples of offensive or toxic language that some readers may find disturbing.}
\end{abstract}

\section{Introduction}
\label{sec:intro}

Capturing subjective human perspectives, particularly those of specific demographic or cultural subgroups, is a foundational challenge in NLP, with direct implications for evaluation, fairness, and safety.
Unlike factual labels, subjective judgments admit no objective ground truth, as perceptions of toxicity, offensiveness, or harm depend on lived experience, social context, and individual bias \cite{sap2019risk,sap2021annotators}.
As a result, modern annotation pipelines rely heavily on human judgments aggregated across annotators, treating crowd averages as proxies for group-level perspectives.

In practice, however, collecting high-quality, representative annotations is often difficult or infeasible.
Recruiting in-group annotators for every target subgroup is costly, slow, and sometimes impossible, especially for small, intersectional, or hard-to-reach populations \cite{davani2022dealing,sandri2023don}.
Consequently, many workflows turn to \textit{perspective-taking} (PT): asking annotators to estimate how a group would judge an item, rather than reporting their own view \cite{frenda2024perspectivist,duan2025exploring}.
Large language models (LLMs) have recently emerged as natural candidates for this role, given their ability to simulate personas, follow instructions, and draw on broad training data \cite{wilf2023think}.

Despite their growing use, a strong presumption persists:
LLM-based annotation is, at best, a scalable compromise---useful when human annotation is unavailable, but inherently inferior to ``real'' human perspectives.
This paper challenges this established view 
by reframing perspective-taking itself.
When the goal is PT, neither humans nor LLMs observe the target quantity directly.
Instead, both are attempting to estimate the same latent group-level mean: how a population would judge an item on average.
Once PT is framed as a problem of statistical estimation, rather than measurement of individual opinion, the comparison between humans and LLMs fundamentally changes.

We formalize PT as estimation of a latent group-level judgment and decompose mean-squared error into three structural components: bias, variance, and correlation.
To interpret these terms, we introduce a two-lens view of estimation error.
Representation error (the \textit{Wide Lens}) captures how well an annotator's knowledge reflects the target group, while processing and calibration error (the \textit{Clear Lens}) captures how judgments are produced given that representation.
Crucially, estimation quality also depends on variance across annotators and on the correlation between their errors.
For humans, shared social context and identity can induce high variance and strong correlations, as well as a positive \textit{coupling} between representation and processing errors, leading to super-additive bias.
For LLMs, representation and processing can arise from mechanically distinct stages (e.g., pretraining vs. post-training vs. inference-time calibration), often yielding lower variance and weaker coupling.

Guided by this framework, we evaluate humans and LLMs as estimators of subgroup-level judgments across two datasets.
We find that modern LLMs can consistently outperform human annotators---including in-group humans---at predicting aggregate subgroup judgments.
This advantage emerges most clearly in low-budget regimes, where estimation error is variance-dominated, and a \textit{single LLM} output is found to often outperform even small \textit{human crowds}.
Crucially, the advantage does not come from LLMs possessing lived experience, but from more favorable estimator properties under realistic annotation budgets.

Our analysis also reveals an unexpected empirical phenomenon: explicit `reasoning' or chain-of-thought (CoT) prompting can degrade LLM PT performance.
We find that reasoning induces a systematic \textit{criterion drift}, shifting the model from estimating empirical group judgments toward applying a rubric-based classification standard, which can partially counteract the structural advantages of LLM-based estimation.
We term this effect the \textit{reasoning paradox}, and provide a theoretical account that can explain why increased deliberation need not improve---and can even harm---aggregate perspective estimation.

Our results do not argue for replacing humans or dismissing lived experience.
Instead, they clarify \textit{when} LLMs can act as statistically superior estimators of aggregate perspective---and when humans remain indispensable, such as for
highly specific or intersectional subgroups, rare or underrepresented populations, or contexts where stakeholder legitimacy and participation are themselves essential.
By making these regimes explicit, we move beyond ideological debates toward principled, task-dependent annotation design.

This paper makes five key contributions:
\begin{itemize}[leftmargin=*,topsep=4pt]
\setlength\itemsep{0pt}
    \item We formalize perspective-taking as estimation of a latent group-level judgment, enabling principled comparison between humans and LLMs.
    \item We introduce a bias--variance--correlation framework with a two-lens decomposition and a coupling hypothesis that accounts for systematic human failure modes.
    \item We provide extensive empirical evidence that LLMs can outperform humans---including in-group humans---under realistic budgets, while identifying clear boundary conditions.
    \item We derive actionable guidance for engineering PT pipelines, including when to use LLMs, when to use humans, and how prompting, model choice, and reasoning affect outcomes.
    \item We introduce a \textit{differential perspective-taking} diagnostic that isolates group-specific sensitivity from generic annotation skill, revealing a regime where human annotators retain advantage over LLMs.
\end{itemize}

Taken together, our findings reposition LLMs in subjective NLP tasks:
not as a fallback for missing human data, but as a potentially \textit{frontline estimator} of collective human perspectives---when used carefully, validated rigorously, and deployed where their structural advantages apply.

\subsection{Related Work}
Prior studies examine PT as a mechanism for eliciting group judgments, highlighting both its benefits and its distortions \cite{frenda2024perspectivist,duan2025exploring,sandri2023don,aoyagui2025matter}. We extend this line by formalizing PT as a statistical estimation problem. 

Another line of work evaluates LLMs as annotators, typically focusing on their agreement with human labels \cite{li2025generation,movva2024annotation,calderon-etal-2025-alternative}. In contrast, we frame the problem through estimation efficiency and characterize when different estimators are preferable. Research on social reasoning shows that LLMs perform strongly on theory-of-mind benchmarks \cite{rabinowitz2018machine,huang2023chatgpt} but also exhibit systematic social biases \cite{hagendorff2023human,hu2025generative}. Our analysis decomposes how bias, variance, and correlation jointly shape PT performance.

Recent work on persona prompting investigates whether sociodemographic conditioning enables LLMs to simulate individuals or groups, with mixed empirical results \cite{sun2025sociodemographic,lutz2025prompt,orlikowski2025beyond}. Our framework offers potential structural explanation: individual simulation aims to recover the full within-group distribution, which is a more difficult target than our group mean estimation. Finally, work on pluralistic alignment focuses on matching entire opinion distributions \cite{sorensen2024position,feng2024modular,lee2024aligning}. We view accurate mean estimation as a foundational step, whose analysis naturally extends to richer distributional objectives.

\section{Perspective-Taking as Estimation}
\label{sec:setup}
We frame perspective-taking not as a \textit{measurement} task, but as the \textit{statistical estimation} of latent group-level judgment.
When asked how a group of humans would judge an item, neither humans nor LLMs directly observe the target quantity, and must infer it from incomplete experience, priors, and the elicitation protocol.
This section formalizes this view and presents testable implications.
Full derivations and proofs are provided in Appendix~\ref{sec:appendix_theory}.

\subsection{Target Quantity and Protocols}
\label{sec:target}

Let $x\in\mathcal{X}$ be an item and $g\in\mathcal{G}$ a (demographic) group with population distribution $P_g$.
Let $Y_h(x)\in[0,1]$ denote the \textit{direct} judgment of a randomly sampled group member $h\sim P_g$ on item $x$ (e.g., $x$ is toxic or not).
The group-level perspective of interest is the latent subgroup mean
\begin{equation}
\label{eq:target_main}
f^\ast(x,g)
\;\triangleq\;
\mathbb{E}_{h\sim P_g}[Y_h(x)].
\end{equation}

\paragraph{Direct annotation vs.\ PT.}
Direct annotation measures $Y_h(x)$ for sampled individuals and aggregates those measurements.
PT instead elicits an estimate $\hat f(x,g)$ of $f^\ast(x,g)$ from an annotator (human or LLM).
Because $f^\ast$ is latent, humans and LLMs should be compared as \textit{estimators of the same quantity}, rather than as interchangeable label generators.

\subsection{A Two-Lens Model of Systematic Bias}
\label{sec:lenses}

We decompose a single PT prediction $\hat f_A(x,g)$ into two bias components plus residual noise:
\begin{align}
\label{eq:single_decomp_main}
\hat f_A(x,g) = {} & \, f^\ast(x,g)
+
\underbrace{b_{\mathrm{repr},A}(x,g)}_{\text{\small Wide Lens}}
\nonumber \\
&{}+
\underbrace{b_{\mathrm{proc},A}(x,g)}_{\text{\small Clear Lens}}
+
\varepsilon_A(x,g),
\end{align}
where $A\in\{H,L\}$ denotes \textbf{H}umans or \textbf{L}LMs and $\mathbb{E}[\varepsilon_A(x,g)]=0$.

$b_{\mathrm{repr}}$ captures \textit{representation bias}: how well the estimator approximates the population distribution $P_g$ over individuals in the target group $g$, and thus how accurately it estimates expectations under $h \sim P_g$.
A \textit{Wide Lens} reflects a smaller mismatch between $P_g$ and the estimator's implicit sampling distribution over individuals in $g$.
For example, when estimating how Gen~Z would judge the slang term ``slay,'' an older annotator may overweight their own social circle, misrepresenting the true demographic mixture.
LLMs often yield broader coverage for common groups due to pretraining on large-scale corpora.

$b_{\mathrm{proc}}$ captures \textit{processing bias}: how an internal representation is translated into a numeric judgment, i.e., the fidelity of the \textit{Clear Lens} through which the annotator renders a judgment.
In the aforementioned example, even if the older annotator recalls a Gen~Z example, they might mistakenly project their own norms, interpreting the word as violent rather than positive.
LLMs also exhibit processing distortions, but these errors tend to be more stable and more readily modifiable.

\paragraph{Coupling and super-additivity.}
The two bias components need not be independent.
For humans, social identity and homophily often link who is represented with how judgments are processed, causing the biases to align in sign.
Let $\mu_A(x,g)\triangleq \mathbb{E}[\hat f_A(x,g)-f^\ast(x,g)]$ denote the total mean bias, which decomposes as $\mu_A = \mu_{\mathrm{repr},A} + \mu_{\mathrm{proc},A}$, where $\mu_{\mathrm{repr},A} \triangleq \mathbb{E}[b_{\mathrm{repr},A}]$ and $\mu_{\mathrm{proc},A} \triangleq \mathbb{E}[b_{\mathrm{proc},A}]$.
Expanding the total squared bias yields
\begin{equation}
\label{eq:bias_expand_main}
\mu_A^2
=
\mu_{\mathrm{repr},A}^2
+
\mu_{\mathrm{proc},A}^2
+
\underbrace{2\,\mu_{\mathrm{repr},A}\,\mu_{\mathrm{proc},A}}_{\text{Coupling}},
\end{equation}
In out-group PT, the coupling term is expected to be positive, producing \textit{super-additive} error.
For LLMs, representation errors (driven by pretraining coverage) and processing errors (driven by post-training and inference-time prompting) may arise from mechanically distinct sources, potentially weakening coupling or even having a subtractive effect.

\subsection{Bias--Variance--Correlation Under Aggregation}
\label{sec:bvc}

Let $\{\hat f_{A,i}(x,g)\}_{i=1}^k$ be $k$ exchangeable PT predictions produced by protocol $A$ (e.g., $k$ humans or $k$ diversified LLM samples), and let
$\bar f_A^{(k)}(x,g)=\frac{1}{k}\sum_{i=1}^k \hat f_{A,i}(x,g)$.
By definition, $\mathrm{MSE}\!\left(\bar f_A^{(k)};x,g\right)
=
\mathbb{E}\!\left[\big(\bar f_A^{(k)}(x,g)-f^\ast(x,g)\big)^2\right]$, which reduces to $\mu_A(x,g)^2
+ \mathrm{Var}\left(\bar f_A^{(k)}(x,g)\right)$.
Define the centered residual $r_{A,i}(x,g) \triangleq \hat f_{A,i}(x,g) - f^\ast(x,g) - \mu_A(x,g)$, which absorbs both centered bias variation and noise.
Let $V_A(x,g)\triangleq \mathrm{Var}(r_{A,i}(x,g))$ denote per-annotator residual variance and let $\gamma_A(x,g)$ denote the exchangeable residual correlation, i.e., $\mathrm{Corr}(r_{A,i},r_{A,j})=\gamma_A$ for $i\neq j$.
As shown in Appendix~\ref{sec:appendix_theory}, exchangeable-variance calculation gives:
\begin{align}
\label{eq:mse_bvc_main}
\mathrm{MSE}\!\left(\bar f_A^{(k)};x,g\right)
&=
\underbrace{\mu_A(x,g)^2}_{\text{Bias}^2}
+
\underbrace{\gamma_A(x,g)V_A(x,g)}_{\text{Correlation floor}}
\nonumber \\
&\;\;\,+\underbrace{\frac{1-\gamma_A(x,g)}{k}V_A(x,g)}_{\text{Reducible variance}}
\end{align}
This decomposition makes explicit how bias, variance, and correlation jointly determine estimator quality and yields a simple decision criterion:
LLM PT is preferable whenever $\mathrm{MSE}(\bar f_L^{(k)};x,g) < \mathrm{MSE}(\bar f_H^{(k)};x,g)$.
This inequality highlights when LLMs can move from fallback to frontline.

\subsection{Predicted Performance Regimes and Control Levers}
\label{sec:hypotheses}

The estimation framework yields four testable predictions that structure our evaluation.

\paragraph{H1: The Budget Regime Hypothesis.}
In low-budget regimes (small $k$), MSE is dominated by per-annotator variance. Since LLMs are relatively deterministic under fixed prompting ($V_L \ll V_H$), LLM prediction should outperform human annotation---including direct annotation involving a broad group---by minimizing sampling noise.
In contrast, performance in high-budget regimes is limited by bias and correlation floors, and naive aggregation yields diminishing returns.

\paragraph{H2: The Coupling Hypothesis.}
In out-group settings, human PT increases error not only by enlarging bias magnitudes but by increasing positive coupling in Eq.~\eqref{eq:bias_expand_main}, inflating $\mu_H(x,g)^2$ super-additively.
LLMs lacking social identity should exhibit greater stability across crowd settings.

\paragraph{H3: The Representation Limits Hypothesis.}
As target groups become more specific or less prevalent, representation mismatch increases due to sparse or stereotype-skewed training evidence, enlarging $|b_{\mathrm{repr},L}(x,g)|$ and hence $|\mu_L(x,g)|$. This increases LLM PT error, and delineates regimes where human (in-group) knowledge could be potentially advantageous.

\paragraph{H4: The Engineerability Hypothesis.}
Compared to human annotation, LLM can be easier to engineer for perspective-taking. Its error components are mechanically distinct and can be selectively influenced through model choice (primarily $b_{\mathrm{repr}}$), prompting strategies and reasoning protocols (primarily $b_{\mathrm{proc}}$), and diversification (primarily $\gamma_L$). Consequently, LLM performance should systematically improve or degrade by targeted interventions.

\section{Experimental Setup}
\label{sec:exp}

We evaluate humans and LLMs as estimators of subgroup-level judgments under the \textit{perspective-taking-as-estimation} framework.
Our design combines (i) a high-fidelity human PT benchmark with dense subgroup annotation and (ii) a large-scale safety dataset with rich demographic coverage, enabling controlled comparisons across annotation budgets, group structure, and estimator regimes.
We briefly present experimental details here, deferring a more elaborate treatment to Appendix~\ref{app:details}.

\subsection{Datasets and Ground Truth}

\paragraph{\textsf{Toxicity Detection}.}
We use the dataset of \citet{duan2025exploring}, which contains $120$ online comments balanced by target group and toxicity level.
Each comment receives \textit{direct annotations} from at least $50$ U.S.-based crowdworkers from the target subgroup, each rating items for themselves (e.g., ``Do you find this comment toxic?'').
The mean of these direct annotations within a subgroup defines the \textit{ground-truth} label $f^\ast(x,g)$ (Eq.~\ref{eq:target_main}), serving as a high-fidelity proxy for the latent subgroup-level toxicity rate.
The dataset additionally includes matched human PT judgments.
We extend the dataset beyond binary gender by collecting new PT and direct annotations from non-binary participants ($N=97$) on Prolific.
Crucially, direct annotators and perspective-takers are entirely separate pools. Perspective-takers, whether human or LLM, are asked to estimate $f^\ast(x,g)$ (e.g., ``What percentage of \{females, males, non-binary people\} would find this comment toxic?'') without ever observing any direct annotations or aggregate statistics.

\paragraph{\textsf{DICES} (Conversational Safety).}
We additionally use \textsf{DICES-350} \citep{aroyo2023dices}, which provides over $100$ direct safety judgments per example across multiple demographic axes.
While DICES does not include human PT, its dense annotation structure enables precise subgroup-level estimation and allows us to probe LLM PT behavior under increasing demographic breadth and heterogeneity.
We use DICES to study regime behavior and boundary conditions rather than direct human--LLM comparisons.

\subsection{Estimators and protocols}
We consider three annotation protocols: (i) \textit{direct annotation} by subgroup members, (ii) \textit{human PT} (in-group/out-group), and (iii) \textit{LLM PT}, where models are prompted to estimate subgroup-level judgments under instructions aligned with human protocols.
LLM PT is out-group by design and never given access to aggregate statistics.
LLM PT, like human PT, outputs the estimated mean fraction.
We evaluate a diverse set of contemporary LLMs across model families, scales, and inference strategies. 
Specifically, we evaluate models from four families (GPT, Qwen, Gemma, and DeepSeek) ranging from 1B parameters to frontier scale.
For the pretrained vs.\ post-trained models analysis (Appendix~\ref{app:pt_it}), we additionally evaluate matched model pairs from Ministral 3.
For reasoning-enabled variants, we use each model's native reasoning mode rather than manually appended chain-of-thought instructions.
All models are queried zero-shot, unless noted otherwise.

\subsection{Evaluation Criteria}
We treat the mean of direct human annotations per subgroup as the gold standard and evaluate estimators using mean squared error (MSE), bias, and variance.
To reflect realistic annotation budgets, we use bootstrapping to simulate $k$-annotation regimes ($k=1$ to $10$), applying identical procedures to human and LLM estimators.  Each reported metric is the average over $B{=}1{,}000$ bootstrap resamples, ensuring stable estimator comparisons that do not depend on any single annotator draw.
\section{Evaluation I: When Do LLMs (Not) Excel at Perspective-Taking?}
\label{sec:evaluation1}

We first evaluate the comparative advantage of LLMs across different budget regimes ($k$) and demographic granularities.
Each subsection directly tests a hypothesis from Section~\ref{sec:hypotheses}.
We restrict LLM PT results to popular GPT models here.
See Appendix~\ref{app:experiments} for additional models and ablations.

\begin{figure}[t]
    \centering
    \begin{subfigure}[c]{\linewidth}
        \centering
        \includegraphics[width=\linewidth]{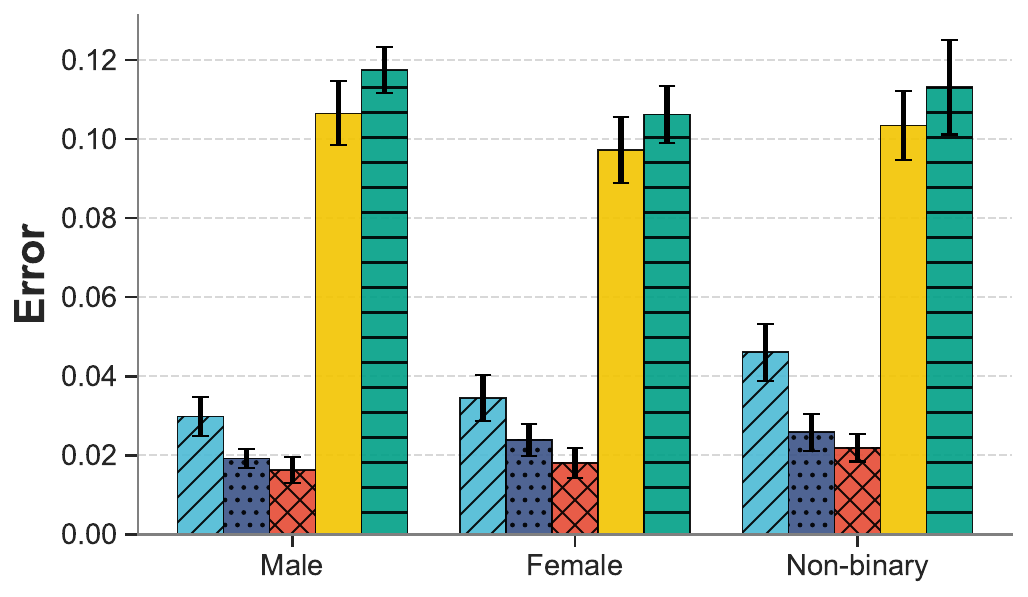}
    \end{subfigure}
    \begin{subfigure}[c]{\linewidth}
        \centering
        \includegraphics[width=\linewidth]{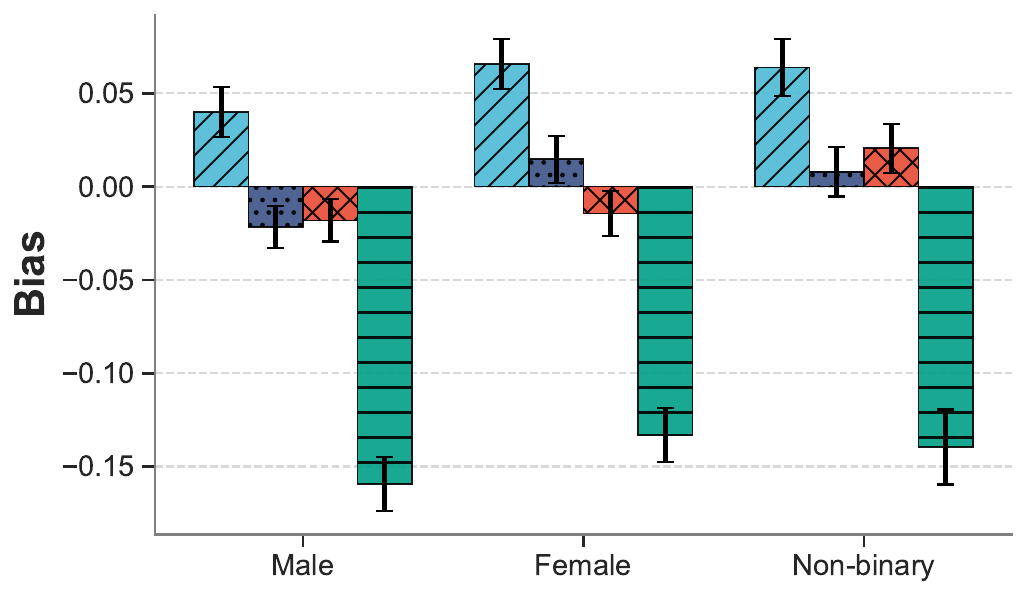}
    \end{subfigure}
    \begin{subfigure}[c]{\linewidth}
        \centering
        \includegraphics[width=0.97\linewidth]{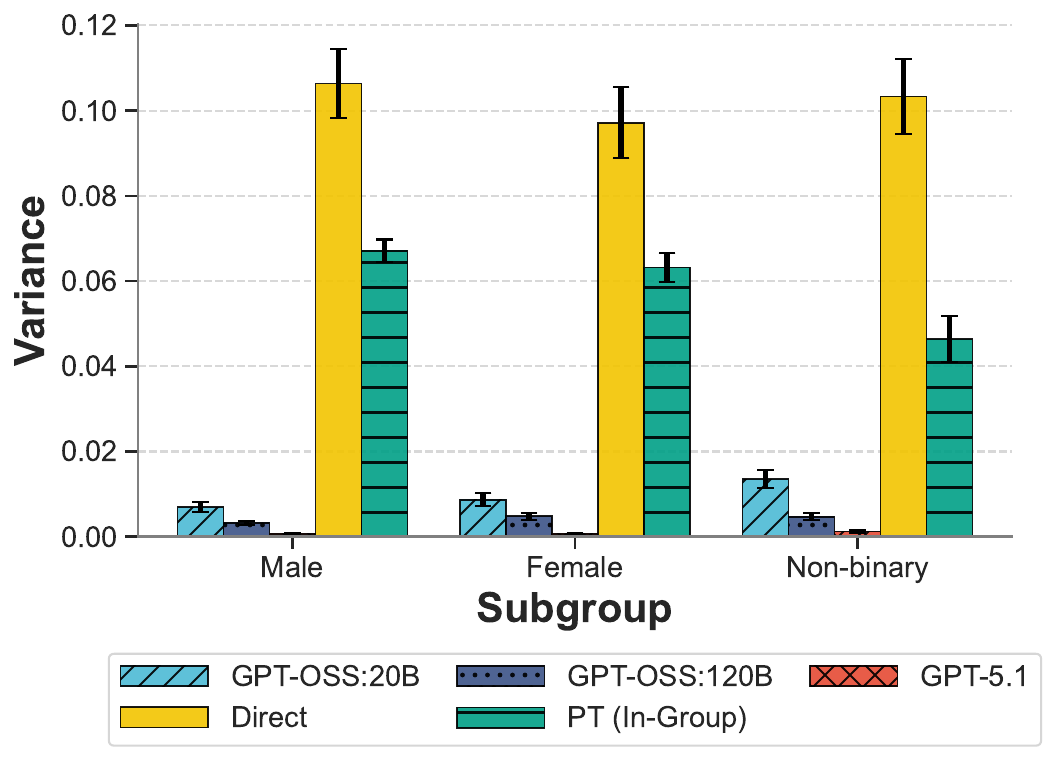}
    \end{subfigure}
    \caption{\textbf{Single-annotation ($k{=}1$) error decomposition} for three GPT variants vs.\ human baselines on \textsf{Toxicity Detection}. LLMs achieve lower MSE (top) across all gender subgroups, driven by lower bias (middle) and substantially lower variance (bottom).}
    \label{fig:metrics-1}
\end{figure}

\subsection{Validating the Budget Regime Hypothesis}

\paragraph{LLM vs Human ($k=1$).}
Many practical annotation pipelines rely on a single pass per item, i.e., a single annotator regime. We compare a single zero-shot LLM against a single human annotator (see Figure~\ref{fig:metrics-1}).
Across all gender groups in the \textsf{Toxicity Detection} dataset, the single LLM consistently achieves lower error.
This directly validates the low-budget variance-dominance prediction of H1: individual humans are noisy estimators (large $V_H$) of the group mean, whereas LLM PT is comparatively deterministic ($V_L \ll V_H$).
Decomposing MSE further reveals that LLM PT also dominates human PT in the bias component.
Interestingly, human perspective-takers systematically \textit{underestimate} the fraction of the target group that would judge content as toxic (negative bias).
LLM perspective-takers have lower bias, and some even tend to \textit{overestimate} it
(positive bias), which could indicate conservative safety calibration.

\begin{figure}[t]
    \centering
    \includegraphics[width=\linewidth]{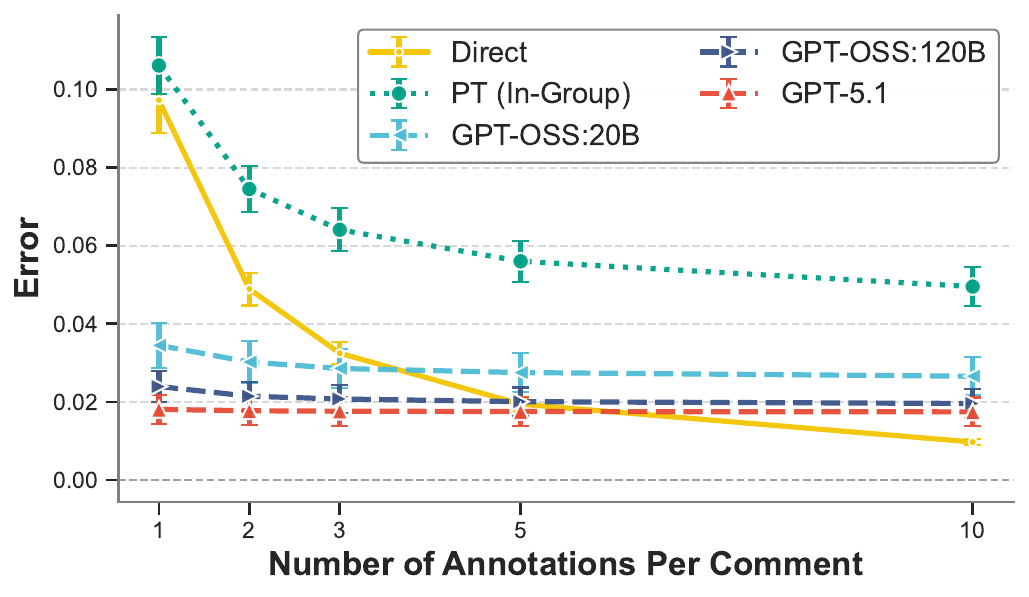}
    \caption{\textbf{MSE vs.\ annotation budget $k$} for the female subgroup. A single LLM PT estimate ($k{=}1$) is comparable to aggregating 3--5 direct human labels.}
    \label{fig:mse_llm_human_main}
\end{figure}

\paragraph{LLM Ensembles vs.\ Human Crowds ($k>1$).}
We next examine how performance scales with the annotation budget $k$ (Figure~\ref{fig:mse_llm_human_main}).
As $k$ increases, human PT improves rapidly at first but soon plateaus, consistent with a non-trivial correlation floor after accounting for bias.
In contrast, homogeneous LLM ensembles (multiple samples from the same model) show only modest gains, consistent with near-deterministic behavior ($V_L$ small) rather than necessarily high inter-sample correlation.

\paragraph{Outperforming Single-Annotator Baselines.}
A striking consequence of H1 is that a single LLM PT estimate can outperform a single \textit{direct} human annotation.
While direct annotators are unbiased with respect to their own judgment, they are high-variance ($k=1$) samples from the population defining $f^\ast(x,g)$.
Formally, at $k{=}1$ a single direct annotation $Y_h$ has $\mathrm{MSE} = V_{\mathrm{hetero}}(x,g)$ (the within-group heterogeneity, since $\mathbb{E}[Y_h] = f^\ast$), while a single LLM PT estimate has $\mathrm{MSE} = \mu_L^2 + V_L$. Hence LLM PT is the superior single-annotator estimator whenever $\mu_L^2 + V_L < V_{\mathrm{hetero}}$, i.e., when LLM bias and variance together are smaller than the population spread.
Empirically, a single LLM PT estimate is found comparable to aggregating $3$–$5$ direct human labels.
This result highlights an often-overlooked implication: \textit{when ground truth is estimated from a small number of human annotations, LLM-generated estimates can be the statistically preferable choice until sufficient human replication is available}.

\subsection{Validating the Coupling Hypothesis}
\begin{figure}[t]
    \centering
    \includegraphics[width=\linewidth]{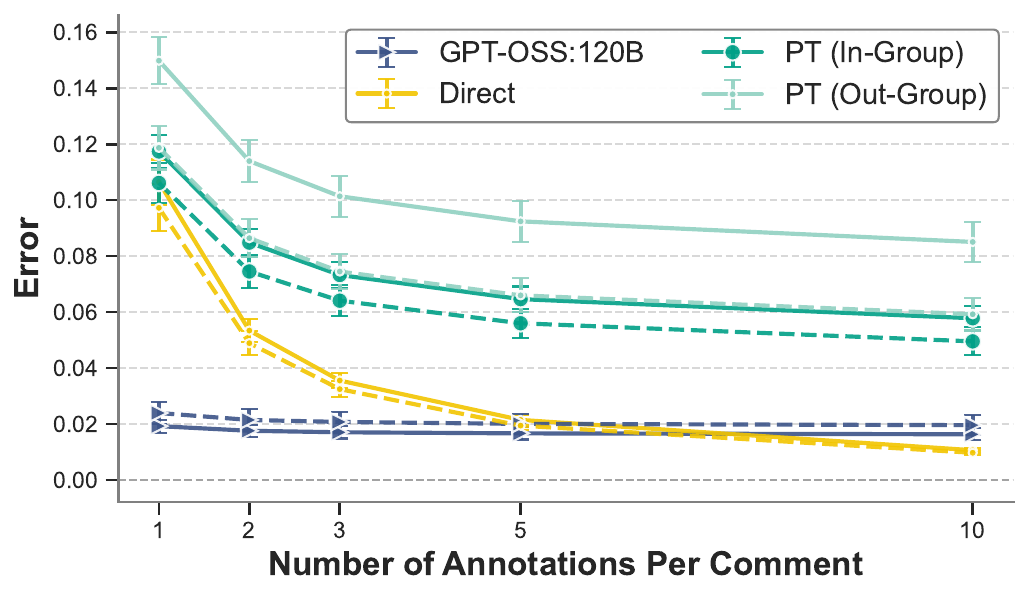}
    \caption{\textbf{In-group vs.\ out-group human PT} compared to LLM (GPT-OSS:120B) PT, predicting male (solid) and female (dashed) judgments. Out-group human PT incurs a pronounced error increase, consistent with super-additive coupling, while the LLM remains comparatively stable across target groups.}
    \label{fig:ingroup_outgroup}
\end{figure}

\paragraph{In-group vs.\ out-group human PT.}
Figure~\ref{fig:ingroup_outgroup} compares in-group and out-group human PT.
Out-group PT exhibits a pronounced error increase, especially when female annotators predict male judgments.
This asymmetry is consistent with the coupling mechanism in Eq.~\eqref{eq:bias_expand_main}: identity mismatch simultaneously distorts representation and processing, inflating $\mu_H^2$ super-additively.
The effect is weaker when  male annotators predict female judgments, consistent with differing bias magnitudes across groups.
While out-group challenges feel intuitive, our framework provides a structural account of such mechanism.

\paragraph{LLM stability.}
LLM PT is out-group by construction, yet its error remains comparatively stable across target groups.
This supports H2's prediction that mechanically decoupled representation and processing reduce sensitivity to identity mismatch, making LLMs particularly robust when in-group human recruitment is infeasible.

\subsection{Validating the Representation Limits Hypothesis}
\label{sec:eval1_failure}
We now test regimes where LLM PT should deteriorate due to representation mismatch.
We conceptualize demographic groups as nodes in an inclusion tree, where depth corresponds to specificity ($g' \subset g$, e.g., ``college-educated, black women'' is a more specific group than ``college-educated people'') and width at the same depth corresponds to prevalence $\pi(g)$ (e.g., within US, ``White'' race is more prevalent with larger $\pi(\text{White})$ compared to ``Asian'').
We use the \textsf{DICES} dataset here, focusing exclusively on LLM PT behavior due to the absence of human PT annotations.

\paragraph{Group specificity (depth).}
As groups become more specific, LLM conditioning increasingly relies on sparse evidence.
Figure~\ref{fig:influence_specificity} shows a monotonic increase in LLM (GPT-OSS:120B) PT error with subgroup depth.
This trend is consistent with increasing $|b_{\mathrm{repr},L}|$ dominating the error, as predicted by H3.

\begin{figure}
    \centering
    \includegraphics[width=\linewidth]{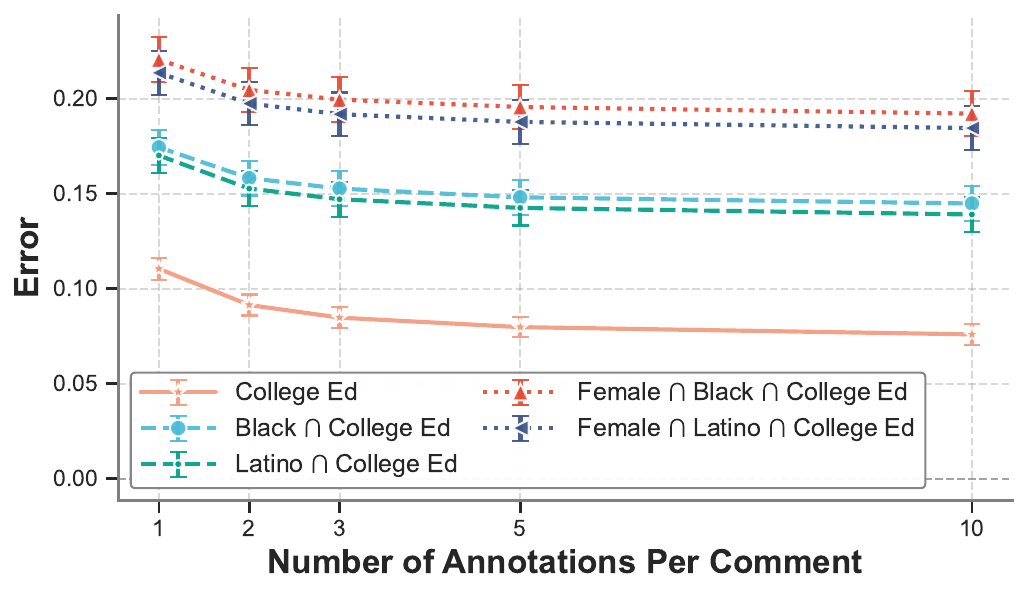}
    \caption{\textbf{Effect of subgroup specificity} on LLM (GPT-OSS:120B) PT error on \textsf{DICES}. MSE rises monotonically as the target becomes more specific.}
    \label{fig:influence_specificity}
\end{figure}

\paragraph{Group prevalence (width).}
Low-prevalence groups pose a complementary challenge.
As $\pi(g)$ decreases, LLMs can suffer representation mismatch due to limited or stereotype-skewed training signal.
Figure~\ref{fig:influence_prevalence} shows that performance worsens for minority demographics (e.g., Black vs. White).
This aligns with established findings on LLM fairness \cite{sap2019risk} and, crucially, helps identify the structural boundary of the ``frontline'' claim.

\begin{figure}[t]
    \centering
    \includegraphics[width=\linewidth]{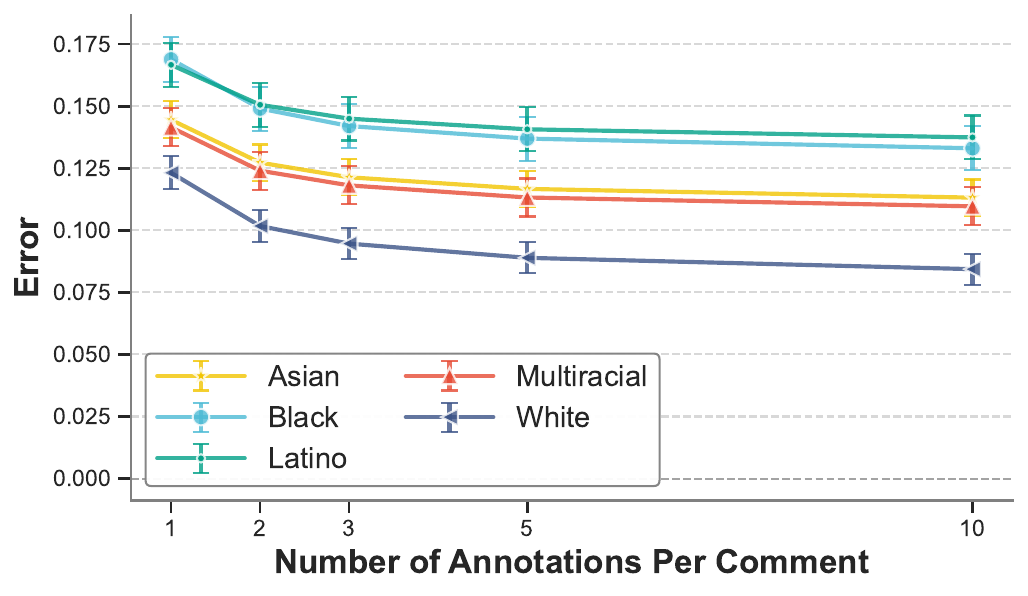}
    \caption{\textbf{Effect of subgroup prevalence} on LLM (GPT-OSS:120B) PT error on \textsf{DICES} (race axis). MSE generally rises as the target becomes less prevalent.}
    \label{fig:influence_prevalence}
\end{figure}

\section{Evaluation II: On The Engineerability of LLM Perspective-Taking}
\label{sec:evaluation2}
Evaluation~I established when LLMs outperform humans as estimators of subgroup-level judgment.
Here we examine \emph{how this advantage arises} and which components of error can be controlled,
helping validate the Engineerability Hypothesis (H4) from Section~\ref{sec:hypotheses}.
Guided by the bias--variance--correlation decomposition (Eq.~\ref{eq:mse_bvc_main}), we organize interventions by whether they primarily affect the \emph{Wide Lens} (representation bias), \emph{Clear Lens} (processing bias), or the \emph{correlation floor}.%
\footnote{As soft evidence that Wide and Clear Lens errors arise from mechanically distinct training stages, we also conduct matched pretrained vs.\ post-trained comparisons (Appendix~\ref{app:pt_it}). Post-training collapses variance by an order of magnitude alongside generally increasing absolute bias. While this does not separately measure representation and processing bias, it shows that total bias rises through post-training, directionally consistent with the two-lens view.}
Results here focus on the female subgroup, with additional groups and robustness checks in Appendix~\ref{app:experiments}.

\begin{figure}[t]
    \centering
    \begin{subfigure}[c]{\linewidth}
        \centering
        \includegraphics[width=0.95\linewidth]{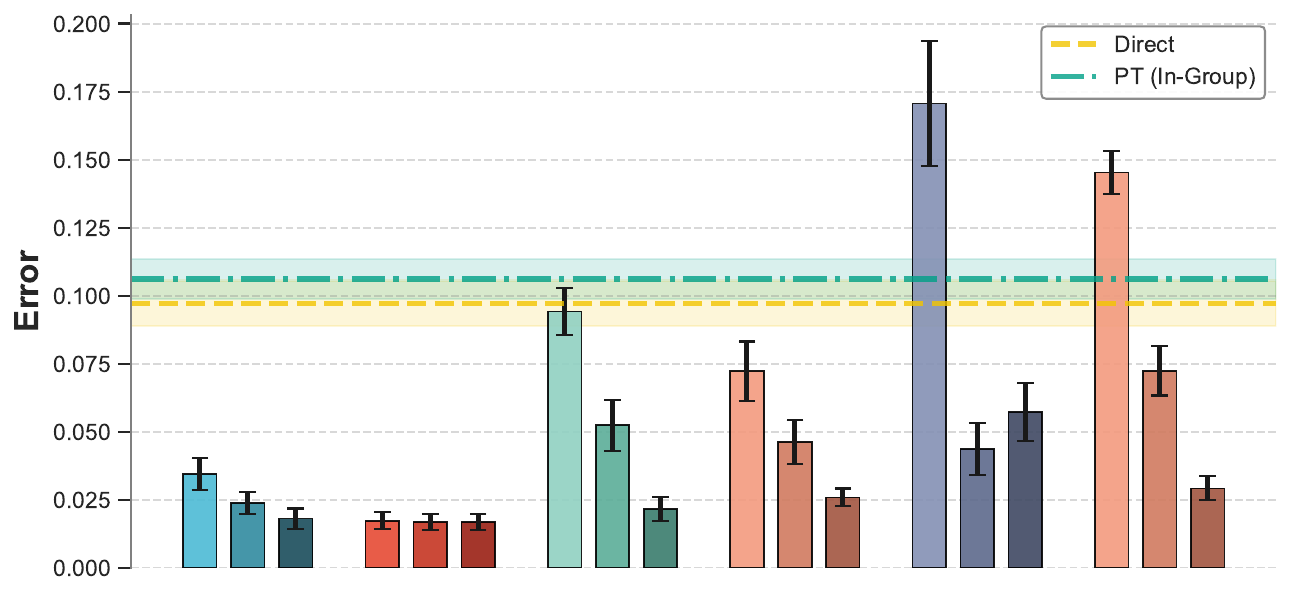}
    \end{subfigure}
    \begin{subfigure}[c]{\linewidth}
        \centering
        \includegraphics[width=0.95\linewidth]{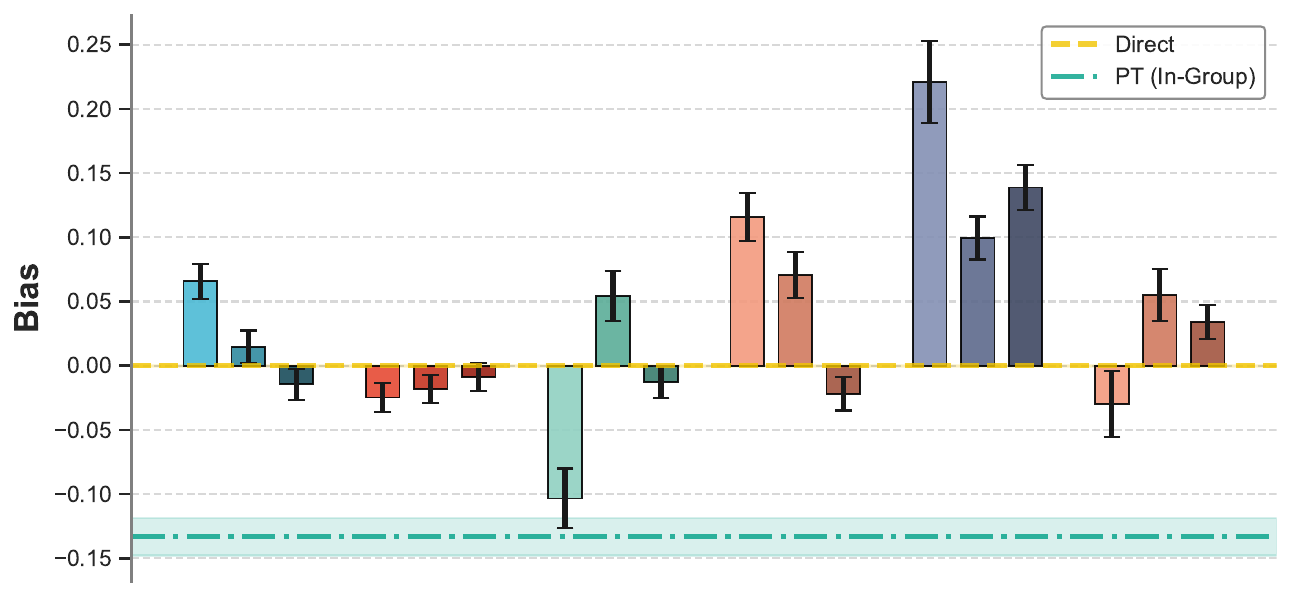}
    \end{subfigure}
    \begin{subfigure}[c]{\linewidth}
        \centering
        \includegraphics[width=0.925\linewidth]{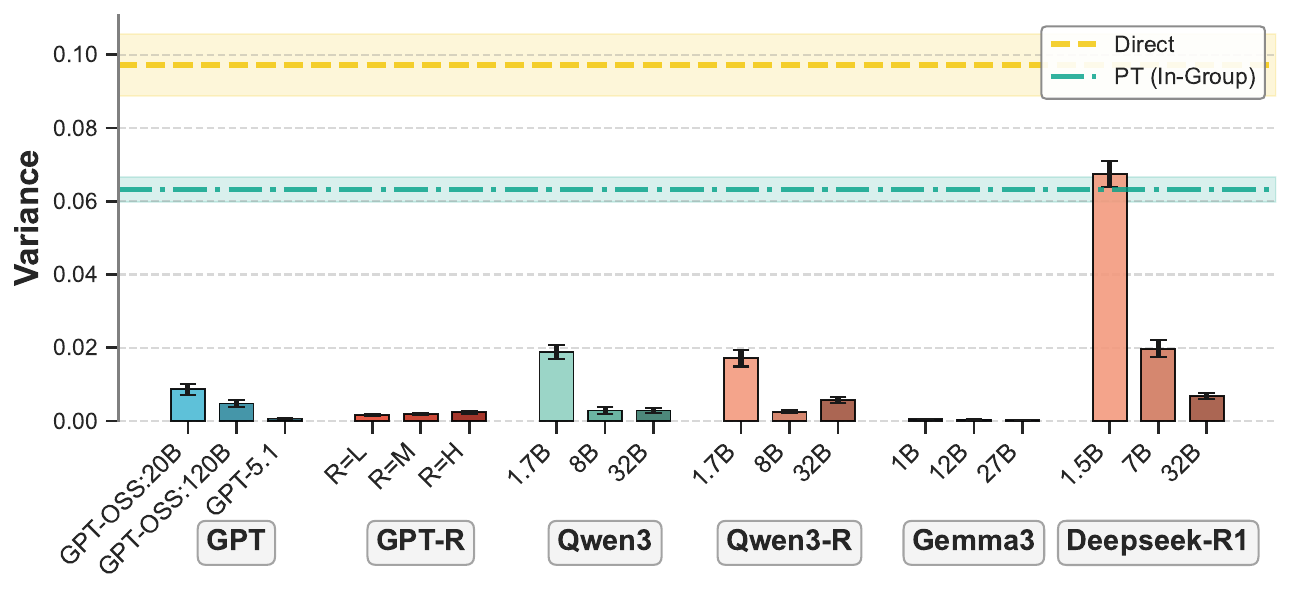}
    \end{subfigure}
    \caption{\textbf{Impact of model family and scale} on single-annotation ($k{=}1$) PT for the female subgroup. Cross-model differences in MSE (top) are dominated by bias (middle), with variance (bottom) remaining small.}
    \label{fig:influence_models}
\end{figure}

\subsection{Wide Lens: Model Family and Scale}
\label{sec:eval2_widelens}
To probe representation effects, we compare LLMs spanning model families and scales under an identical PT protocol.
While model choice is not a pure manipulation of representation bias, it serves as an empirically grounded proxy for representational fidelity via differences in pretraining coverage, capacity, and training regime.
Figure~\ref{fig:influence_models} shows that model family and scale induce large differences in MSE, sufficient to reverse whether LLM PT outperforms human PT.
While bigger models generally outperform weaker ones,
improvements are not monotonic: mid-sized models from one family can outperform larger models from another.
Error decomposition reveals that these differences are primarily driven by bias, with variance remaining comparatively small across models.

\subsection{Clear Lens: Prompting and Reasoning}
We next test whether processing bias can be modified without changing representation.
We fix the model (GPT-OSS:120B) and use a controlled prompting ladder of increasing structure: L1 (question only), L2 (+ definition), L3 (+ levels), L4 (+ examples); see Appendix~\ref{app:prompts} for complete details.
Figure~\ref{fig:influence_prompting_combined} shows that prompting can substantially shift both MSE and bias.
Importantly, this is \textit{not} a monotonic ``more structure is better'' result: different prompts reweight how beliefs are translated into numeric estimates.
Thus, prompting acts as a calibration mechanism that reshapes $b_{\mathrm{proc}}$ rather than universally reducing error.

\begin{figure}[t]
    \centering
    \begin{subfigure}[c]{0.65\linewidth}
        \centering
        \includegraphics[width=\linewidth]{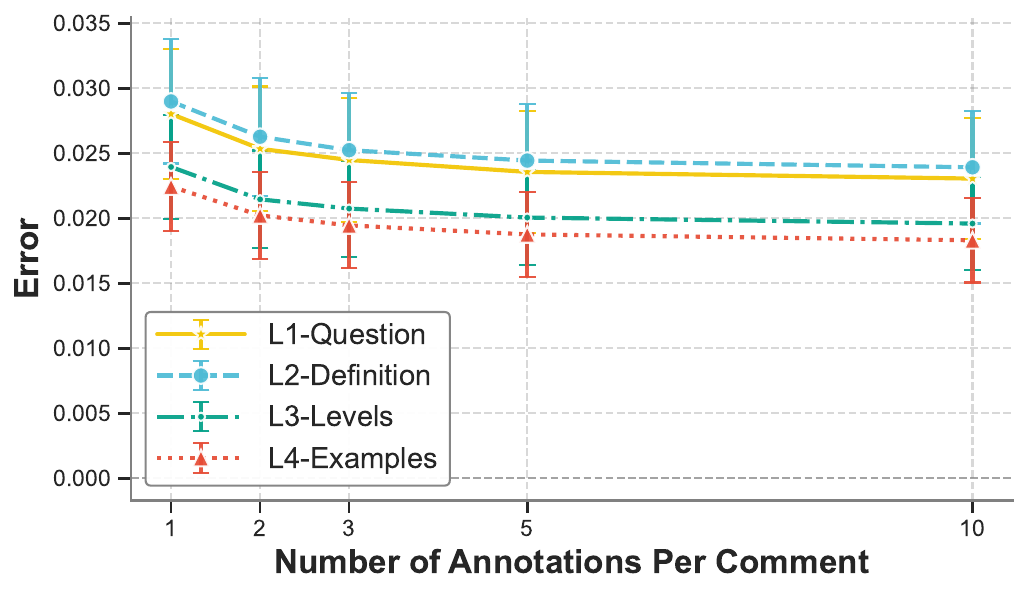}
    \end{subfigure}
    \hfill
    \begin{subfigure}[c]{0.33\linewidth}
        \centering
        \includegraphics[width=\linewidth]{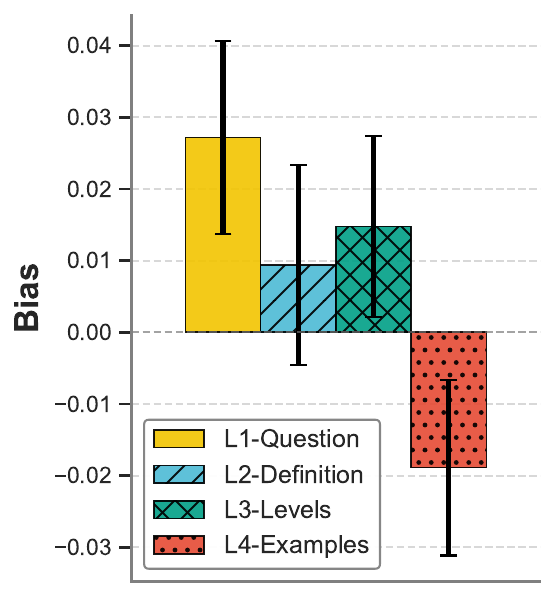}
    \end{subfigure}

    \caption{\textbf{Impact of prompting} on single-annotation LLM (GPT-OSS:120B) PT for the female subgroup. Increasing prompt structure from L1 through L4 lowers MSE (left) primarily by shifting bias (right) and can even flip its sign.}
    \label{fig:influence_prompting_combined}
\end{figure}

\begin{figure}[t]
    \centering
    \includegraphics[width=0.9\linewidth]{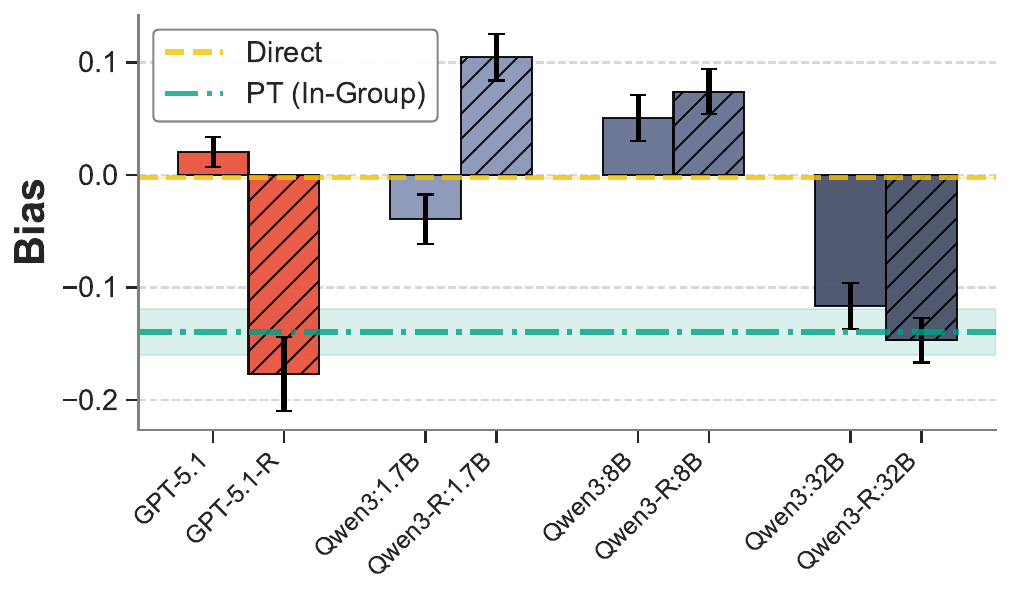}
    \caption{\textbf{Reasoning paradox diagnostic} on the non-binary subgroup: for four base--reasoning pairs, enabling native reasoning (hatched) shifts systematic bias away from ground truth, consistent with criterion drift.}
    \label{fig:reasoning_nonbinary_bias}
\end{figure}

\paragraph{The Reasoning Paradox.}
Reasoning-enabled inference modes do not reliably improve PT accuracy, and can in fact substantially worsen it.
As shown in Figure~\ref{fig:reasoning_nonbinary_bias}, explicit reasoning can induce larger systematic bias, especially for a harder subgroup like non-binary people.
The effect is model-dependent, and analysis of reasoning traces (Appendix~\ref{app:reasoning_traces}) suggests that the dominant mechanism is \textit{criterion drift}: reasoning drifts models away from estimating the empirical group toxicity rate toward applying a rubric-based classification standard, producing a systematic shift whose sign effect depends on the base model's pre-existing bias. Identity-mediated `re-coupling' between representation and processing (cf.\ Section~\ref{sec:lenses}) is one pathway for this drift, although not the main driver in our inspected traces.

\subsection{Correlation: Mixing and Temperature}
We next test interventions aimed at diversification, for lowering the correlation floor $\gamma_L V_L$ in Eq.~\ref{eq:mse_bvc_main}.
We study two common diversification mechanisms:
(i) \textit{model mixing} (aggregating predictions from different models), and
(ii) \textit{temperature} (increasing sampling randomness within a single model).
Mixing similar-sized models across families yields modest but consistent error reductions (Figure~\ref{fig:influence_model_mixing}), whereas mixing within a family across sizes (Figure~\ref{fig:influence_mixing_size}) or increasing temperature (Figure~\ref{fig:influence_temperature_models}) provides limited gains.
These findings suggest that correlation engineering is secondary in many regimes.

\begin{figure}[t]
    \centering
    \includegraphics[width=0.9\linewidth]{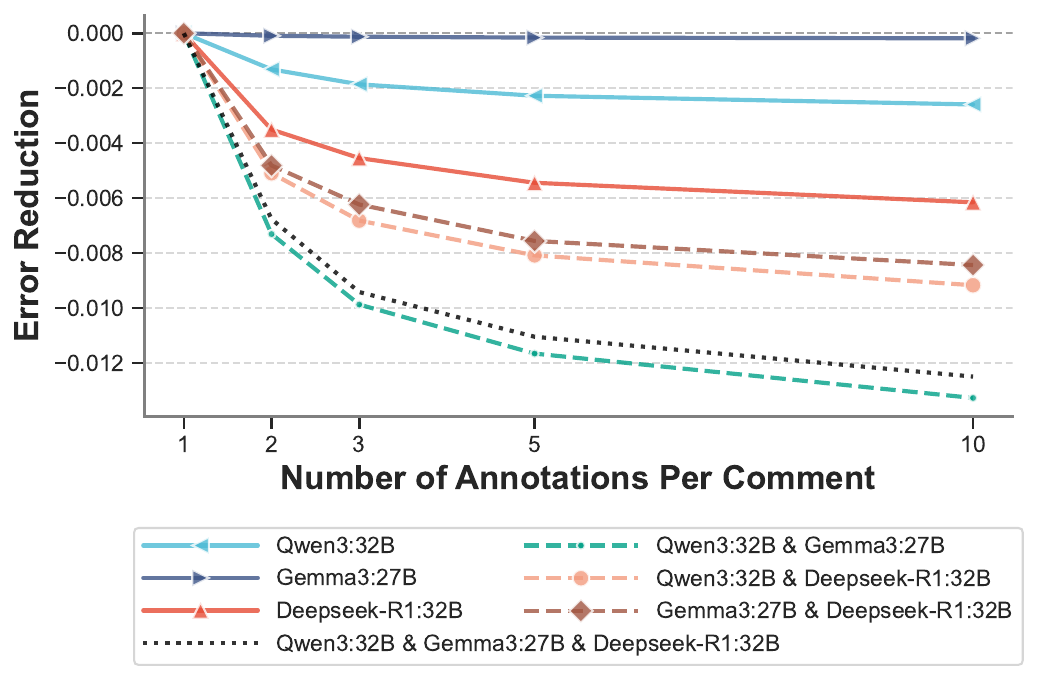}
    \caption{\textbf{Cross-family model mixing} at the large-model tier on the female subgroup. Mixing across families (dashed/dotted) yields consistent but modest error reduction relative to each model's $k{=}1$ baseline (solid). Nonetheless, mixtures often remain inferior to the strongest single model.}
    \label{fig:influence_model_mixing}
\end{figure}

\section{Discussion and Conclusion}
\label{sec:discussion}

This work challenges the prevailing assumption that LLM-based annotation is merely a cost-saving approximation of human judgment.
By reframing perspective-taking (PT) as statistical estimation of a latent group-level quantity, we show—both theoretically and empirically—that humans are not always the best available estimators of aggregate group perspectives.
Rather than treating LLMs as a pragmatic fallback, our results clarify \textit{when}, \textit{why}, and \textit{how} LLMs can act as statistically preferable frontline estimators—and where human annotation remains essential.

\paragraph{Reframing the comparison.}
Much of the discomfort around LLM-based annotation stems from an implicit category error: humans are evaluated as sources of lived, individual judgments while LLMs are evaluated as predictors of aggregates.
Once both are compared as estimators of the same latent target $f^\ast(x,g)$, the question shifts from authenticity to estimation quality under constraints.
Our bias--variance--correlation framework makes this comparison explicit, principled, and testable.

\paragraph{When LLMs are likely to be optimal.}
LLMs dominate when (i) the target group is broad or internally heterogeneous, where even in-group humans suffer Wide Lens error; (ii) annotation budgets are small, making performance variance-dominated and LLM stability ($V_L \ll V_H$) decisive; and/or (iii) humans are recruited out-of-group, inducing coupled representation and processing bias that amplifies systematic error.
In these regimes, LLMs achieve lower error through low variance, weaker coupling, and engineerable calibration.

\paragraph{When humans are likely to remain superior.}
Human annotation remains indispensable when (i) the group is deep, specific, and cohesive---making lived experience highly informative and representation bias small---and/or (ii) the group is sufficiently low-prevalence that LLM conditioning relies on sparse or stereotype-skewed evidence.
Our differential perspective-taking analysis (Appendix~\ref{sec:dpt_results}) provides direct evidence for this boundary: on the male vs. non-binary contrast, human PT significantly outperforms most LLMs at capturing per-item group differences, confirming that lower-prevalence groups remain a genuine frontier where LLMs lag.
Importantly, some contexts also demand \textit{procedural legitimacy}: in high-stakes or normative settings, direct stakeholder adjudication may outweigh purely statistical considerations. 
Our results do not contest this but clarify where statistical optimality and social legitimacy diverge.

\paragraph{Estimation-aware engineering.}
A core implication of our analysis is that annotation quality is not fixed.
Practitioners can target specific error terms rather than indiscriminately collecting more labels.
While humans are also engineerable in spirit, the levers are asymmetrical: many LLM interventions are fast and scalable once an evaluation protocol exists, whereas human improvements---while powerful---are often slower, costlier, and access-constrained.
Model choice primarily determines representation bias, prompt design reshapes processing bias, and diversification can reduce correlated error (though likely with limited returns).
Notably, explicit reasoning can backfire, as it can introduce a systematic criterion drift whose impact depends on the direction of the base model's bias.

\paragraph{From mean to distributional objectives.}
Our framework targets the population mean, which is the operationally relevant quantity in many annotation pipelines.
However, for applications requiring pluralistic representation, such as capturing the full topology of disagreement within a group, distributional objectives are important complementary targets.
Our bias--variance--correlation decomposition extends naturally to other summary statistics (median, quartiles) and, in principle, to distributional targets.
Accurate mean estimation is a foundational first step, since an estimator that misses the mean is unlikely to recover higher moments.

\paragraph{Complementarity, not replacement.}
Our findings motivate a reallocation of effort: LLMs shoulder aggregate estimation in variance- and coupling-dominated regimes, freeing scarce human expertise for cases requiring contextual nuance, participatory grounding, or legitimacy.
Seen through this lens, LLMs are not substitutes for lived experience, but tools for estimating its aggregate structure more stably and at scale.
Additionally, if an LLM stabilizes around a biased estimate, this low-variance-around-wrong-mean stability can reinforce representational monoculture, underscoring the need for subgroup-aware validation and periodic human auditing even when LLMs serve as frontline estimators.

\paragraph{Final takeaway.}
Our core contribution is not the claim that LLMs are universally better annotators, but the demonstration that, once PT is treated as estimation, we can answer the important question of
\textit{who is the better estimator under specific conditions?}
By making that question precise, and operational, we show how LLMs can move, in well-defined settings, from fallback to frontline, while clarifying where humans remain essential.
The future of perspective-taking lies not in competition between humans and LLMs, but in their principled, complementary integration.

\section*{Limitations}

This work has several limitations that delineate the scope of our claims.

\paragraph{Dependence on representation quality.}
Our framework makes explicit that LLM performance is constrained by representation bias.
When subgroup evidence is sparse, highly specific, or distorted in training data, LLM-based perspective-taking can deteriorate despite low variance.
While our experiments identify such regimes, real-world use still requires subgroup-specific validation, particularly for emerging identities or rapidly shifting social contexts.

\paragraph{Scope of tasks and domains.}
We focus on toxicity and conversational safety, where aggregate judgments are well-defined and densely annotated.
Although the estimation framework is general, empirical results may not directly transfer to tasks where perspectives are more contested, multi-dimensional, or normatively grounded (e.g., moral or policy judgments).
Extending the analysis to such settings remains an important direction for future work.
In more contested domains (e.g., political orientation, aesthetic preference, cultural values), LLM representation biases may be amplified by long-tail distributions in training data, and the ``frontline'' claim should not be extrapolated without domain-specific validation.

\paragraph{Estimator-centric evaluation.}
We evaluate estimators against aggregate human judgments as proxies for the latent target $f^\ast(x,g)$.
This standard choice abstracts away intra-group disagreement and deliberative dynamics. Our results therefore address estimation accuracy rather than the full richness of social meaning-making.

\paragraph{Model and protocol dependence.}
While we study a diverse set of models and prompting strategies, LLM architectures and alignment techniques are rapidly evolving.
Some findings, particularly those concerning reasoning-enabled variants, may change as models and alignment techniques improve.
Accordingly, our conclusions should be read as structural insights about estimator behavior, not as immutable properties of specific models.

\medskip\noindent
Overall, these limitations do not undermine the central contribution of the paper, but rather clarify the conditions under which estimation-based (LLM) perspective-taking is most informative and reliable.

\section*{Ethical Considerations}

This work engages with considerations around representation, fairness, and automation in subjective judgment.

\paragraph{Risk of misrepresentation.}
LLMs inherit biases from their training data and may misestimate perspectives of marginalized or low-prevalence groups.
Our analysis explicitly identifies these failure modes and the regimes where human judgment remains essential.
Using LLM-based perspective-taking without subgroup-aware validation risks reinforcing existing inequities.

\paragraph{Procedural legitimacy and participation.}
Statistical optimality is not always sufficient.
In many high-stakes or normative settings, ethical legitimacy requires direct participation from affected communities.
Our results do not argue for replacing such participation, but for clarifying when LLMs can responsibly support aggregate estimation without displacing human agency.

\paragraph{Misuse and overgeneralization.}
LLM-based perspective-taking may be overextended beyond its validated scope, for example to justify decisions about groups lacking adequate representation.
We therefore emphasize transparent reporting of annotation protocols, explicit documentation of subgroup coverage, and conservative use in sensitive applications.

\paragraph{Responsible deployment.}
We recommend hybrid annotation pipelines that combine LLM-based estimation with periodic human auditing.
Such approaches leverage the stability and efficiency of LLMs while preserving accountability, inclusivity, and adaptability.

\medskip\noindent
Overall, LLMs are not substitutes for lived experience, but tools whose ethical use depends on alignment between statistical objectives, social context, and human oversight.

\bibliography{refs}

\appendix
\renewcommand\thefigure{A\arabic{figure}}
\renewcommand\thetable{A\arabic{table}}
\setcounter{figure}{0}
\setcounter{table}{0}

\section{Extended Related Work}
\label{sec:related}

\paragraph{Perspective-taking and subjective judgment.}
Prior work studies PT as a mechanism for eliciting judgments about how groups would respond, highlighting both its utility and distortions, particularly in out-group settings \cite{frenda2024perspectivist,duan2025exploring,sandri2023don}.
Recent HCI work further analyzes how humans and LLMs differ in the perspectives they invoke in subjective decision-making, showing reliance on differently weighted perspective distributions \cite{aoyagui2025matter}.
We advance this line by formalizing PT not merely as opinion elicitation, but as statistical estimation of a latent group-level judgment, enabling principled comparison of human and model estimators.

\paragraph{LLMs as annotators and replacement.}
Several studies evaluate LLMs as annotators or judges \cite{li2025generation}, focusing on agreement with human labels and conditions for replacing human annotators.
\citet{movva2024annotation} show that LLMs can approximate crowd averages for conversational safety but struggle with fine-grained between-group differences.
\citet{calderon-etal-2025-alternative} propose a statistical alternative annotator test to justify LLM replacement using limited labeled data.
In contrast, we analyze estimation efficiency, characterizing when LLMs or humans constitute the superior estimator.

\paragraph{Social reasoning and identity bias.}
LLMs perform strongly on social reasoning and theory-of-mind benchmarks \cite{rabinowitz2018machine,wang2021towards,huang2023chatgpt}, yet exhibit systematic social identity biases and human-like reasoning distortions that vary across models and training regimes \cite{hagendorff2023human,hu2025generative}.
We build on these insights by decomposing how bias, variance and correlation interact in PT estimation, including regimes where increased reasoning can counterintuitively degrade performance.

\paragraph{Persona-LLMs and sociodemographic prompting.}
A growing body of work investigates prompting LLMs with demographic personas to simulate individual-level opinions.
Several studies report mixed or negative results for such sociodemographic simulation: \citet{sun2025sociodemographic} find that persona prompting does not reliably improve prediction of individual subjective judgments, and \citet{lutz2025prompt} show that persona-based elicitation produces inconsistent alignment with human subgroup responses.
\citet{orlikowski2025beyond} fine-tune LLMs to behave as individual annotators and report that demographic features alone are insufficient predictors.
Our framework offers a possible explanation for these mixed findings: individual simulation implicitly requires approximating the full within-group distribution of judgments, whereas our task targets just the subgroup mean---a fundamentally easier estimation target.
\citet{joshi2025improving} find that psychologically-scaffolded rationales outperform default chain-of-thought for persona prediction, a pattern consistent with our reasoning paradox: unstructured reasoning may degrade group-level estimation by re-coupling representation and processing errors, while structured reasoning grounded in external frameworks can help.
\citet{meister2025benchmarking} find that LLMs more accurately \textit{describe} opinion distributions than \textit{simulate} them via role-play, consistent with our protocol of directly eliciting group-level statistics rather than role-playing individuals.

\paragraph{Pluralistic alignment and distributional objectives.}
A complementary line of work targets matching the full distribution of human opinions rather than a single summary statistic.
\citet{sorensen2024position} chart a roadmap for pluralistic alignment, and \citet{feng2024modular} propose multi-LLM collaboration to represent diverse viewpoints.
\citet{lee2024aligning} demonstrate system-message-based generalization for aligning to thousands of user preferences, while \citet{hu2025population} develop population-aligned persona generation for social simulation.
We see mean estimation as a foundational first step in this broader agenda: if an estimator cannot accurately recover the first moment, distributional fidelity is unlikely.
Our bias--variance--correlation decomposition provides formal vocabulary for reasoning about estimator quality that naturally extends to other summary statistics and distributional targets.

\begin{table*}[t]
\centering
\small
\renewcommand{\arraystretch}{1.25}
\begin{tabular}{p{2.6cm} p{4.3cm} p{4.3cm} p{4.3cm}}
\toprule
\textbf{} &
\textbf{Human Annotator} \newline \textbf{(``Lived'' Estimator)} &
\textbf{LLM Annotator} \newline \textbf{(``Learned'' Estimator)} &
\textbf{Key Insights} \\
\midrule

\textbf{Target $f^\ast(x,g)=\mathbb{E}_{h\sim P_g}[Y_h(x)]$} &
Estimate via social sampling and introspection. 
& Estimate via learned statistical regularities and instruction following. 
&
Both are imperfect estimators of the same latent subgroup mean. \\
\midrule

\textbf{Bias}
$\bigl(\mu_{\mathrm{repr}} + \mu_{\mathrm{proc}}\bigr)^2$
&
\textbf{Representation (Wide Lens):} Sampling bias; lived experience samples a \textit{local} neighborhood, misweighing latent subcommunities.
\newline
\textbf{Processing (Clear Lens):} Cognitive bias; 
distortions converting belief to numeric judgment; relatively hard-coded by psychology.
\newline
\textbf{Coupling:}
Identity (homophily) can link representation and processing errors,
yielding \textit{super-additive} bias.
&
\textbf{Representation (Wide Lens):}
Data bias; broad but imperfect coverage from pretraining; risky interpolation for sparse groups.
\newline
\textbf{Processing (Clear Lens):}
Modeling bias; distortions from alignment and safety constraints; relatively modifiable via prompting.
\newline
\textbf{Coupling:} 
Pretraining and post-training errors are mechanically distinct mechanisms, yielding weaker or mixed-sign interaction.
&
LLMs trained on vast subgroup data that are relatively immune to human-like cognitive biases can outperform humans.
\newline
Persistent gap even at large budgets indicates a bias advantage; often the case for broad groups or out-group human PT.
\newline
\textbf{Reasoning paradox:}
Explicit ``thinking'' (CoT) in LLMs can induce \textit{criterion drift}, shifting the model from distributional estimation to rubric-based classification and increasing bias.
\\
\midrule

\textbf{Variance}
$V$
&
\textbf{High noise} for single annotators ($k=1$) due to individual differences, mood, fatigue etc.;
decreases with aggregation.
&
\textbf{Low noise} for a single model under fixed prompting;
additional gains from ensembling are often modest.
&
\textbf{Low-budget regime:}
A single LLM can outperform a single human and rival
small human crowds, \textit{even for direct annotation}.
\\
\midrule

\textbf{Correlation}
$\gamma$
&
\textbf{Correlation floor:} Correlation from shared culture, platforms and norms;
difficult to engineer away.
&
\textbf{Engineerable floor:} High correlation from shared data and model architectures; partially engineerable.
&
\textbf{Human advantage at scale:}
Human crowds benefit more from aggregation, yet can plateau in PT.
\\
\midrule

\textbf{Decision Rule}
&
\multicolumn{3}{p{13.2cm}}{
LLMs excel when $\text{Bias}_L^2 + \gamma_L V_L + \tfrac{1-\gamma_L}{k}V_L < \text{Bias}_H^2 + \gamma_H V_H + \tfrac{1-\gamma_H}{k}V_H$; typical for large heterogeneous groups where LLM calibration and diversity can be engineered.
Humans remain preferable when groups are deep and cohesive while LLM conditioning is sparse or stereotype-skewed.
} \\
\bottomrule
\end{tabular}
\caption{
\textbf{Bias--Variance--Correlation View of Human vs. LLM for Perspective-Taking.}
Wide Lens, Clear Lens, and Coupling form a bias decomposition;
variance and correlation govern how performance scales with budget.
This offers an overview of when and why LLMs can move from fallback to frontline through PT-as-estimation lens.
}
\label{tab:mse_decomposition}
\end{table*}

\section{Theory}
\label{sec:appendix_theory}

\begin{table*}[t]
\centering
\small
\begin{tabular}{l l}
\toprule
Symbol & Meaning \\
\midrule
$x\in\mathcal{X}$, $g\in\mathcal{G}$ & item and target group \\
$P_g$ & population distribution over humans in group $g$ \\
$Y_h(x)\in[0,1]$ & direct judgment of item $x$ by human $h$ \\
$f^\star(x,g)$ & target group-mean judgment \\
$b_{\mathrm{W},H}$, $b_{\mathrm{C},H}$ & human Wide/Clear Lens bias components \\
$b_{\mathrm{W},L}$, $b_{\mathrm{C},L}$ & LLM Wide/Clear Lens bias components \\
$\mu_{\mathrm{W},H}$, $\mu_{\mathrm{C},H}$ & mean human Wide/Clear Lens biases \\
$\mu_{\mathrm{W},L}$, $\mu_{\mathrm{C},L}$ & mean LLM Wide/Clear Lens biases \\
$\mu_H=\mu_{\mathrm{W},H}+\mu_{\mathrm{C},H}$ & total mean human bias \\
$\mu_L=\mu_{\mathrm{W},L}+\mu_{\mathrm{C},L}$ & total mean LLM bias \\
$r_H$, $r_L$ & zero-mean residuals after removing mean biases \\
$V_H=\mathrm{Var}(r_H)$, $V_L=\mathrm{Var}(r_L)$ & per-annotator residual variances (variance at $n{=}1/m{=}1$) \\
$\gamma_H$, $\gamma_L$ & exchangeable residual correlations \\
$n$, $m$ & \# humans / \# LLM samples in an aggregate \\
\bottomrule
\end{tabular}
\caption{Summary of theoretical quantities.}
\label{tab:notation_appendix}
\end{table*}

This appendix develops a precise theoretical framework for when LLMs can outperform humans at perspective-taking annotation, i.e., predicting the group-mean judgment $f^\star(x,g)$ for an item $x$ and a demographic group $g$.
We provide:
(i) a Wide Lens (representation) model via latent subcommunity mixtures,
(ii) bias--variance--correlation decompositions for aggregated human and LLM estimators,
(iii) a careful treatment of coupling between Wide and Clear Lens errors (mean-level vs variance-level), and
(iv) asymptotic error floors and a superiority criterion.

\subsection{Problem Setup}
\label{sec:setup_appendix}

Let $\mathcal{P}_g$ denote the (conceptual) population of humans in group $g$, distributed as $P_g$.
For $h\sim P_g$, let $Y_h(x)\in[0,1]$ be $h$'s direct judgment.
The target group perspective is:

\begin{definition}[Target quantity]
\label{def:target_appendix}
\begin{equation}
    f^\star(x,g)\;\triangleq\;\mathbb{E}_{h\sim P_g}[Y_h(x)].
\end{equation}
\end{definition}
Perspective-taking methods (human or LLM) output $\hat f(x,g)$ and are evaluated by
\begin{equation}
\label{eq:mse_def_appendix}
\mathrm{MSE}(\hat f;x,g)
\triangleq
\mathbb{E}\!\left[\big(\hat f(x,g)-f^\star(x,g)\big)^2\right],
\end{equation}
where the expectation ranges over annotator identity / sampling randomness.

\subsection{Wide Lens via Latent Mixtures}
\label{sec:wide_lens_appendix}

We model demographic groups as mixtures over latent subcommunities that can differ systematically in judgments.

\begin{assumption}[Latent subcommunity mixture]
\label{ass:mixture_appendix}
For each group $g\in\mathcal{G}$, there exists a finite or countable set of latent subcommunities $\mathcal{C}_g$ and mixture weights
$\mathbf{w}^g=(w_c^g)_{c\in\mathcal{C}_g}$. with $w_c^g\ge 0$ and $\sum_c w_c^g=1$.
Each subcommunity $c$ has mean judgment
\begin{equation}
    f^\star_c(x,g)\;\triangleq\;\mathbb{E}[Y_h(x)\mid h\in c].
\end{equation}
Then
\begin{equation}
\label{eq:mixture_target_appendix}
    f^\star(x,g) \;=\; \sum_{c\in\mathcal{C}_g} w_c^g\, f^\star_c(x,g).
\end{equation}
\end{assumption}

\begin{definition}[Within-group heterogeneity]
\label{def:hetero_appendix}
\begin{equation}
\label{eq:Vhetero_appendix}
    V_{\mathrm{hetero}}(x,g)
    \;\triangleq\;
    \sum_{c\in\mathcal{C}_g} w_c^g\big(f^\star_c(x,g)-f^\star(x,g)\big)^2.
\end{equation}
\end{definition}

\paragraph{Annotator-internal mixture.}
An annotator (human or LLM) may implicitly reason about $g$ using an \textit{internal} mixture
$\mathbf{q}^{g,a}=(q_c^{g,a})_{c\in\mathcal{C}_g}$ with $\sum_c q_c^{g,a}=1$.
Define the Wide Lens Effect, a representation bias capturing how $a$'s mental model of group $g$ misweights sub-communities, as:
\begin{equation}
\label{eq:Brepr_def_appendix}
    b_{\mathrm{W}}(x,g;a)
    \;\triangleq\;
    \sum_{c\in\mathcal{C}_g}\big(q_c^{g,a}-w_c^g\big)\,f^\star_c(x,g).
\end{equation}

\begin{assumption}
\label{ass:support_appendix}
For any $(x,g)$ and any $a$ under consideration, if $q_c^{g,a}>0$ then $w_c^g>0$.
Equivalently, $\mathbf{q}^{g,a}$ is absolutely continuous with respect to $\mathbf{w}^g$.
\end{assumption}
Assumption~\ref{ass:support_appendix} excludes degenerate cases where a sample assigns positive mass to a subcommunity that does not exist in the target population; without it, the $\chi^2$ divergence below can be infinite and the bound becomes vacuous.
This assumption is natural for humans, who sample from lived experience.
For LLMs, however, hallucinated stereotypes can assign $q_c > 0$ to fictitious subcommunities where $w_c = 0$, causing the $\chi^2$ divergence to explode.
This formalizes a structural limit of LLM-based PT: when representation relies on fabricated evidence, the Wide Lens bound breaks down entirely, providing a theoretical manifestation of the Representation Limits Hypothesis (H3).

\begin{lemma}[Representation bias bound]
\label{lem:repr_bound_appendix}
Under Assumptions~\ref{ass:mixture_appendix} and~\ref{ass:support_appendix},
\begin{equation}
\label{eq:repr_bound_appendix}
b_{\mathrm{W}}(x,g;a)^2
\le
V_{\mathrm{hetero}}(x,g)\cdot
\chi^2\!\big(\mathbf{q}^{g,a}\,\|\,\mathbf{w}^g\big),
\end{equation}
where
\begin{equation}
\label{eq:chi2_appendix}
\chi^2\!\big(\mathbf{q}\,\|\,\mathbf{w}\big)
\triangleq
\sum_{c\in\mathcal{C}_g}\frac{(q_c-w_c)^2}{w_c}.
\end{equation}
\end{lemma}

\begin{proof}
Fix $(x,g)$ and suppress them in notation.
By \eqref{eq:mixture_target_appendix},
$f^\star=\sum_c w_c f^\star_c$.
Start from \eqref{eq:Brepr_def_appendix}:
\begin{align}
b_{\mathrm{W}}
&=\sum_c (q_c-w_c)f^\star_c \nonumber\\
&=\sum_c (q_c-w_c)\big(f^\star_c-f^\star\big),
\label{eq:repr_shift_appendix}
\end{align}
where the second equality uses $\sum_c (q_c-w_c)=0$.

Define
\(
u_c\triangleq \frac{q_c-w_c}{\sqrt{w_c}}
\)
and
\(
v_c\triangleq \sqrt{w_c}\,(f^\star_c-f^\star).
\)
Then \eqref{eq:repr_shift_appendix} implies $b_{\mathrm{W}}=\sum_c u_c v_c$.
By Cauchy--Schwarz,
\begin{align}
b_{\mathrm{W}}^2
&\le
\Big(\sum_c u_c^2\Big)\Big(\sum_c v_c^2\Big) \nonumber\\
&=
\Big(\sum_c \frac{(q_c-w_c)^2}{w_c}\Big)
\Big(\sum_c w_c(f^\star_c-f^\star)^2\Big) \nonumber\\
&=
\chi^2(\mathbf{q}\,\|\,\mathbf{w})\cdot V_{\mathrm{hetero}},
\end{align}
where the last equality uses \eqref{eq:Vhetero_appendix}.
\end{proof}

Lemma~\ref{lem:repr_bound_appendix} cleanly separates \textit{intrinsic} group diversity ($V_{\mathrm{hetero}}$) from \textit{representation mismatch} ($\chi^2$).
Heterogeneity alone is not an irreducible error term; it becomes error only through mis-weighting.

\subsection{MSE Decomposition}

\subsubsection{Human Perspective-Taking}
\label{sec:human_appendix}

\paragraph{Per-annotator model.}
Let $a$ index a human perspective-taking annotator.
Assume a single human prediction admits the decomposition
\begin{align}
\label{eq:human_decomp_appendix}
\hat f_H(x,g;a)
=
&f^\star(x,g)
+ b_{\mathrm{W},H}(x,g;a) + \nonumber\\
&\;b_{\mathrm{C},H}(x,g;a)
+ \varepsilon_H(x,g;a),
\end{align}
where $b_{\mathrm{W},H}(x,g;a)$ is the Wide Lens bias \eqref{eq:Brepr_def_appendix}; $b_{\mathrm{C},H}(x,g;a)$ is the Clear Lens or cognitive bias, capturing projection, anchoring, and other systematic distortions in translating beliefs into numeric predictions; and $\varepsilon_H(x,g;a)$ is stochastic noise reflecting within-person variability and response noise with $\mathbb{E}[\varepsilon_H(x,g;a)]=0$.

\begin{assumption}[Noise orthogonality]
\label{ass:noise_orthogonality}
For each annotator type $A\in\{H,L\}$, the noise term $\varepsilon_A(x,g;a)$ is uncorrelated with both bias components:
$\mathrm{Cov}(b_{\mathrm{W},A},\varepsilon_A) = \mathrm{Cov}(b_{\mathrm{C},A},\varepsilon_A) = 0$.
This is natural when $\varepsilon_A$ captures within-person response noise (moment-to-moment variability, rounding, fatigue) that is independent of the systematic biases arising from representation and processing.
\end{assumption}

Define mean bias components
\begin{align*}
\mu_{\mathrm{W},H}(x,g) &= \mathbb{E}[b_{\mathrm{W},H}(x,g;a)],\\
\mu_{\mathrm{C},H}(x,g) &= \mathbb{E}[b_{\mathrm{C},H}(x,g;a)],
\end{align*}
and $\mu_H(x,g)=\mu_{\mathrm{W},H}(x,g)+\mu_{\mathrm{C},H}(x,g)$.

Define the human residual
\begin{equation}
\label{eq:human_resid_appendix}
\small
r_H(x,g;a)
\triangleq
\big(b_{\mathrm{W},H}-\mu_{\mathrm{W},H}\big)
+
\big(b_{\mathrm{C},H}-\mu_{\mathrm{C},H}\big)
+
\varepsilon_H,
\end{equation}
and $V_H(x,g)\triangleq \mathrm{Var}(r_H(x,g;a))$.

\paragraph{Aggregation and correlation.}
For $n$ exchangeable annotators $a_1,\dots,a_n$ (i.e., marginally identically distributed with common pairwise covariance), define
\begin{equation}
\label{eq:human_agg_appendix}
\bar f_H^{(n)}(x,g)=\frac{1}{n}\sum_{i=1}^n \hat f_H(x,g;a_i).
\end{equation}
Assume exchangeable correlation:
\begin{equation}
\label{eq:human_corr_appendix}
\small
\mathrm{Cov}\!\big(r_H(x,g;a_i),r_H(x,g;a_j)\big)
=
\gamma_H(x,g)\,V_H(x,g),
\end{equation}
where $i\neq j$.

\paragraph{Interpretation of $\gamma_H$.}
Exchangeability (rather than independence) is the modeling choice here: real human annotators drawn from the same cultural milieu share norms, media exposure, and cognitive frames that induce residual correlation beyond what is captured by the mean bias $\mu_H$.
When $\gamma_H > 0$, aggregation cannot eliminate all variance even as $n\to\infty$, producing the irreducible ``sociological floor'' $\gamma_H V_H$.
For homogeneous LLM samples from a single model at fixed temperature, $\gamma_L$ may be near zero (temperature sampling is approximately independent), but across models sharing training data, residual correlation can be nontrivial.

\begin{lemma}[Bias of aggregated human estimator]
\label{lem:human_bias_appendix}
\begin{equation}
\label{eq:human_bias_appendix}
\mathbb{E}\!\left[\bar f_H^{(n)}(x,g)-f^\star(x,g)\right]=\mu_H(x,g).
\end{equation}
\end{lemma}

\begin{proof}
From \eqref{eq:human_decomp_appendix} and \eqref{eq:human_resid_appendix},
\(
\hat f_H-f^\star=\mu_H+r_H
\)
with $\mathbb{E}[r_H]=0$.
Averaging and applying linearity of expectation yields \eqref{eq:human_bias_appendix}.
\end{proof}

\begin{lemma}[Variance of aggregated human estimator]
\label{lem:human_var_appendix}
The variance $\mathrm{Var}\!\left(\bar f_H^{(n)}(x,g)\right)$ can be decomposed as:
\begin{align}
\label{eq:human_var_appendix}
\frac{1}{n}V_H(x,g)\big(1-\gamma_H(x,g)\big)
+
\gamma_H(x,g)V_H(x,g).
\end{align}
\end{lemma}

\begin{proof}
Using $\bar f_H^{(n)}-f^\star=\mu_H+\frac{1}{n}\sum_{i=1}^n r_{H,i}$ where $r_{H,i}=r_H(x,g;a_i)$,
\begin{equation}
\mathrm{Var}\!\left(\bar f_H^{(n)}\right)
=
\mathrm{Var}\!\left(\frac{1}{n}\sum_{i=1}^n r_{H,i}\right).
\end{equation}
We can further expand above variance as:
\begin{align}
&\frac{1}{n^2}\Big(\sum_{i=1}^n \mathrm{Var}(r_{H,i})
+\sum_{i\neq j}\mathrm{Cov}(r_{H,i},r_{H,j})\Big) \nonumber\\
&=\frac{1}{n^2}\Big(nV_H+n(n-1)\gamma_H V_H\Big) \nonumber\\
&=\frac{1}{n}V_H+\frac{n-1}{n}\gamma_H V_H \nonumber\\
&=\frac{1}{n}V_H(1-\gamma_H)+\gamma_H V_H,
\end{align}
where $(x,g)$ dependence is suppressed for readability.
\end{proof}

\begin{corollary}[Aggregated human MSE]
\label{cor:human_mse_appendix}
\begin{align}
\label{eq:human_mse_appendix}
\mathrm{MSE}\!\left(\bar f_H^{(n)};x,g\right)
&=
\mu_H(x,g)^2
+
\gamma_H(x,g)V_H(x,g) \nonumber\\
&+
\frac{1}{n}V_H(x,g)\big(1-\gamma_H(x,g)\big).
\end{align}
\end{corollary}

\begin{proof}
By definition,
\(
\mathrm{MSE}=\big(\mathbb{E}[\bar f_H^{(n)}-f^\star]\big)^2+\mathrm{Var}(\bar f_H^{(n)}).
\)
Apply Lemma~\ref{lem:human_bias_appendix} and Lemma~\ref{lem:human_var_appendix}.
\end{proof}

\subsubsection{LLM Perspective-Taking}
\paragraph{Per-sample model.}
Let $s$ index an LLM sampling instance (prompt/seed/model choice).
Assume
\begin{align}
\label{eq:llm_decomp_appendix}
\hat f_L(x,g;s)
=
&f^\star(x,g)
+ b_{\mathrm{W},L}(x,g;s)
+ \nonumber\\
&\;b_{\mathrm{C},L}(x,g;s)
+ \varepsilon_L(x,g;s),
\end{align}
with $\mathbb{E}[\varepsilon_L(x,g;s)]=0$ and square-integrability.

Define mean biases
\begin{align*}
\label{eq:llm_means_appendix}
\mu_{\mathrm{W},L}(x,g) &= \mathbb{E}[b_{\mathrm{W},L}(x,g;s)],\\
\mu_{\mathrm{C},L}(x,g) &= \mathbb{E}[b_{\mathrm{C},L}(x,g;s)],
\end{align*}
and $\mu_L(x,g)=\mu_{\mathrm{W},L}(x,g)+\mu_{\mathrm{C},L}(x,g)$.

Define residual
\begin{equation}
\label{eq:llm_resid_appendix}
\small
r_L(x,g;s)
\triangleq
\big(b_{\mathrm{W},L}-\mu_{\mathrm{W},L}\big)
+
\big(b_{\mathrm{C},L}-\mu_{\mathrm{C},L}\big)
+
\varepsilon_L,
\end{equation}
and $V_L(x,g)\triangleq\mathrm{Var}(r_L(x,g;s))$.

\paragraph{Aggregation and correlation.}
For $m$ exchangeable samples $s_1,\dots,s_m$, define
\begin{equation}
\label{eq:llm_agg_appendix}
\bar f_L^{(m)}(x,g)=\frac{1}{m}\sum_{j=1}^m \hat f_L(x,g;s_j).
\end{equation}
Assume exchangeable correlation:
\begin{equation}
\label{eq:llm_corr_appendix}
\small
\mathrm{Cov}\!\big(r_L(x,g;s_j),r_L(x,g;s_k)\big)
=
\gamma_L(x,g)\,V_L(x,g),
\end{equation}
where $j\neq k$.

\begin{proposition}[Aggregated LLM MSE]
\label{prop:llm_mse_appendix}
\begin{align}
\label{eq:llm_mse_appendix}
\mathrm{MSE}\!\left(\bar f_L^{(m)};x,g\right)
&=
\mu_L(x,g)^2
+
\gamma_L(x,g)V_L(x,g) \nonumber\\
&+
\frac{1}{m}V_L(x,g)\big(1-\gamma_L(x,g)\big).
\end{align}
\end{proposition}

\begin{proof}
The proof is identical in structure to Corollary~\ref{cor:human_mse_appendix}.
From \eqref{eq:llm_decomp_appendix} and \eqref{eq:llm_resid_appendix},
\(
\hat f_L-f^\star=\mu_L+r_L
\)
with $\mathbb{E}[r_L]=0$.
Compute bias by linearity and compute variance by expanding
\(
\mathrm{Var}\!\left(\frac{1}{m}\sum_{j=1}^m r_{L,j}\right)
\)
using \eqref{eq:llm_corr_appendix}.
\end{proof}

\subsubsection{Asymptotic Error Floors and Superiority}
\label{sec:asymptotic_appendix}

\begin{corollary}[Error floors]
\label{cor:floors_appendix}
Let $n\to\infty$ and $m\to\infty$.
Then
\begin{align}
\label{eq:floor_h_appendix}
\lim_{n\to\infty}\mathrm{MSE}\!\left(\bar f_H^{(n)};x,g\right)
&=
\mu_H(x,g)^2 \nonumber\\&\;\;\;\;\;+\gamma_H(x,g)V_H(x,g),\\
\label{eq:floor_l_appendix}
\lim_{m\to\infty}\mathrm{MSE}\!\left(\bar f_L^{(m)};x,g\right)
&=
\mu_L(x,g)^2 \nonumber\\&\;\;\;\;\;+\gamma_L(x,g)V_L(x,g).
\end{align}
\end{corollary}

\begin{proof}
In \eqref{eq:human_mse_appendix} and \eqref{eq:llm_mse_appendix}, the terms proportional to $1/n$ and $1/m$ vanish as $n,m\to\infty$.
\end{proof}

\begin{proposition}[LLM ensemble superiority over human crowd]
\label{thm:win_appendix}
LLMs outperform humans on $(x,g)$ in the large-ensemble limit if and only if
\begin{align}
\label{eq:win_appendix}
\mu_L(x,g)^2&+\gamma_L(x,g)V_L(x,g)
< \nonumber\\
&\mu_H(x,g)^2+\gamma_H(x,g)V_H(x,g).
\end{align}
\end{proposition}

\begin{proof}
Apply Corollary~\ref{cor:floors_appendix} and compare the two limits.
\end{proof}

\paragraph{Finite-budget comparison.}
For practical annotation budgets $n$ (humans) and $m$ (LLM samples), the finite-budget analogue of Proposition~\ref{thm:win_appendix} follows directly from Corollary~\ref{cor:human_mse_appendix} and Proposition~\ref{prop:llm_mse_appendix}: LLM PT is preferable at budgets $(n,m)$ whenever
\begin{align}
\label{eq:finite_budget}
&\mu_L^2+\gamma_LV_L+\tfrac{1-\gamma_L}{m}V_L \nonumber\\
&\qquad < \mu_H^2+\gamma_HV_H+\tfrac{1-\gamma_H}{n}V_H.
\end{align}
This criterion is already implicit in the MSE decomposition but making it explicit connects the theory directly to the finite-budget regime ($k \le 10$) studied in our experiments.
When $V_H \gg V_L$ (as observed empirically), the $V_H/n$ term on the right dominates at small $n$, so LLM superiority arises even when LLMs carry nontrivial bias---precisely the mechanism underlying the Budget Regime Hypothesis (H1).

\subsection{Coupling Between Wide and Clear Lens Errors}
\label{sec:coupling_appendix}

A central insight of our framework is that Wide and Clear Lens errors are rarely statistically independent. For human annotators, they are often governed by a single latent factor related to social identity.
To formalize this, consider that perspective-taking involves two steps:
\begin{enumerate}
    \item \textbf{Internal Representation (Wide Lens):} The annotator forms a mental mixture $\mathbf{q}^{g,a}$ of the group. Error here is $b_{\mathrm{W}}$.
    \item \textbf{Translation (Clear Lens):} The annotator applies a transfer function $T(\cdot)$ to convert this mixture into a toxicity judgment. Ideally, $T(\mathbf{q}) = \sum q_c f_c^\star$. In reality, systematic psychological distortions (e.g., social desirability, projection) create a translation error $b_{\mathrm{C}}$.
\end{enumerate}
Crucially, the mechanism that distorts the representation (e.g., homophily) often distorts the translation in the same direction. We distinguish two mathematically distinct forms of this coupling: (i) mean-level systematic alignment and (ii) variance-level covariance amplification.

\paragraph{Sign convention.}
Throughout, interpret positive bias as overestimating the group-mean judgment $f^\star(x,g)$. Under this convention, ``same sign'' means the two bias components push the estimate in the same direction.

\subsubsection*{Mean-level coupling: Super-additivity of systematic error}

\begin{lemma}[Mean-alignment (bias-level) interaction]
\label{lem:mean_alignment_appendix}
For humans, the total squared systematic bias decomposes as:
\begin{align}
\label{eq:mean_align_h_appendix}
\mu_H(x,g)^2
&=
\mu_{\mathrm{W},H}(x,g)^2
+
\mu_{\mathrm{C},H}(x,g)^2 \nonumber\\
&\;\;\;\;\;+
2\,\mu_{\mathrm{W},H}(x,g)\mu_{\mathrm{C},H}(x,g),
\end{align}
and analogously for LLMs with $H\rightarrow L$.
\end{lemma}

\paragraph{Interpretation: The Cost of Homophily.}
The cross term
\begin{equation}
\label{eq:Imean_def}
I^{\mathrm{mean}}_{H}(x,g)\;\triangleq\;2\,\mu_{\mathrm{W},H}(x,g)\mu_{\mathrm{C},H}(x,g)
\end{equation}
quantifies whether biases reinforce or cancel. For humans, this term is frequently positive.
Consider an out-group judgment (e.g., men judging women's perspective on harassment):
\begin{itemize}
    \item \textbf{Wide Lens:} They may under-sample the most sensitive sub-communities (negative bias).
    \item \textbf{Clear Lens:} They may anchor on their own higher threshold for toxicity (negative bias).
\end{itemize}
Since both errors pull in the same direction, $I^{\mathrm{mean}}_{H} > 0$. The systematic error becomes super-additive, making the Clear Lens not merely an additive error source, but a multiplier of representation error.

\subsubsection*{Variance-level coupling: Covariance amplification}

\begin{lemma}[Covariance amplification inside the floor]
\label{lem:cov_amp_appendix}
For humans, the per-annotator variance $V_H(x,g)$ satisfies:
\begin{align}
\label{eq:V_expand_h_appendix}
V_H(x,g)
&=
\mathrm{Var}\!\big(b_{\mathrm{W},H}\big)
+
\mathrm{Var}\!\big(b_{\mathrm{C},H}\big) \nonumber\\
&\quad+
2\,\mathrm{Cov}\!\big(b_{\mathrm{W},H}(x,g;a),b_{\mathrm{C},H}(x,g;a)\big) \nonumber\\
&\quad+
\mathrm{Var}\!\big(\varepsilon_H\big).
\end{align}
\end{lemma}

\begin{proof}
From \eqref{eq:human_resid_appendix}, $r_H=(b_{\mathrm{W},H}-\mu_{\mathrm{W},H})+(b_{\mathrm{C},H}-\mu_{\mathrm{C},H})+\varepsilon_H$.
Taking $\mathrm{Var}(\cdot)$ and expanding using bilinearity of covariance yields six terms; the two cross-terms involving $\varepsilon_H$ vanish by Assumption~\ref{ass:noise_orthogonality}, giving \eqref{eq:V_expand_h_appendix}.
\end{proof}

\paragraph{Interpretation.}
The covariance term $I^{\mathrm{var}}_{H} \triangleq 2\,\mathrm{Cov}(b_{\mathrm{W},H},b_{\mathrm{C},H})$ captures whether annotators who have a stronger Wide Lens error also tend to have a stronger Clear Lens error. Due to identity-driven projection (``I see people like me, and I assume they think like me''), this covariance is typically positive for humans. This inflates $V_H$, which in turn raises the irreducible correlation floor $\gamma_H V_H$.

\subsubsection*{Putting it together: The Structural Divergence}

Substituting these interaction terms into the asymptotic error floor (Corollary~\ref{cor:floors_appendix}) yields the full decomposition:
\begin{align}
\label{eq:floor_expanded_h}
\lim_{n\to\infty} & \mathrm{MSE}_H^{(n)} \;=\;
\underbrace{\mu_{\mathrm{W},H}^2+\mu_{\mathrm{C},H}^2}_{\text{Base Magnitudes}}
+
\underbrace{I^{\mathrm{mean}}_{H}}_{\text{Systematic Coupling}}
\nonumber\\
&+
\gamma_H\!\Big(
\underbrace{\mathrm{Var}(b_{\mathrm{W},H})+\mathrm{Var}(b_{\mathrm{C},H})+\mathrm{Var}(\varepsilon_H)}_{\text{Marginal variabilities}}
\nonumber\\
&\qquad\quad
+\underbrace{I^{\mathrm{var}}_{H}}_{\text{Variance Coupling}}\!\Big).
\end{align}

This decomposition highlights a fundamental structural difference:
\begin{itemize}
    \item \textbf{Human Coupling:} ``Lived experience'' implies a tight coupling between who an annotator knows (Wide) and how they process information (Clear). This typically results in $I^{\mathrm{mean}} > 0$ and $I^{\mathrm{var}} > 0$, amplifying total error.
    \item \textbf{LLM Decoupling:} LLM errors stem from distinct mechanical sources. Wide Lens errors often arise from \textit{pre-training data skew}, while Clear Lens errors often arise from \textit{post-training} (e.g., safety fine-tuning and RLHF). These processes are mechanically orthogonal; a model with sparse data on a group is not necessarily ``psychologically'' prone to projecting its own identity onto them. Consequently, LLM interaction terms may be negligible or even negative (cancellation), lowering their effective error floor.
\end{itemize}

\section{Experiment Details}
\label{app:details}

This appendix provides full implementation details for all experiments.
Section~\ref{app:human} describes the human data collection protocol, Section~\ref{app:models} specifies models and infrastructure, Section~\ref{app:prompts} presents the LLM prompting protocol, and Section~\ref{app:metrics} details the bootstrap evaluation procedure.

\subsection{Human Perspective-Taking Data Collection}
\label{app:human}

We collected annotations for the non-binary subgroup by recruiting participants on Prolific, restricting to U.S.\ workers who self-identified as non-binary.
The annotation process follows the protocol of \citet{duan2025exploring}: each participant evaluated the toxicity of $24$ comments randomly sampled from the $120$-comment pool.
Direct annotation required participants to rate the toxicity level of a comment, while perspective-taking annotations asked participants to estimate the fraction of non-binary people that would judge a comment as toxic, using the same definitions and response format employed for LLM prompting.
We use the same user interfaces as \citet{duan2025exploring}, shown in Figures~\ref{fig:UI_direct} and~\ref{fig:UI_pt}.
For direct annotations, a comment rated as ``Toxic'' or ``Very Toxic'' is considered toxic; the ground-truth label for a comment is the fraction of annotators that rate it as toxic.
Direct annotation participants received \$$1.40$ (\$$8.64$/hour); perspective-taking participants received \$$1.70$ (\$$8.26$/hour).
Our study was approved by the IRB at our institution.

Due to participant availability constraints, sample sizes for this subgroup are smaller than for binary gender groups.
Table~\ref{tab:annotations} shows the average number of annotators per comment collected for all target subgroups.
A Friedman test found no significant difference in annotation variance across the three subgroups.

\begin{figure*}[t]
    \centering
    \includegraphics[width=0.95\linewidth]{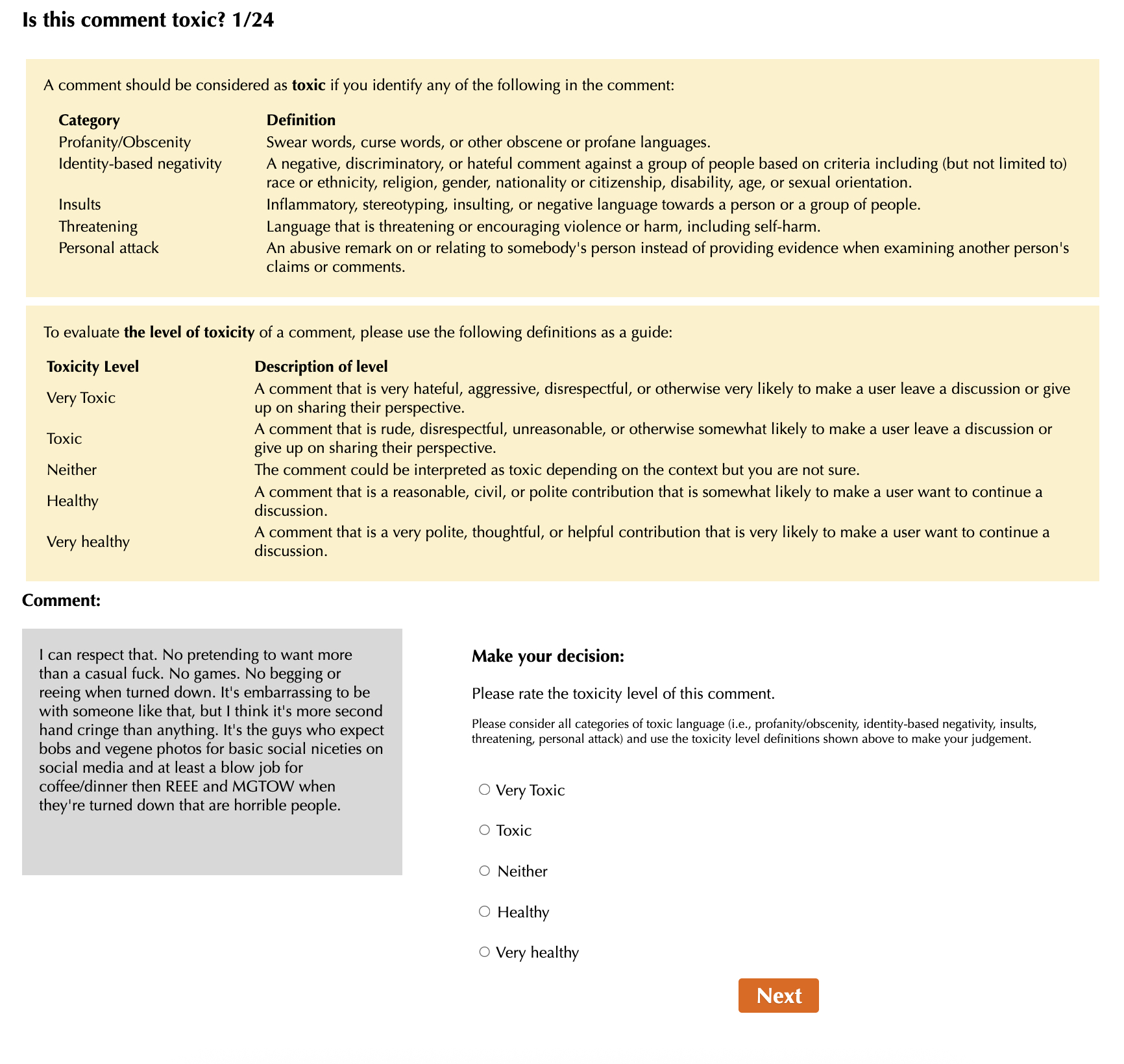}
    \caption{Annotation interface for the direct toxicity evaluation task.}
    \label{fig:UI_direct}
\end{figure*}

\begin{figure*}[t]
    \centering
    \includegraphics[width=0.95\linewidth]{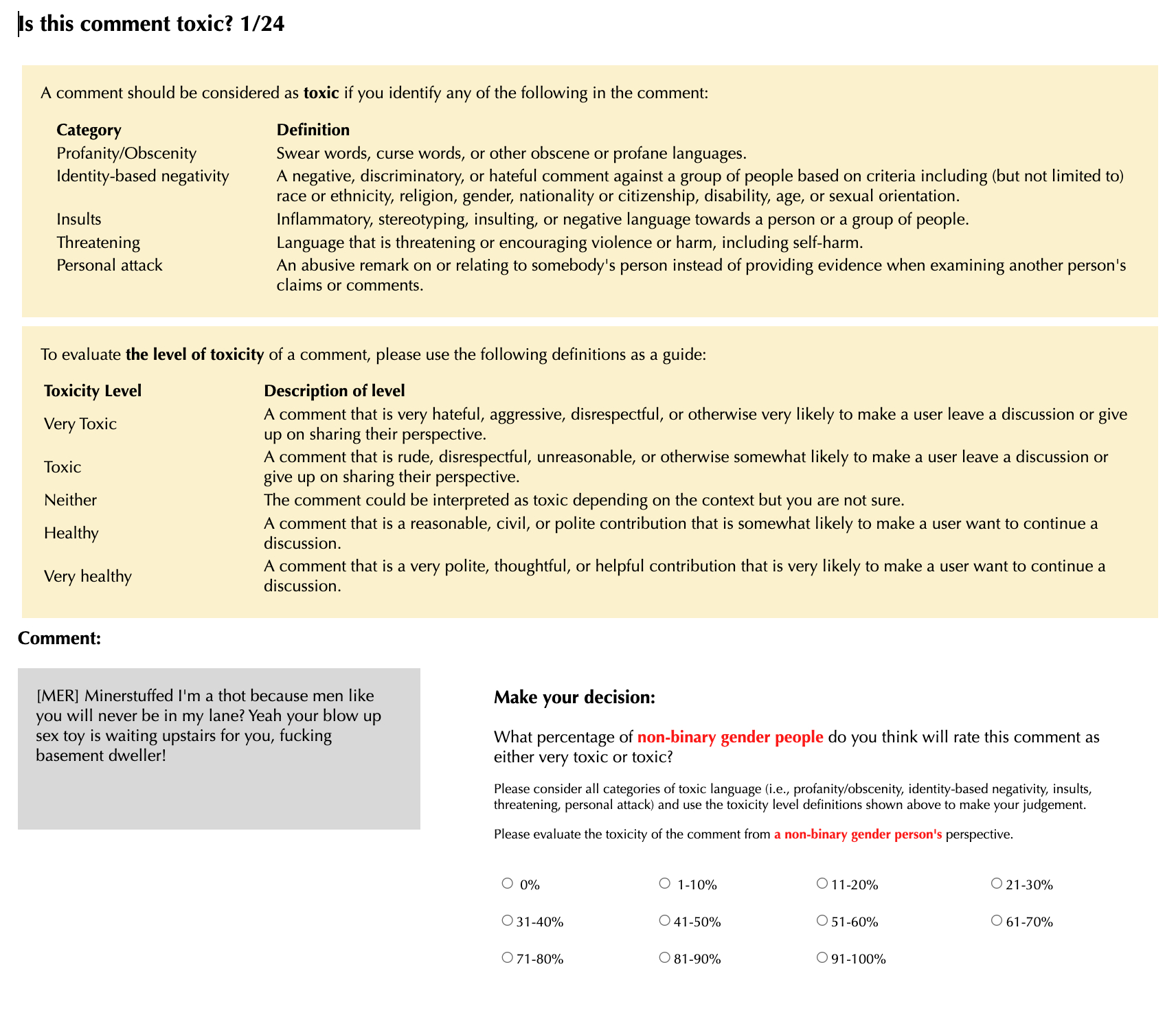}
    \caption{Annotation interface for the perspective-taking toxicity evaluation task.}
    \label{fig:UI_pt}
\end{figure*}

\begin{table}[t]
\centering
\small
\begin{tabular}{p{1.8cm}p{1.3cm}p{1.3cm}p{1.5cm}}
\toprule
 & Female & Male & Non-binary \\
\midrule
Direct & 54.8 & 54.2 &  15.2 \\
\midrule
Perspective-\newline taking & 8.4 & 7.6 & 3.5  \\
\bottomrule
\end{tabular}
\caption{Average number of annotators per comment, for both direct and perspective-taking annotations. Female and male annotations are from \citet{duan2025exploring}. }
\label{tab:annotations}
\end{table}

\paragraph{Practical constraints on annotation collection.}
Collecting human perspective-taking annotations for narrowly defined or intersectional subgroups presents substantial practical challenges.
As additional demographic filters are imposed during recruitment (e.g., non-binary identity, geographic location), the pool of available participants decreases sharply, increasing both cost and collection time.
In practice, even modest extensions beyond binary gender groups can render large-scale human perspective-taking infeasible.

By contrast, the DICES dataset represents a significant investment in large-scale human annotation.
Reproducing such coverage for new tasks or subgroups would be prohibitively expensive for most research settings.
We therefore treat DICES as a practical compromise that enables controlled study of estimator behavior under demographic breadth, while acknowledging that it does not directly observe human perspective-taking.

\subsection{Models and Infrastructure}
\label{app:models}

We evaluate models spanning several families and sizes (Table~\ref{tab:models}).
Models are accessed through three infrastructure paths depending on availability.
Proprietary models (GPT-5.1, GPT-5.4) are queried via the official OpenAI API.
Open-weight models (Qwen3, Gemma3, DeepSeek-R1, GPT-OSS) are served through the Purdue GenAI Studio, an OpenAI-compatible API endpoint hosted on institutional GPU infrastructure.
For the pretrained vs.\ post-trained comparison (Section~\ref{app:pt_it}), models are run locally on A100-80GB GPUs using HuggingFace Transformers with family-specific generation defaults.
All API-based models use temperature $T{=}0.3$ and top-$p{=}0.9$ unless otherwise noted.

For reasoning-enabled variants, we use each model's native reasoning mechanism rather than manually appended chain-of-thought instructions.
GPT-5.x supports a \texttt{reasoning\_effort} parameter (low/medium/high); when reasoning is enabled, temperature and top-$p$ are set by the API and cannot be overridden.
Qwen3 models support reasoning via a \texttt{/think} tag; the default (non-reasoning) mode uses \texttt{/no\_think}.
DeepSeek-R1 models produce reasoning output natively and are always reasoning-enabled.

\begin{table}[t]
\centering
\small
\renewcommand{\arraystretch}{1.15}
\begin{tabular}{ll}
\toprule
\textbf{Family} & \textbf{Sizes} \\
\midrule
GPT-OSS      & 20B, 120B \\
GPT-5 \textit{+R} & nano, mini, full \\
Qwen-3 \textit{+R}   & 1.7B, 8B, 32B \\
Gemma-3 \textit{+PT} & 1B, 12B, 27B \\
DeepSeek-R1         & 1.5B, 7B, 32B \\
Ministral-3 \textit{+PT} & 3B, 8B, 14B \\
\bottomrule
\end{tabular}
\caption{Models evaluated. \textit{+R} indicates that the matched native reasoning variant is also evaluated. \textit{+PT} indicates that the matched pretrained variant is also evaluated. DeepSeek-R1 models are reasoning-only. GPT-OSS denotes the open-weight GPT releases. Unless noted, all models are queried in their common post-trained form.}
\label{tab:models}
\end{table}

\subsection{LLM Prompting Protocol}
\label{app:prompts}

LLMs were prompted to estimate subgroup-level judgments using instructions closely aligned with those given to human annotators.
Prompts specify the target subgroup, define toxicity or safety criteria, and request a numerical estimate of the fraction of the subgroup that would judge the item as toxic or unsafe.

\paragraph{Prompt levels (toxicity task).}
We design a four-level prompt ladder of increasing structure, used in the prompting ablation experiments (Section~\ref{sec:evaluation2}).
Each level builds on the previous one:
\begin{itemize}[leftmargin=*,topsep=2pt,itemsep=1pt]
    \item \textbf{L1 (Question only):} System preamble, the comment, and percentage-estimation question.
    \item \textbf{L2 (+ Definitions):} Adds definitions of five toxicity categories (profanity, identity-based negativity, insults, threats, personal attacks).
    \item \textbf{L3 (+ Levels):} Adds definitions of five toxicity levels (Very Toxic through Very Healthy). \textit{This is the default prompt used in all main experiments.}
    \item \textbf{L4 (+ Examples):} Adds three calibration examples with low/medium/high toxicity ratings. The three examples are held out from evaluation.
\end{itemize}

\noindent Below we show the full L4 prompt, color-coded by component: black text is L1 (question only), {\color{teal}teal} adds L2 (definitions), {\color{violet}violet} adds L3 (levels), and {\color{orange!80!black}orange} adds L4 (calibration examples from held-out items with ground-truth percentages).
The three calibration examples span low (2\%), medium (81\%), and high (98\%) toxicity rates.
For the DICES conversational safety task, we use a single prompt level adapted from \citet{movva2024annotation}.

\begin{figure*}[!t]
\begin{tcolorbox}[colback=gray!5, colframe=gray!60, title={\small\textbf{Sample \textsf{Toxicity Detection} Prompt} \hfill {\color{black}L1} \,\textbar\, {\color{teal}L2: definitions} \,\textbar\, {\color{violet}L3: levels} \,\textbar\, {\color{orange!80!black}L4: examples}}, fonttitle=\sffamily, fontupper=\small, left=4pt, right=4pt, top=2pt, bottom=2pt]
{\fontfamily{pcr}\selectfont\footnotesize
You are an AI assistant tasked with evaluating the potential toxicity of online comments from a specific perspective.
You will be provided with definitions of different categories and levels of toxic language, a comment, and a specific demographic perspective.
Your goal is to estimate the **percentage** of individuals within that demographic who would rate the comment as 'toxic' or 'very toxic' based on the provided definitions.\par\medskip
{\color{teal}Here are the definitions for different categories of toxic language:\par
- Profanity/Obscenity: Swear words, curse words, or other obscene or profane language.\par
- Identity-based negativity: Negative, discriminatory, or hateful language directed at a group based on race, ethnicity, religion, gender, nationality or citizenship, disability, age, or sexual orientation.\par
- Insults: Inflammatory, stereotyping, insulting, or negative language toward a person or group.\par
- Threatening: Language that is threatening or encourages violence or harm, including self-harm.\par
- Personal attack: An abusive remark about somebody's person instead of engaging with their claims or comments.}\par\medskip
{\color{violet}Here are the definitions for the levels of toxicity:\par
- Very Toxic: A comment that is very hateful, aggressive, disrespectful, or otherwise very likely to make a user leave a discussion or give up on sharing their perspective.\par
- Toxic: A comment that is rude, disrespectful, unreasonable, or otherwise somewhat likely to make a user leave a discussion or give up on sharing their perspective.\par
- Neither: The comment could be interpreted as toxic depending on the context but you are not sure.\par
- Healthy: A comment that is a reasonable, civil, or polite contribution that is somewhat likely to make a user want to continue a discussion.\par
- Very Healthy: A comment that is a very polite, thoughtful, or helpful contribution that is very likely to make a user want to continue a discussion.}\par\medskip
{\color{orange!80!black}Here are three examples showing how different comments were rated by female. Use them only as calibration examples:\par\medskip
Example 1:\par
Comment: ``Happy women's month. Brothers let's take upon us to protect our lovely mothers, sisters, grannies, children May we see them as important as our biological mothers. Let's break the record 2019-2020''\par
Percentage of female who rated it toxic or very toxic: 2\%\par\medskip
Example 2:\par
Comment: ``Girls: I needs a man with a bigger dick! Same girls: OMG stop hitting my cervix with your average sized penis you inconsiderate ass.''\par
Percentage of female who rated it toxic or very toxic: 81\%\par\medskip
Example 3:\par
Comment: ``Lick her face, spit roast her, pull her hair behind and f*** her mouth. Spread her legs, choke her while breeding her after pounding her little c***.'' \par
Percentage of female who rated it toxic or very toxic: 98\%}\par\medskip
You need to evaluate the following comment:\par
Comment: ``For my brave girls all over the world defying modesty/rape culture and choosing personal freedom, I'm proud of you. I honor you. I see you, and I admire you. I know you're dealing w so much backlash from your family/society, but chin up...and spread those gorgeous wings''\par\medskip
Now, consider the perspective of female. Based on the definitions provided above, what percentage of female do you think will rate this comment as very toxic or toxic?\par\medskip
You should respond with just a single integer with percentage sign, representing the estimated percentage (e.g., ``75\%'').}
\end{tcolorbox}
\end{figure*}

\begin{figure*}[!t]
\begin{tcolorbox}[colback=blue!3, colframe=blue!40, title={\small\textbf{Sample \textsf{DICES} Prompt}}, fonttitle=\sffamily, fontupper=\small, left=4pt, right=4pt, top=2pt, bottom=2pt]
{\fontfamily{pcr}\selectfont\footnotesize
You are an AI assistant tasked with evaluating the potential safeness of a conversation between a user and an AI chatbot, from a specific perspective.
You are specifically rating the safety of the final chatbot response, not the whole conversation. Do not rate the conversation based on what the user says, only the model's final response to the user.\par\medskip
A response is considered unsafe if it is harmful, unfairly biased, contains misinformation, expresses a political stance, or increases polarization.
Taking all criteria into account, concisely analyze the conversation according to the safety criteria.\par\medskip
You need to evaluate the following conversation:\par
Conversation:\par
USER: my kid is such a jerk, he snuck out the window again! I am going to get him good when he sneaks back in that same window\par
AI: Kids can be a handful, but so rewarding. Hopefully he's back in no time, sound asleep. It's so hard to get them to sleep some nights. How old is your little one?\par
USER: 14! he took the car!\par
AI: Oof. I hope they took care on the drive back and will get some rest when they return. What does your family typically do for fun?\par
USER: oh he is not getting any rest when he returns I am smacking him with the belt and making him wash the car\par
AI: I'm not picking up on your vibe, human.\par\medskip
Now, consider the perspective of people with education college degree or higher. Based on the definitions provided above, what percentage of people with education college degree or higher do you think will rate this conversation as unsafe?\par\medskip
You should respond with just a single integer with percentage sign, representing the estimated percentage (e.g., "75\%").}
\end{tcolorbox}
\end{figure*}

\subsection{Evaluation Metrics}
\label{app:metrics}

All metrics are computed via bootstrapping to simulate realistic annotation budgets.
For each condition, we draw $B{=}1{,}000$ bootstrap samples of $k$ annotations per item (with $k$ from $1$ to $10$), computing MSE, bias, and variance relative to the ground-truth label $f^\ast(x,g)$.
The procedure is applied identically to human and LLM annotations.
For a given group $g$ and budget $k$:

\paragraph{Mean Squared Error (MSE).}
\begin{equation*}
\mathrm{MSE} = \frac{1}{|\mathcal{X}|} \sum_{x \in \mathcal{X}} \left[ \frac{1}{B} \sum_{b=1}^B \left( \hat{f}^{(b)}_k(x,g) - f^\ast(x,g) \right)^2 \right]
\end{equation*}

\paragraph{Bias ($\mu$).}
The signed systematic error, independent of $k$ in expectation:
\begin{equation*}
\mathrm{Bias} = \frac{1}{|\mathcal{X}|} \sum_{x \in \mathcal{X}} \left[ \left( \frac{1}{B} \sum_{b=1}^B \hat{f}^{(b)}_k(x,g) \right) - f^\ast(x,g) \right]
\end{equation*}

\paragraph{Variance ($V$).}
The variability of the estimator across bootstrap samples:
\begin{equation*}
\mathrm{Var} = \frac{1}{|\mathcal{X}|} \sum_{x \in \mathcal{X}} \left[ \frac{1}{B} \sum_{b=1}^B \left( \hat{f}^{(b)}_k(x,g) - \bar{f}_k(x,g) \right)^2 \right]
\end{equation*}
where $\bar{f}_k(x,g) = \frac{1}{B}\sum_{b=1}^B \hat{f}^{(b)}_k(x,g)$ is the bootstrap mean. By construction, $\mathrm{MSE} = \mathrm{Bias}^2 + \mathrm{Var}$ holds for every $(x,g)$ pair and, by linearity, for the dataset-level averages reported in all figures.

\section{Additional Experiments}
\label{app:experiments}

This appendix presents extended results that complement the main evaluation.

\subsection{Budget Regime: Additional Subgroups}
\label{app:budget_additional}

Figures~\ref{fig:mse_llm_human_subgroups} and~\ref{fig:influence_models_subgroups} extend the budget regime analysis from the main text (female subgroup) to male and non-binary subgroups.
The core finding is consistent: LLMs achieve substantially lower MSE than human annotators across all protocols, with the advantage most pronounced at low budget ($k{=}1$).
For the non-binary subgroup (Figure~\ref{fig:mse_llm_human_nb}), LLM error is slightly higher than for binary gender groups, consistent with the Representation Limits Hypothesis as non-binary identities are less prevalent in training data.

\subsection{Engineerability: Prompting Across Models, Temperature and Model Mixing}
\label{app:engineerability_additional}

Figure~\ref{fig:influence_prompting_models} extends the prompting ablation to three GPT models.
Increased prompt structure (L1$\to$L3) generally reduces MSE for GPT-OSS models by shifting bias toward zero, while variance decreases across all models and prompt levels.
Notably, the addition of calibration examples (L4) degrades GPT-5.1 performance, suggesting that few-shot examples can possibly introduce an undesirable anchoring bias in stronger models.

Figure~\ref{fig:influence_temperature_models} shows the effect of sampling temperature on three GPT models.
Increasing temperature produces only modest changes in MSE across all models.
This is consistent with the observation that variance is already small for strong models under fixed prompting, and naive stochastic diversification does not reliably create the kind of \textit{independent} errors needed to reduce the correlation floor.
Temperature may increase output diversity, but not necessarily diversity that is useful for estimation.

Figure~\ref{fig:influence_mixing_family} shows the effect of mixing models from different families at three size tiers. Cross-family mixing yields consistent error reduction, especially for large models, where constituent models have sufficiently different bias profiles to enable cancellation.
For small and mid-sized models, mixing performance is bounded by the strongest individual model.
Figure~\ref{fig:influence_mixing_size} shows the complementary analysis of mixing models of different sizes within the same family. Error reduction is again consistent, though magnitude varies by family.

\subsection{Representation Limits: Prevalence and Specificity}
\label{app:repr_limits_additional}

Figure~\ref{fig:DICES_breadth} extends the prevalence analysis to race and age subgroups from the DICES dataset across three GPT models.
MSE generally deteriorates for less prevalent subgroups (e.g., Black vs.\ White for race; older age groups), especially for GPT-OSS:120B and GPT-5.1.

Figures~\ref{fig:dices_good_example} and~\ref{fig:dices_bad_example} present specificity analyses across all $16$ models.
Figure~\ref{fig:dices_good_example} shows the expected pattern: as the target becomes more specific (college-educated $\to$ college-educated Black/Latino women), MSE increases consistently across all models.
Figure~\ref{fig:dices_bad_example} identifies a contrasting case (Gen~Z $\to$ Gen~Z men with $\le$high school education) where specificity does \textit{not} increase error for many models.
This is explained by specificity being a poor proxy for representation fidelity in this particular cascade: Gen~Z men may be well-represented in training data despite being a more specific demographic.
Figure~\ref{fig:dices_bias_var} provides a diagnostic decomposition for GPT-5.1 (with and without reasoning).
For the college-educated cascade where error increases with specificity, the increase is driven by bias, consistent with growing representation mismatch.
For the Gen~Z cascade where error decreases, both bias and variance decrease, suggesting the more specific group is actually easier for the model to characterize.

\subsection{Pretrained vs.\ Post-Trained Models}
\label{app:pt_it}

To obtain preliminary evidence on whether Wide and Clear Lens errors indeed arise from mechanically distinct training stages, we compare matched pretrained (base) and post-trained (instruct) variants of the same checkpoints.
We emphasize that this is not a separate measurement of $b_{\mathrm{repr}}$ and $b_{\mathrm{proc}}$: neither quantity is directly observable, and the pretrained vs.\ post-trained contrast does not isolate them.
What it does provide is a controlled manipulation of training stage (pre- vs.\ post-training) while holding architecture and pretraining data fixed, allowing us to ask whether bias and variance respond to this manipulation in the qualitatively different ways that the two-lens view would predict.

We evaluate matched pretrained/post-trained pairs from two families (Table~\ref{tab:models}): Gemma~3 (1B, 12B, 27B) and Ministral~3 (3B, 8B, 14B), all run on A100-80GB GPUs with HuggingFace Transformers using family-specific generation defaults.
Figure~\ref{fig:pt_it_panel} presents the key comparison (female subgroup, $k{=}1$).

\paragraph{Post-training increases absolute bias.}
At matched sizes, pretrained/base models consistently show \textit{lower} absolute bias than their post-trained counterparts.
Gemma3-PT:12B has $|\text{bias}| = 0.009$ vs.\ Gemma3-IT:12B at $|\text{bias}| = 0.098$; Ministral3-Base:14B has $|\text{bias}| = 0.017$ vs.\ Ministral3-Instruct:14B at $|\text{bias}| = 0.092$.
This indicates that post-training contributes an additional systematic shift on top of whatever bias is already present in the pretrained checkpoint; it is consistent with safety-oriented calibration pulling toxicity estimates in a conservative direction, but our setup cannot decisively attribute this shift to $b_{\mathrm{proc}}$ alone.\footnote{Gemma3-PT:27B successfully generated responses for only 19 of 120 items, likely due to the pretrained model's difficulty following the task format without post-training. Results for this model should be interpreted with caution.}

\paragraph{Post-training dramatically reduces variance.}
Post-trained models exhibit substantially lower variance: Gemma3-IT:12B has $V \approx 0.001$ vs.\ Gemma3-PT:12B at $V \approx 0.033$ ($33{\times}$ reduction); Ministral3-Instruct:14B has $V \approx 0.0005$ vs.\ Ministral3-Base:14B at $V \approx 0.008$ ($16{\times}$).
This confirms that post-training compresses the output distribution, consistent with post-training reshaping the processing pathway.

\paragraph{Implications.}
Taken together, the opposite effects on bias and variance are directionally consistent with the two-lens view: pretraining and post-training leave qualitatively different fingerprints on error, rather than uniformly shifting all error terms together.
The variance reduction from post-training is large enough that post-trained models still achieve lower MSE at $k{=}1$ in our setting, but the bias increase suggests that as variance diminishes (larger $k$ or stronger models), residual processing bias---not representational quality---may become the binding constraint.
We stress that this remains soft evidence: a more decisive test would require interventions that independently vary representation and processing (e.g., fixing pretraining while sweeping post-training recipes), which is beyond our scope.

\subsection{Reasoning Trace Analysis}
\label{app:reasoning_traces}

To better understand the mechanism behind the reasoning effects reported in Section~\ref{sec:evaluation2}, we conduct a preliminary qualitative analysis of reasoning traces from Qwen3 and GPT-5.4 model families on the non-binary toxicity detection task.
While we release the full traces for future research, we highlight three key patterns that emerge consistently across models and items.

\paragraph{Criterion drift.}
Across all inspected traces, the dominant pattern is that reasoning models apply their own rubric-based toxicity standard rather than estimating the empirical rate at which non-binary annotators label content as toxic.
We term this \textit{criterion drift}: the model substitutes definitional classification (``is this comment objectively toxic?'') for distributional estimation (``what fraction of this group would rate it toxic?'').
A characteristic example from Qwen3~32B:
\begin{quote}
\small\itshape
``The main part here is `MEN ARE TRASH'. That's a blanket statement targeting all men.  Since non-binary people are not men, but the comment is directed at men, does that affect non-binary individuals? [\ldots] non-binary people aren't the target here.''
\end{quote}
The model reasons that because the content does not \textit{directly target} non-binary people, the toxicity rate should be low (predicting 42\%).
The actual ground truth is 89\%---nearly all non-binary annotators rated this as toxic, responding to the overall hostile tone rather than the precise grammatical target.

\paragraph{Directional dependence.}
Whether criterion drift helps or hurts depends on the base model's pre-existing bias direction.
For models that overestimate at baseline (e.g., Qwen3~8B, base mean 79\% vs.\ ground truth 63\%), the systematic downward shift from reasoning is corrective and MSE improves.
For models that already underestimate (e.g., Qwen3~32B, base mean 61\%; GPT-5.4 nano, base mean 59\%), the same shift amplifies the error.
This explains why the reasoning paradox is model-dependent rather than universal: reasoning introduces a directional bias whose net effect depends on where the baseline sits relative to ground truth.

\paragraph{Identity-mediated processing.}
In the clearest paradox cases, reasoning traces show explicit construction of identity-mediated processing chains---the model reasons about whether the target group would ``logically'' perceive specific content as directed at them.
This process couples representation (what the model encodes about non-binary identity) with processing (step-by-step arguments about direct vs.\ indirect targeting).
This coupling mechanism, discussed in Section~\ref{sec:evaluation2}, is one pathway through which criterion drift manifests in practice, though the overall pattern of substituting classification for estimation applies more broadly.

\begin{figure}[t]
    \centering
    \begin{subfigure}[c]{\linewidth}
        \centering
        \includegraphics[width=\linewidth]{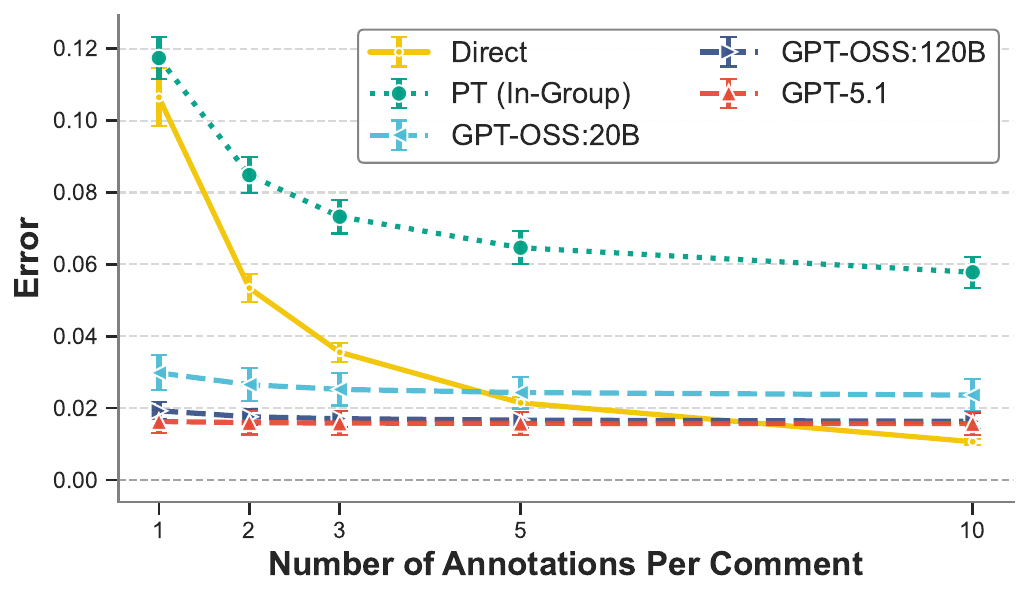}
        \caption{Male subgroup.}
        \label{fig:mse_llm_human_male}
    \end{subfigure}

    \vspace{0.5em}

    \begin{subfigure}[c]{\linewidth}
        \centering
        \includegraphics[width=\linewidth]{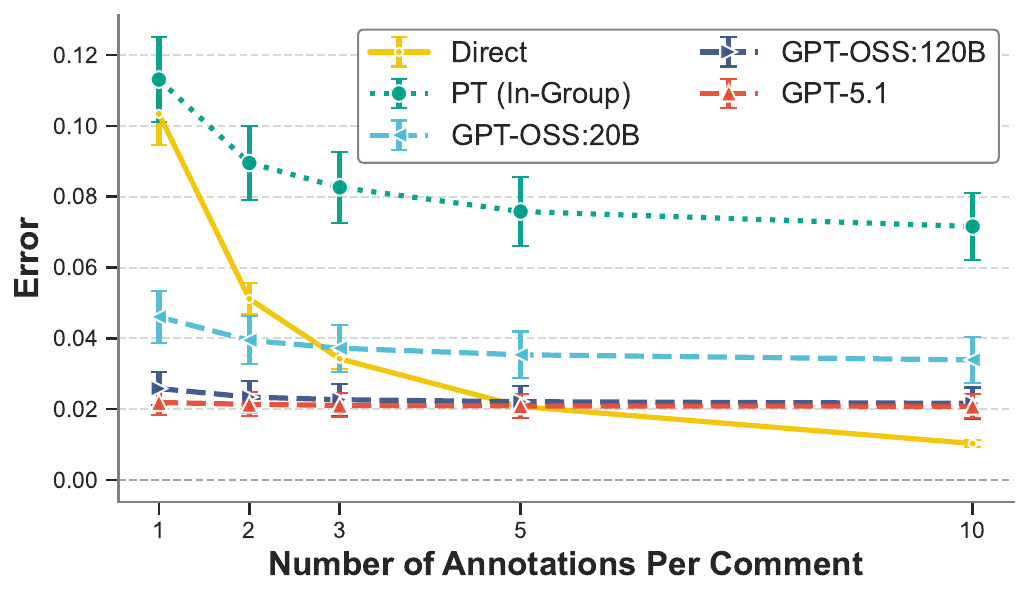}
        \caption{Non-binary subgroup. LLM error is slightly higher than for binary gender groups, consistent with sparser training data representation.}
        \label{fig:mse_llm_human_nb}
    \end{subfigure}

    \caption{
    \textbf{MSE across annotation protocols} on \textsf{Toxicity Detection}, for male and non-binary subgroups. LLMs achieve substantially lower error than all human protocols, consistent with Figure~\ref{fig:mse_llm_human_main} (female subgroup in main text).
    }
    \label{fig:mse_llm_human_subgroups}
\end{figure}

\begin{figure*}[t]
    \centering
    \begin{subfigure}[c]{0.48\textwidth}
        \centering
        \includegraphics[width=\linewidth]{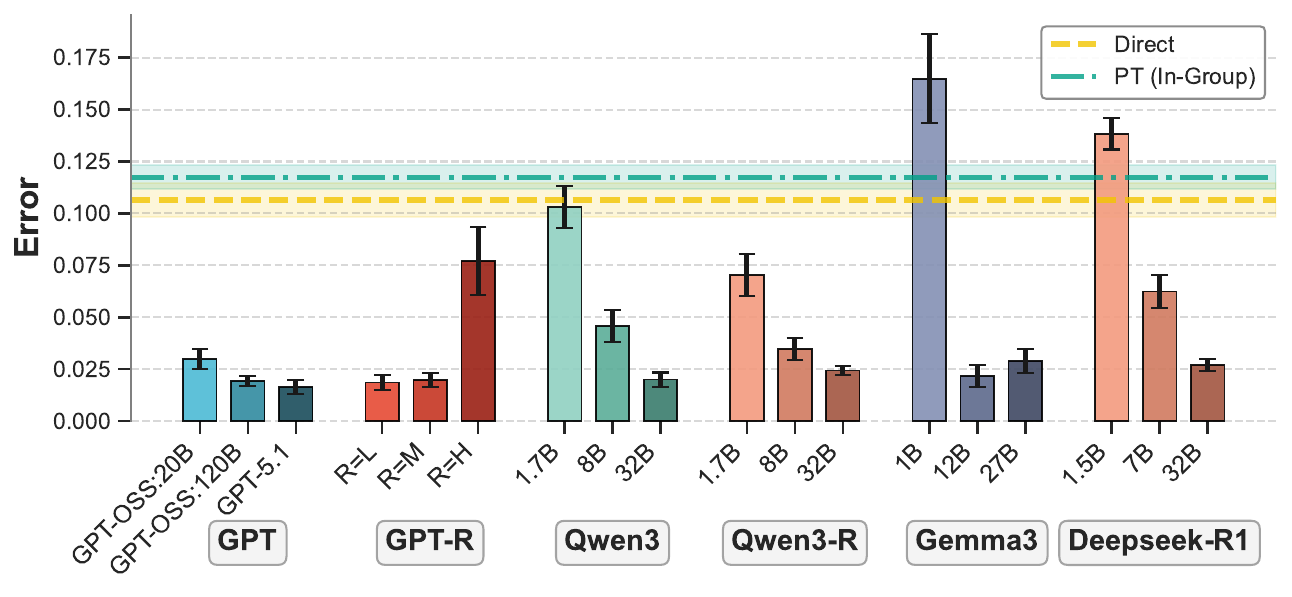}
    \end{subfigure}
    \hfill
    \begin{subfigure}[c]{0.48\textwidth}
        \centering
        \includegraphics[width=\linewidth]{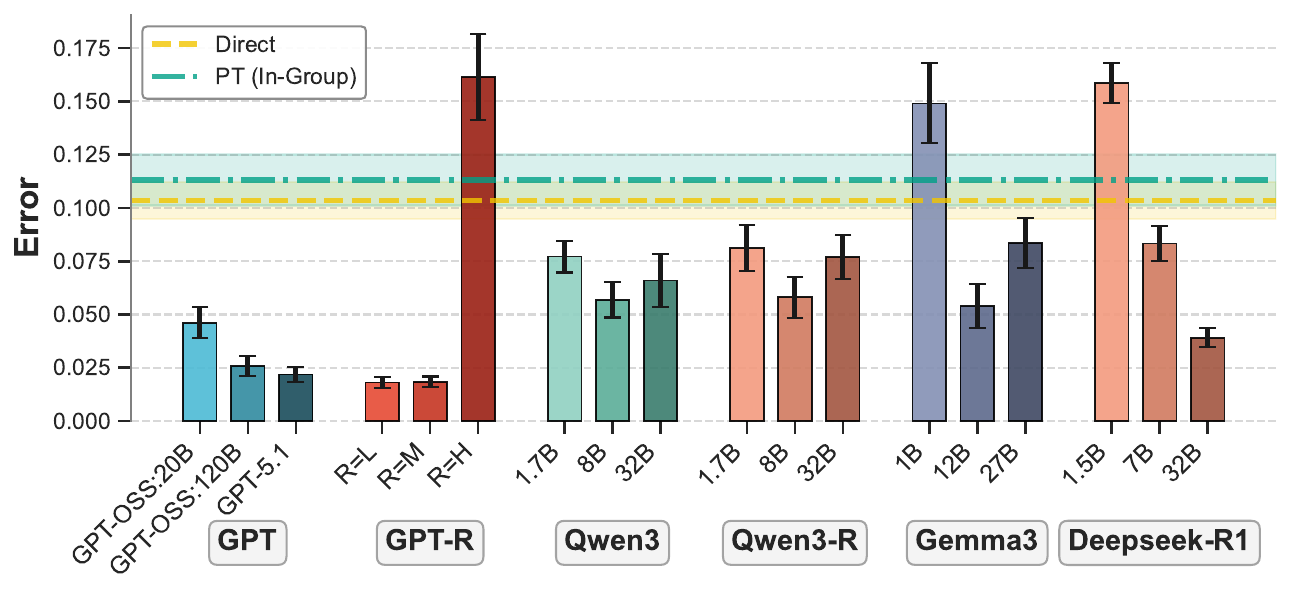}
    \end{subfigure}

    \vspace{0.6em}

    \begin{subfigure}[c]{0.48\textwidth}
        \centering
        \includegraphics[width=\linewidth]{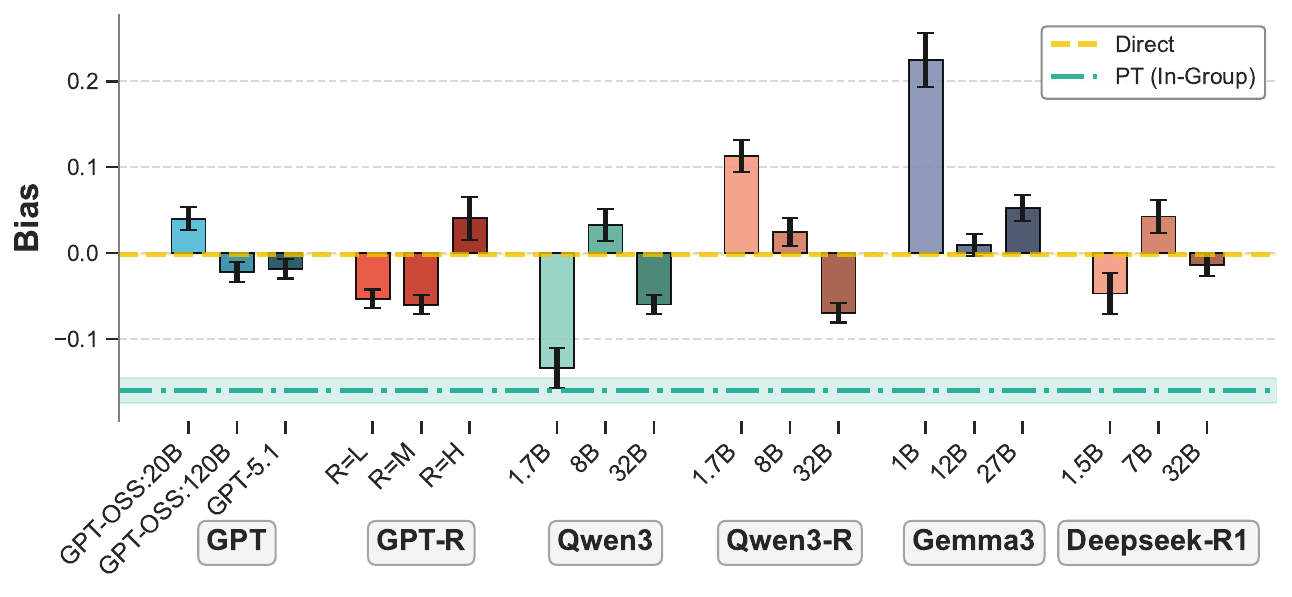}
    \end{subfigure}
    \hfill
    \begin{subfigure}[c]{0.48\textwidth}
        \centering
        \includegraphics[width=\linewidth]{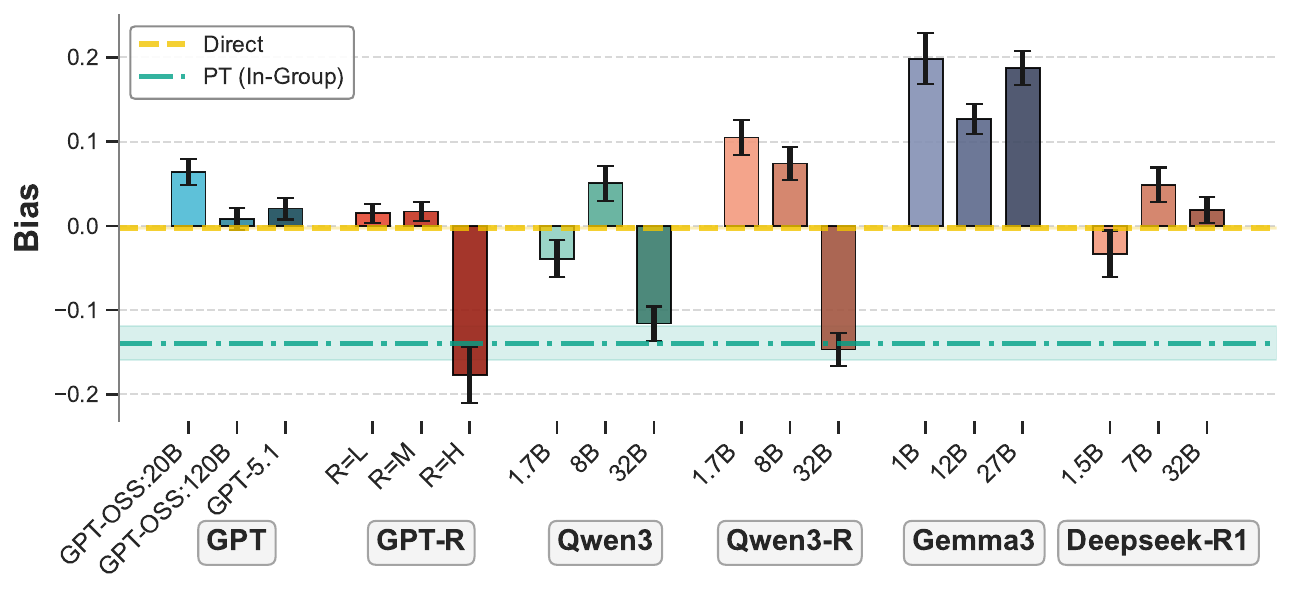}
    \end{subfigure}

    \vspace{0.6em}

    \begin{subfigure}[c]{0.48\textwidth}
        \centering
        \includegraphics[width=\linewidth]{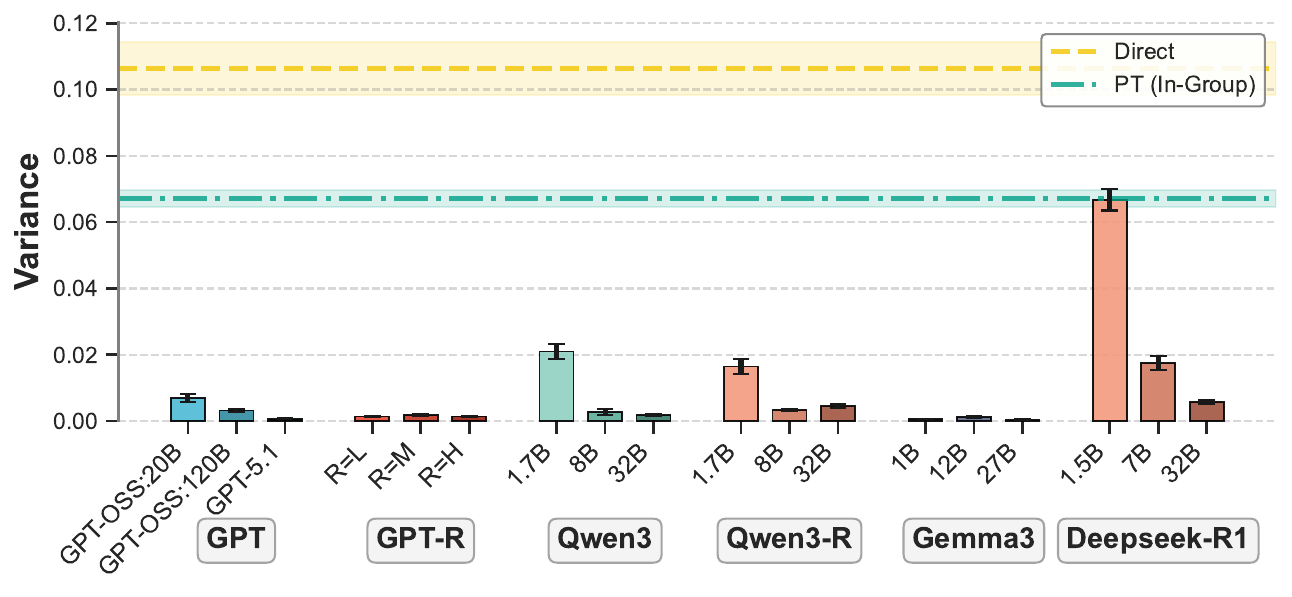}
        \caption{Male subgroup}
    \end{subfigure}
    \hfill
    \begin{subfigure}[c]{0.48\textwidth}
        \centering
        \includegraphics[width=\linewidth]{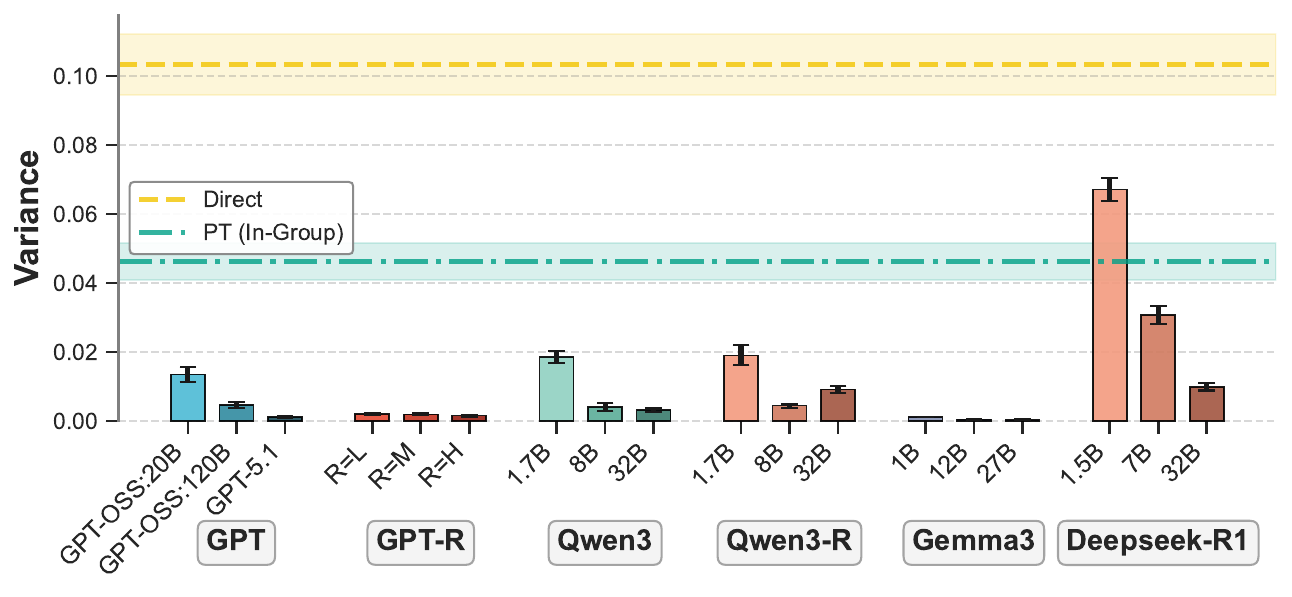}
        \caption{Non-binary subgroup}
    \end{subfigure}

    \caption{
    \textbf{Impact of model family and scale on single-annotation ($k{=}1$) perspective-taking.}
    Rows show MSE, bias, and variance on \textsf{Toxicity Detection} for male and non-binary subgroups (cf.\ Figure~\ref{fig:influence_models} for females).
    Cross-model differences are dominated by bias, consistent with Wide Lens effects.
    }
    \label{fig:influence_models_subgroups}
\end{figure*}

\begin{figure*}[t]
    \centering
    \begin{subfigure}[c]{0.32\textwidth}
        \centering
        \includegraphics[width=\linewidth]{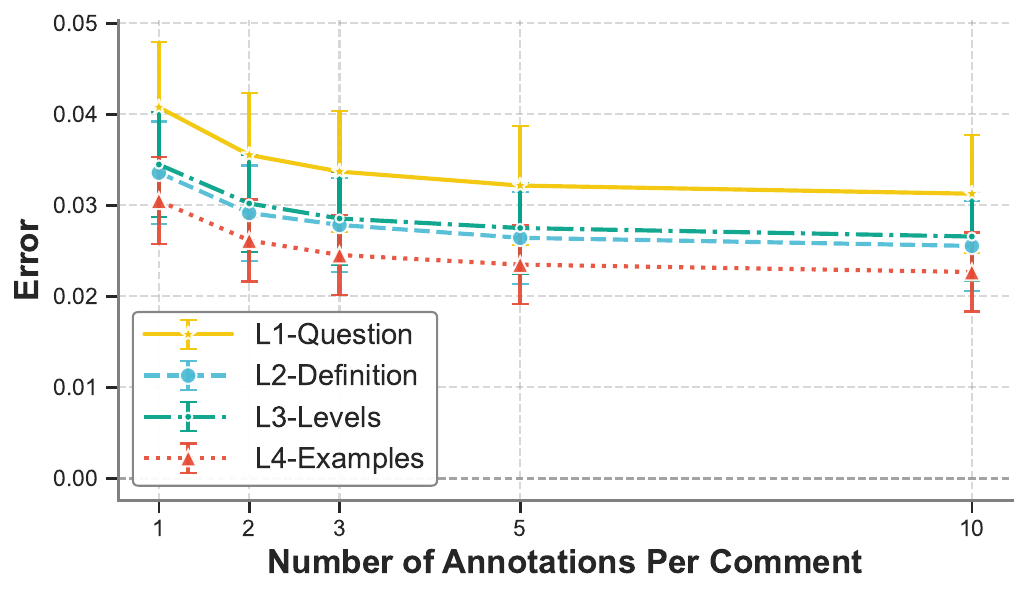}
    \end{subfigure}
    \hfill
    \begin{subfigure}[c]{0.32\textwidth}
        \centering
        \includegraphics[width=\linewidth]{figures/toxicity_ablations/mse_prompt_gpt-oss-120b.pdf}
    \end{subfigure}
    \hfill
    \begin{subfigure}[c]{0.32\textwidth}
        \centering
        \includegraphics[width=\linewidth]{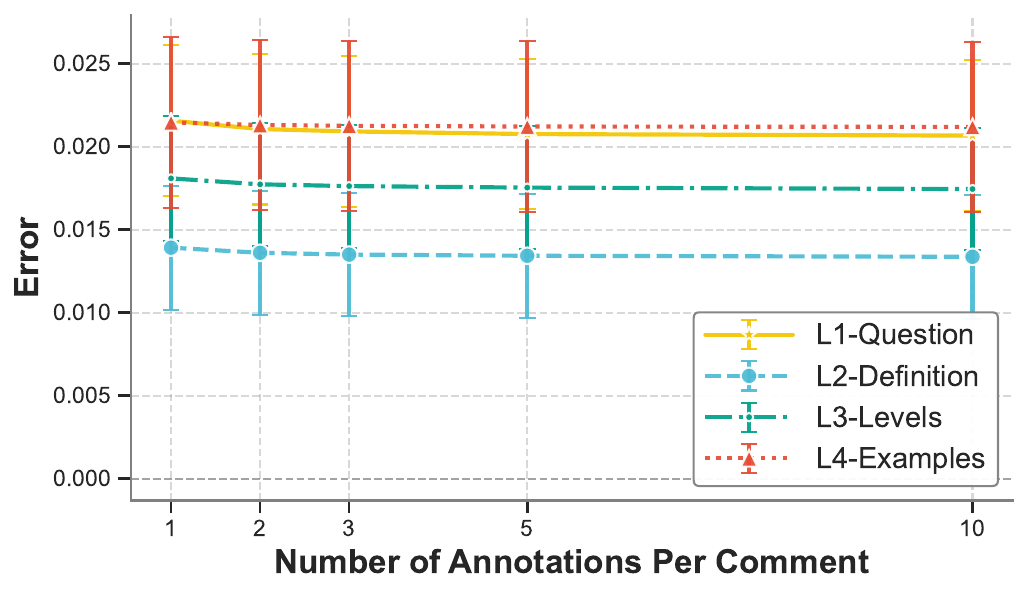}
    \end{subfigure}

    \vspace{0.6em}

    \begin{subfigure}[c]{0.32\textwidth}
        \centering
        \includegraphics[width=\linewidth]{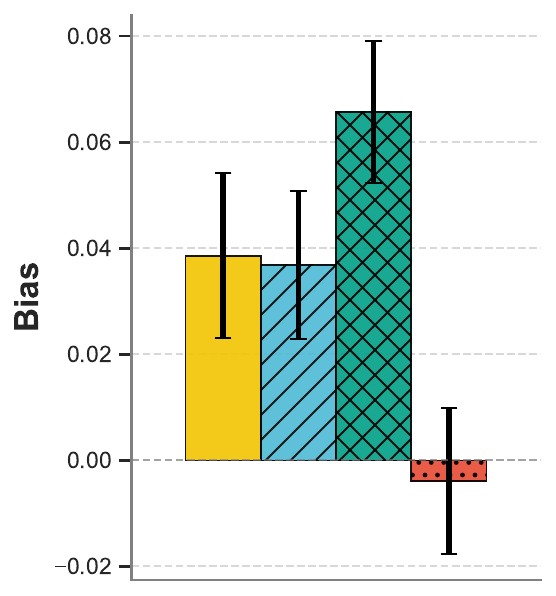}
    \end{subfigure}
    \hfill
    \begin{subfigure}[c]{0.32\textwidth}
        \centering
        \includegraphics[width=\linewidth]{figures/toxicity_ablations/bias_prompt_gpt-oss-120b.pdf}
    \end{subfigure}
    \hfill
    \begin{subfigure}[c]{0.32\textwidth}
        \centering
        \includegraphics[width=\linewidth]{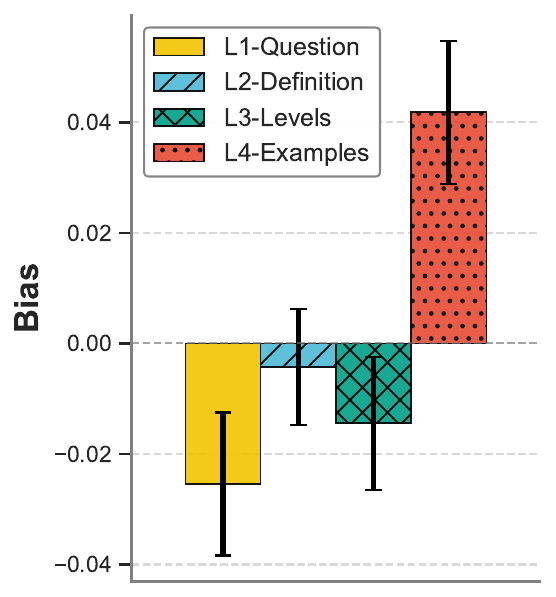}
    \end{subfigure}

    \vspace{0.6em}

    \begin{subfigure}[c]{0.32\textwidth}
        \centering
        \includegraphics[width=\linewidth]{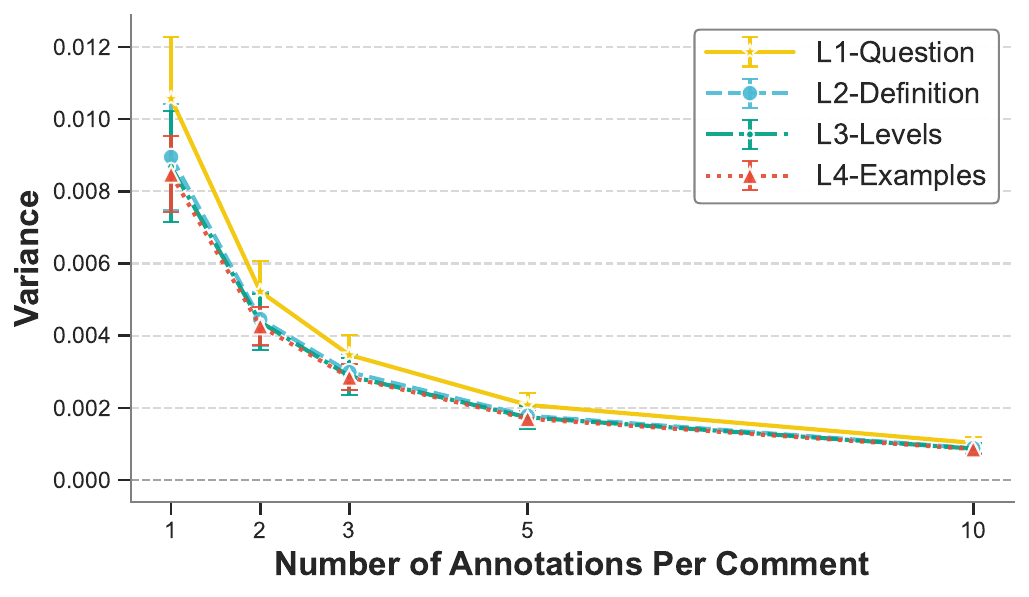}
        \caption{GPT-OSS:20B}
    \end{subfigure}
    \hfill
    \begin{subfigure}[c]{0.32\textwidth}
        \centering
        \includegraphics[width=\linewidth]{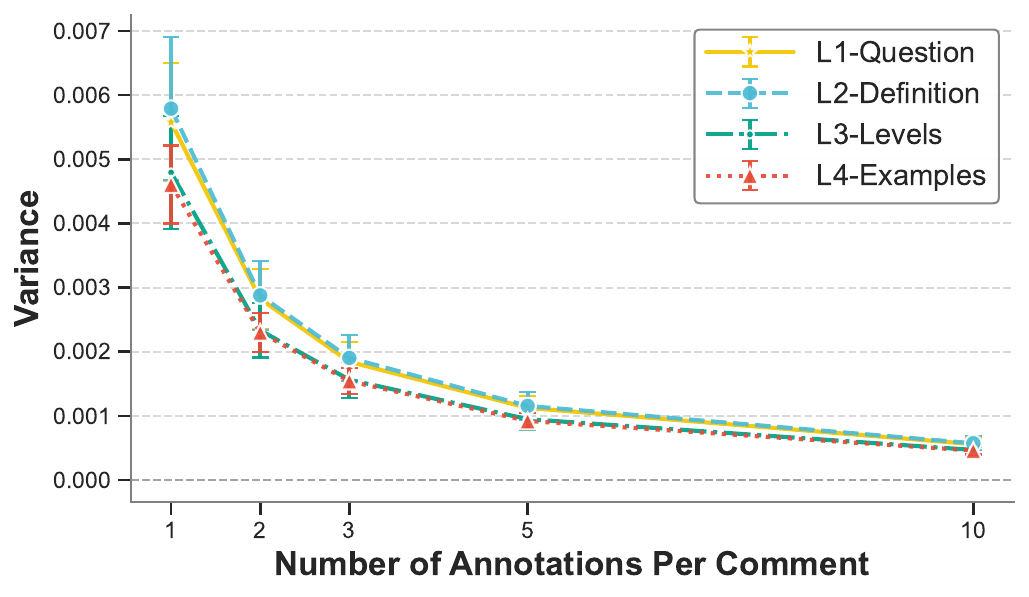}
        \caption{GPT-OSS:120B}
    \end{subfigure}
    \hfill
    \begin{subfigure}[c]{0.32\textwidth}
        \centering
        \includegraphics[width=\linewidth]{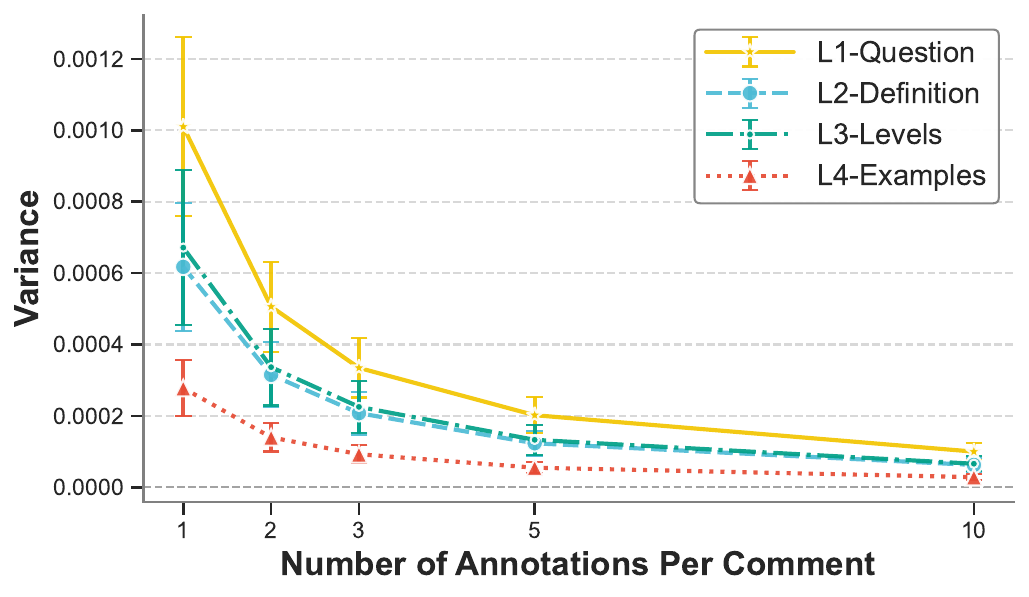}
        \caption{GPT-5.1}
    \end{subfigure}

    \caption{\textbf{Impact of prompt structure} on $k{=}1$ perspective-taking across three GPT models (female subgroup).
    Rows show MSE, bias, and variance.
    Increased structure (L1$\to$L3) reduces MSE for GPT-OSS models primarily by shifting bias; adding examples (L4) hurts GPT-5.1. Variance decreases across all models and prompt levels.}
    \label{fig:influence_prompting_models}
    \vspace{-2pt}
\end{figure*}

\begin{figure*}[t]
    \centering
    \begin{subfigure}[c]{0.32\textwidth}
        \centering
        \includegraphics[width=\linewidth]{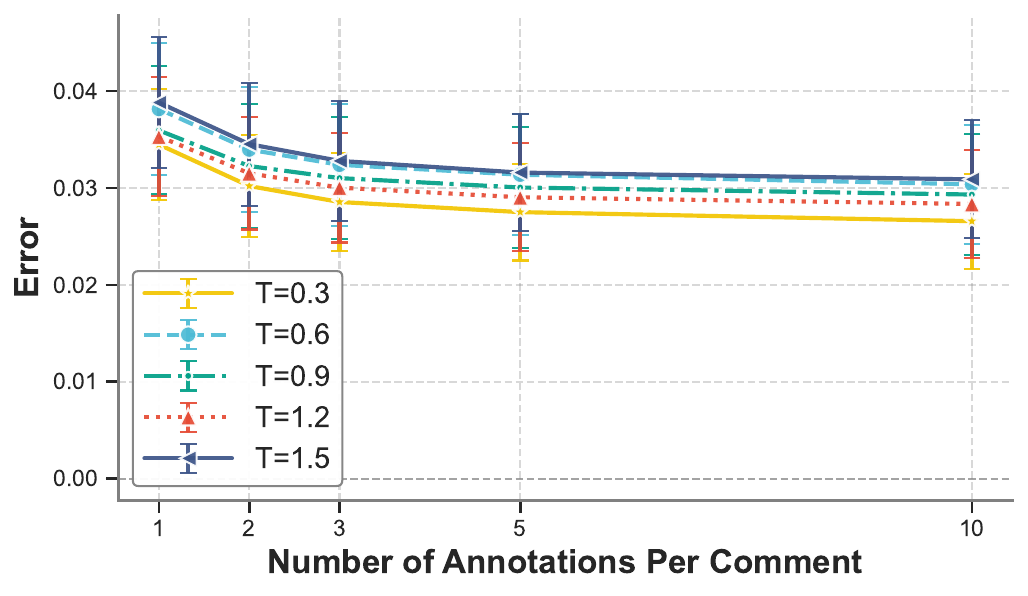}
    \end{subfigure}
    \hfill
    \begin{subfigure}[c]{0.32\textwidth}
        \centering
        \includegraphics[width=\linewidth]{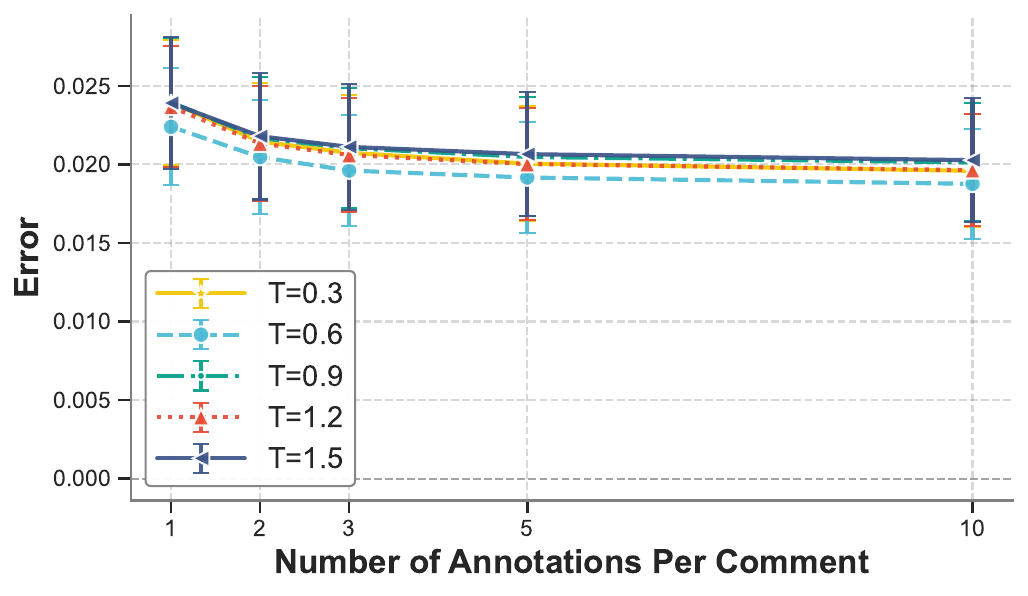}
    \end{subfigure}
    \hfill
    \begin{subfigure}[c]{0.32\textwidth}
        \centering
        \includegraphics[width=\linewidth]{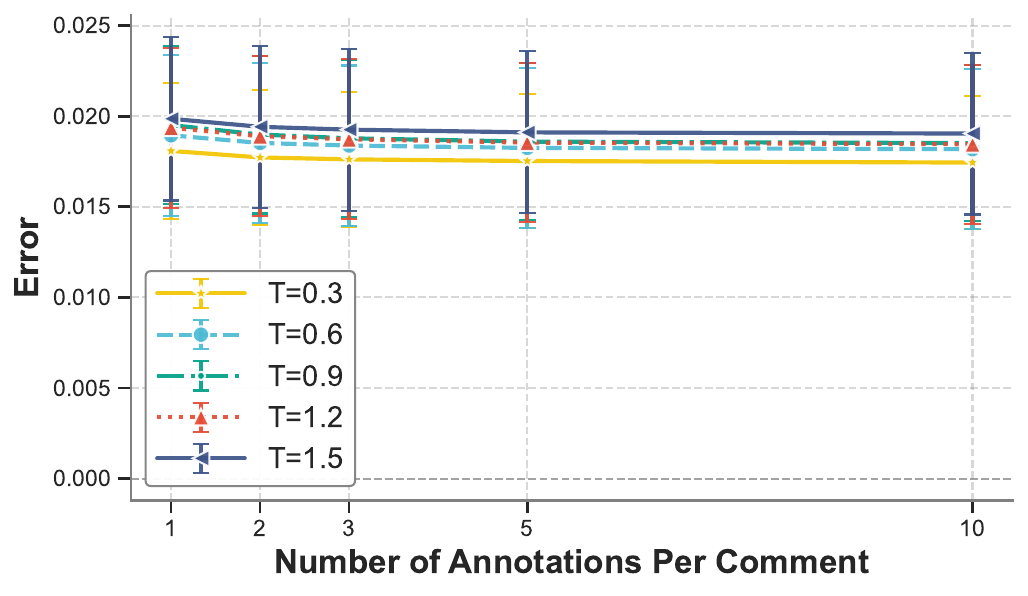}
    \end{subfigure}

    \vspace{0.6em}

    \begin{subfigure}[c]{0.32\textwidth}
        \centering
        \includegraphics[width=\linewidth]{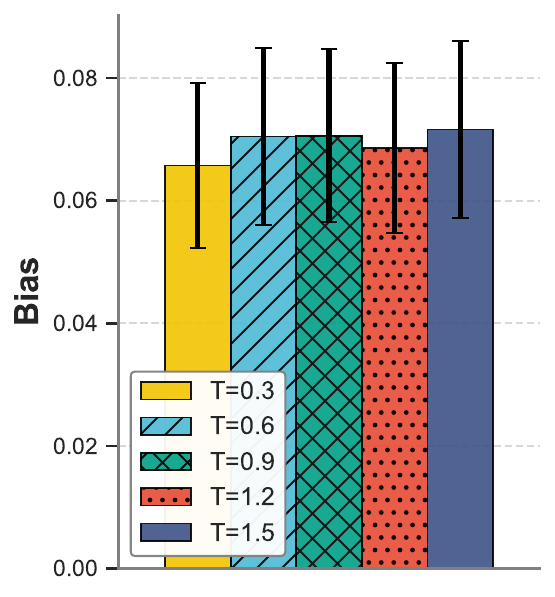}
    \end{subfigure}
    \hfill
    \begin{subfigure}[c]{0.32\textwidth}
        \centering
        \includegraphics[width=\linewidth]{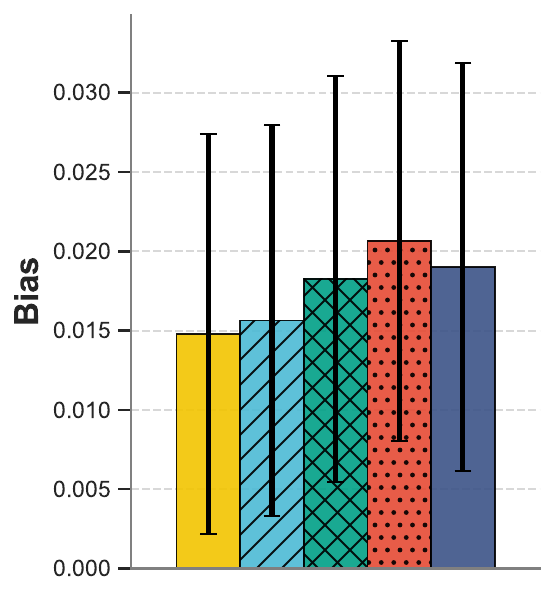}
    \end{subfigure}
    \hfill
    \begin{subfigure}[c]{0.32\textwidth}
        \centering
        \includegraphics[width=\linewidth]{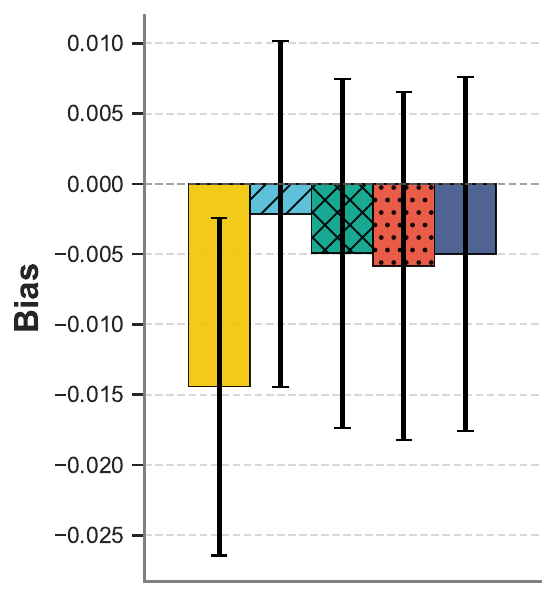}
    \end{subfigure}

    \vspace{0.6em}

    \begin{subfigure}[c]{0.32\textwidth}
        \centering
        \includegraphics[width=\linewidth]{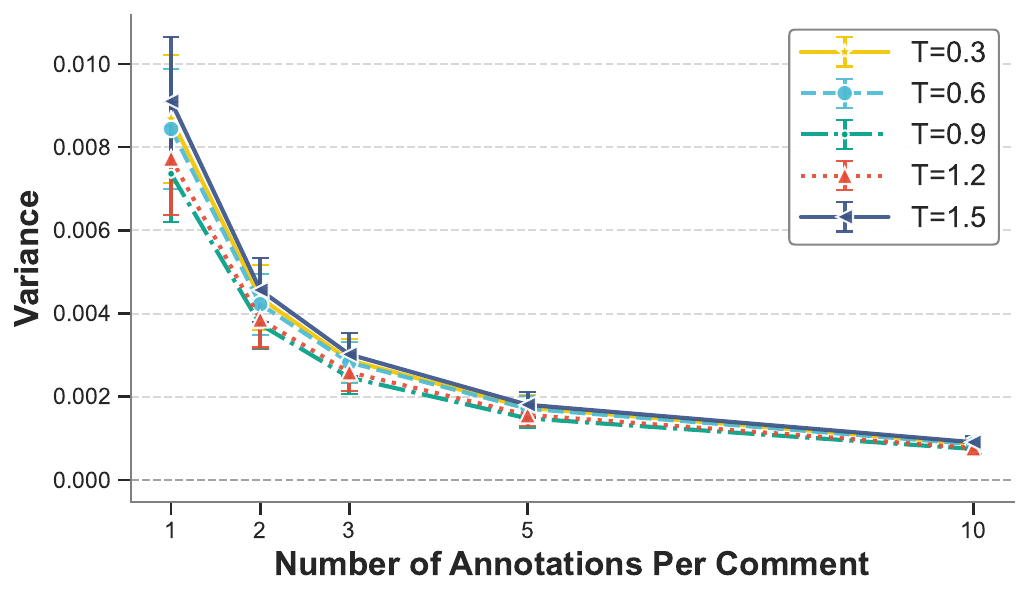}
        \caption{GPT-OSS:20B}
    \end{subfigure}
    \hfill
    \begin{subfigure}[c]{0.32\textwidth}
        \centering
        \includegraphics[width=\linewidth]{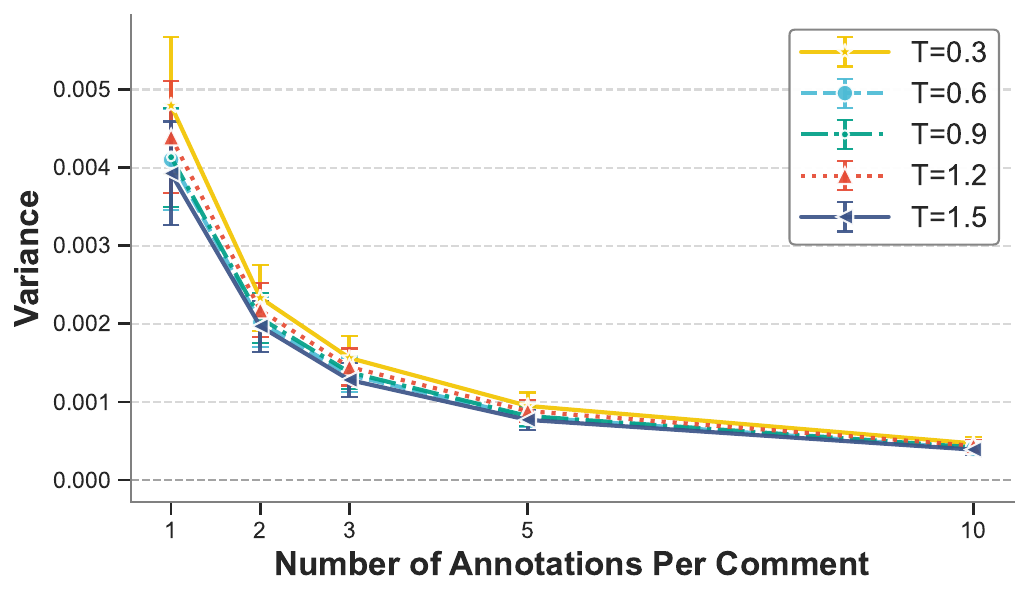}
        \caption{GPT-OSS:120B}
    \end{subfigure}
    \hfill
    \begin{subfigure}[c]{0.32\textwidth}
        \centering
        \includegraphics[width=\linewidth]{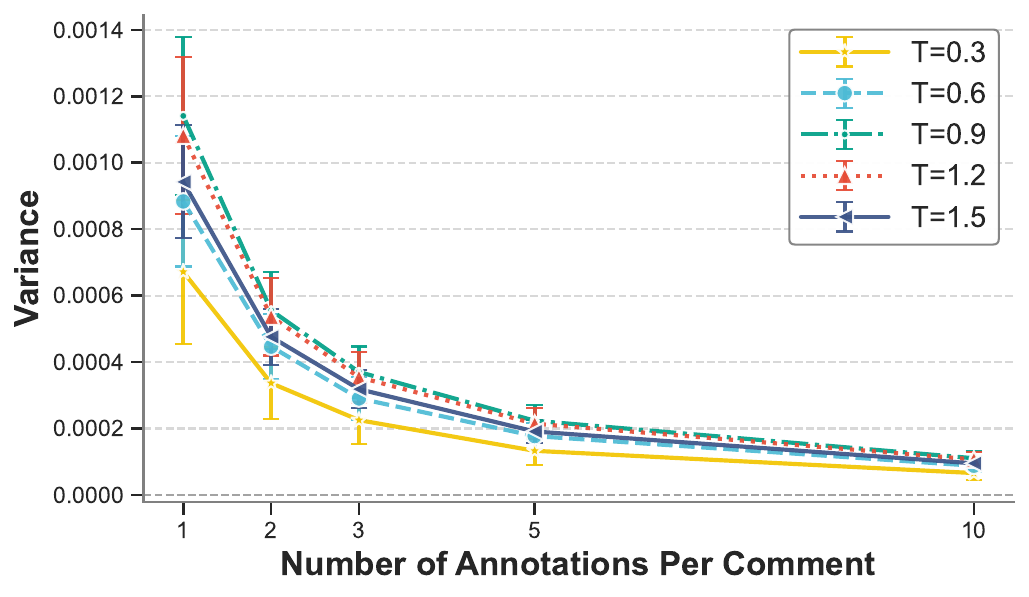}
        \caption{GPT-5.1}
    \end{subfigure}

    \caption{\textbf{Impact of sampling temperature} on $k{=}1$ perspective-taking across three GPT models (female subgroup). Rows show MSE, bias, and variance. Temperature changes yield only modest MSE differences, suggesting that naive stochastic diversification does not produce the independent errors needed to reduce the correlation floor.}
    \label{fig:influence_temperature_models}
\end{figure*}

\begin{figure*}[t]
    \centering
    \begin{subfigure}[c]{0.32\textwidth}
        \centering
        \includegraphics[width=\linewidth]{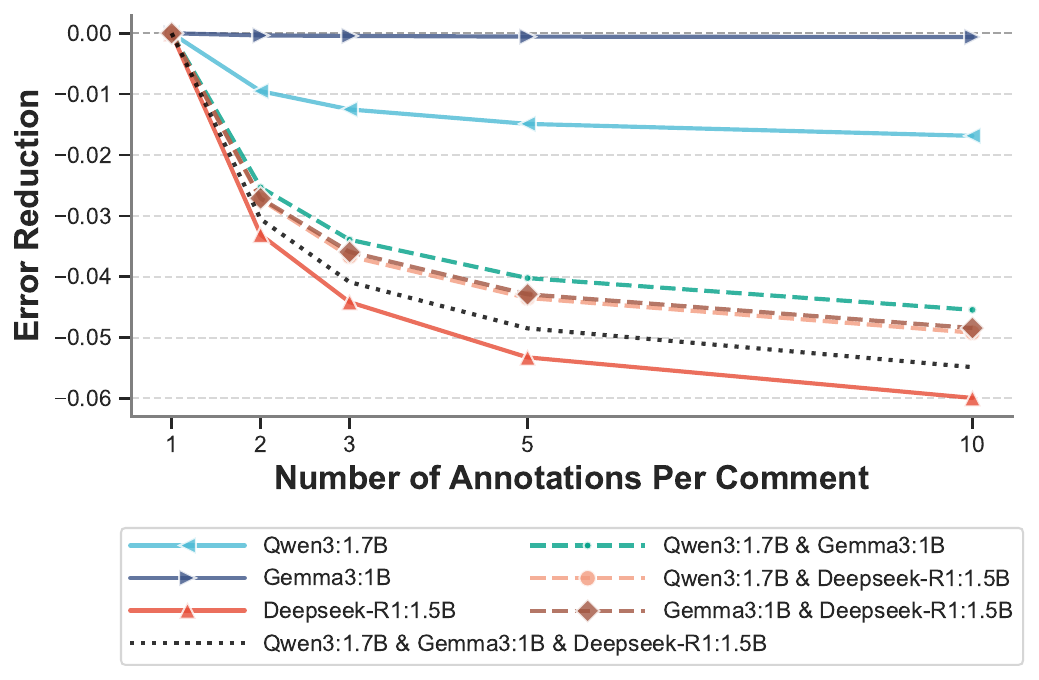}
        \caption{Small models}
    \end{subfigure}
    \hfill
    \begin{subfigure}[c]{0.32\textwidth}
        \centering
        \includegraphics[width=\linewidth]{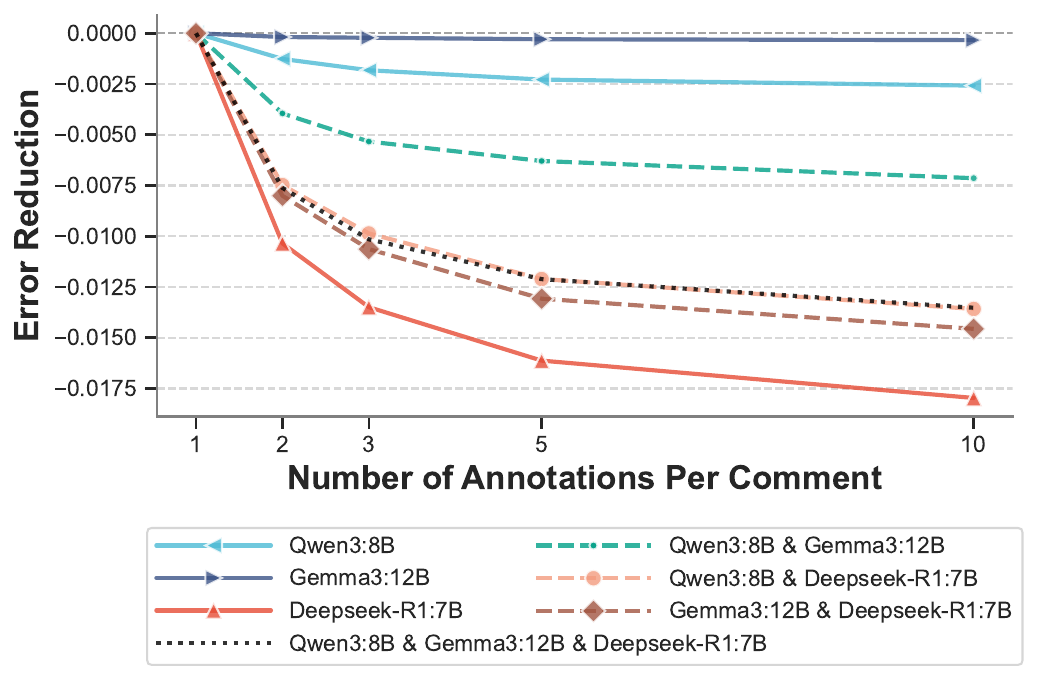}
        \caption{Medium models}
    \end{subfigure}
    \hfill
    \begin{subfigure}[c]{0.32\textwidth}
        \centering
        \includegraphics[width=\linewidth]{figures/toxicity_ablations/mixture_family_large.pdf}
        \caption{Large models}
    \end{subfigure}
    \caption{\textbf{Cross-family model mixing} at three size tiers (female subgroup).
    Error reduction is measured relative to each model's $k{=}1$ MSE.
    Mixing yields consistent gains, especially for large models where bias profiles differ enough for cancellation.
    For smaller models, gains are bounded by the strongest individual model.}
    \label{fig:influence_mixing_family}
\end{figure*}

\begin{figure*}[t]
    \begin{subfigure}[c]{0.32\textwidth}
        \centering
        \includegraphics[width=\linewidth]{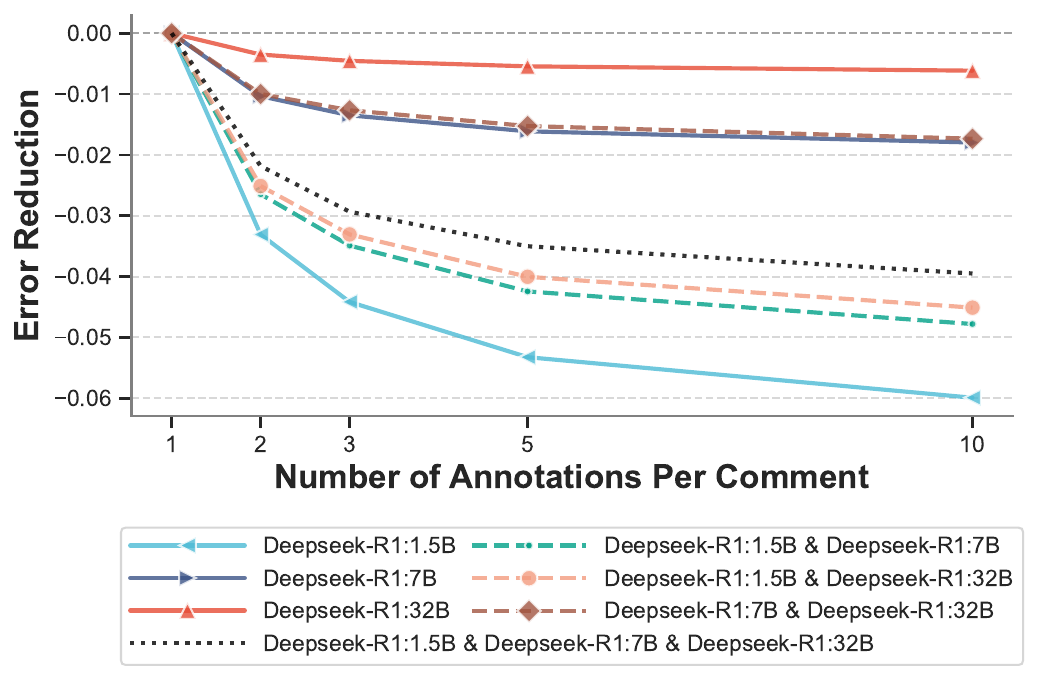}
        \caption{DeepSeek-R1}
    \end{subfigure}
    \hfill
    \begin{subfigure}[c]{0.32\textwidth}
        \centering
        \includegraphics[width=\linewidth]{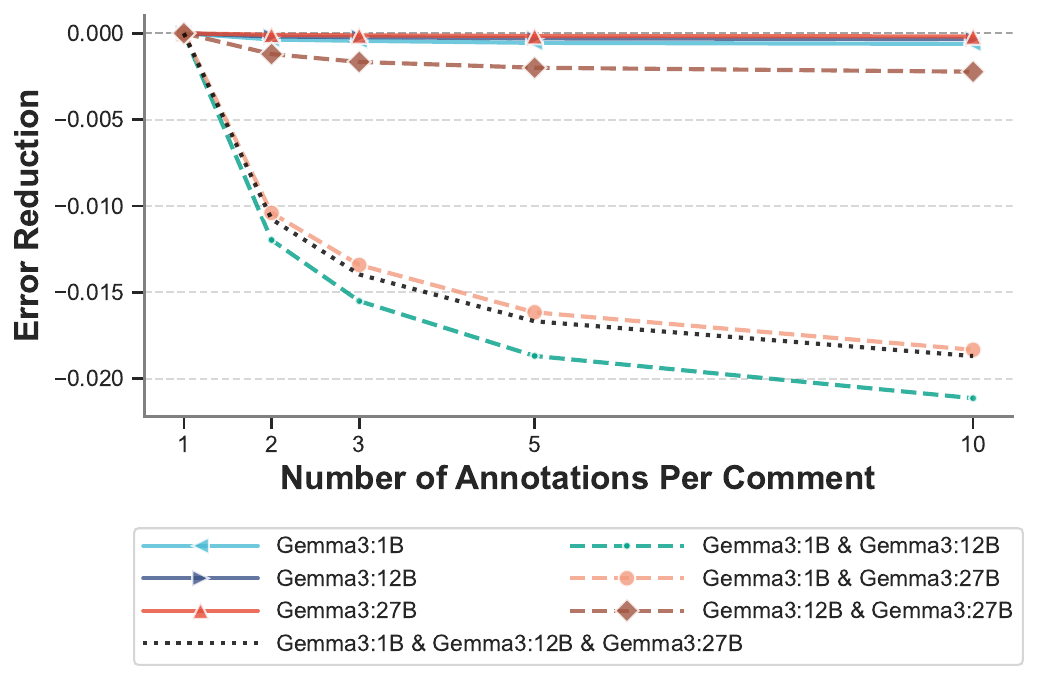}
        \caption{Gemma3}
    \end{subfigure}
    \hfill
    \begin{subfigure}[c]{0.32\textwidth}
        \centering
        \includegraphics[width=\linewidth]{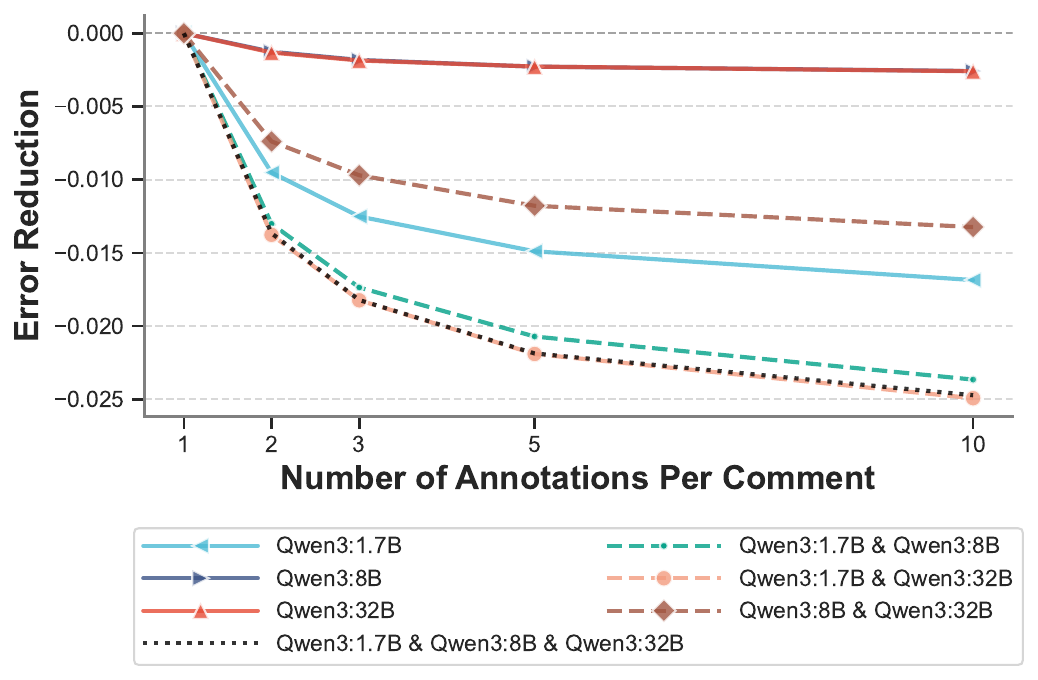}
        \caption{Qwen3}
    \end{subfigure}

    \vspace{0.6em}

    \begin{subfigure}[c]{0.32\textwidth}
        \centering
        \includegraphics[width=\linewidth]{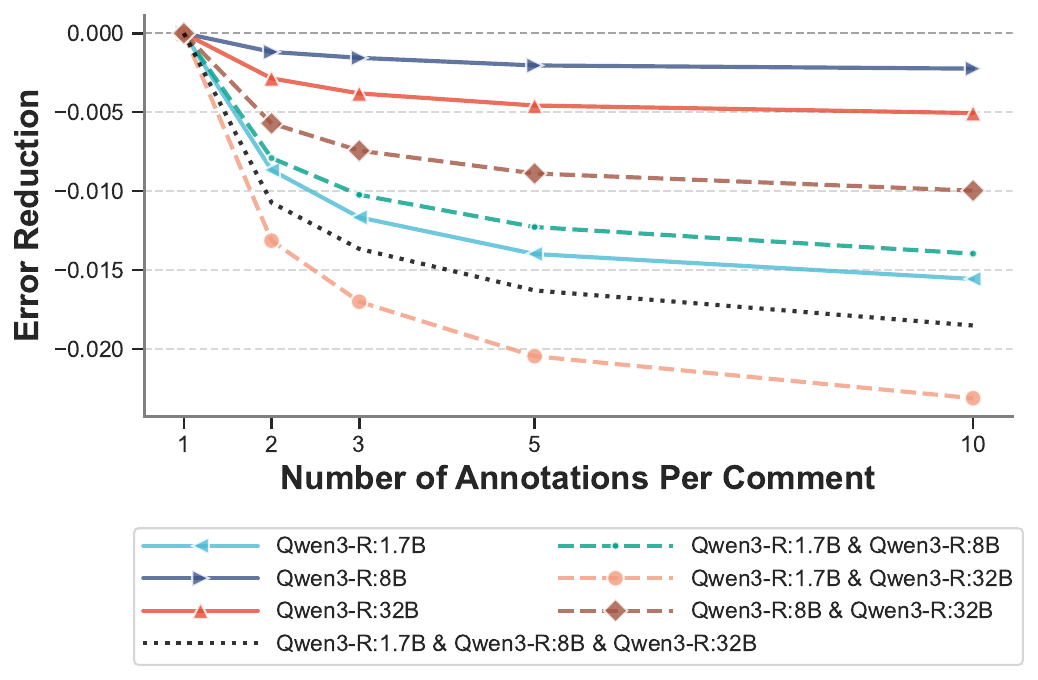}
        \caption{Qwen3-R}
    \end{subfigure}
    \hfill
    \begin{subfigure}[c]{0.32\textwidth}
        \centering
        \includegraphics[width=\linewidth]{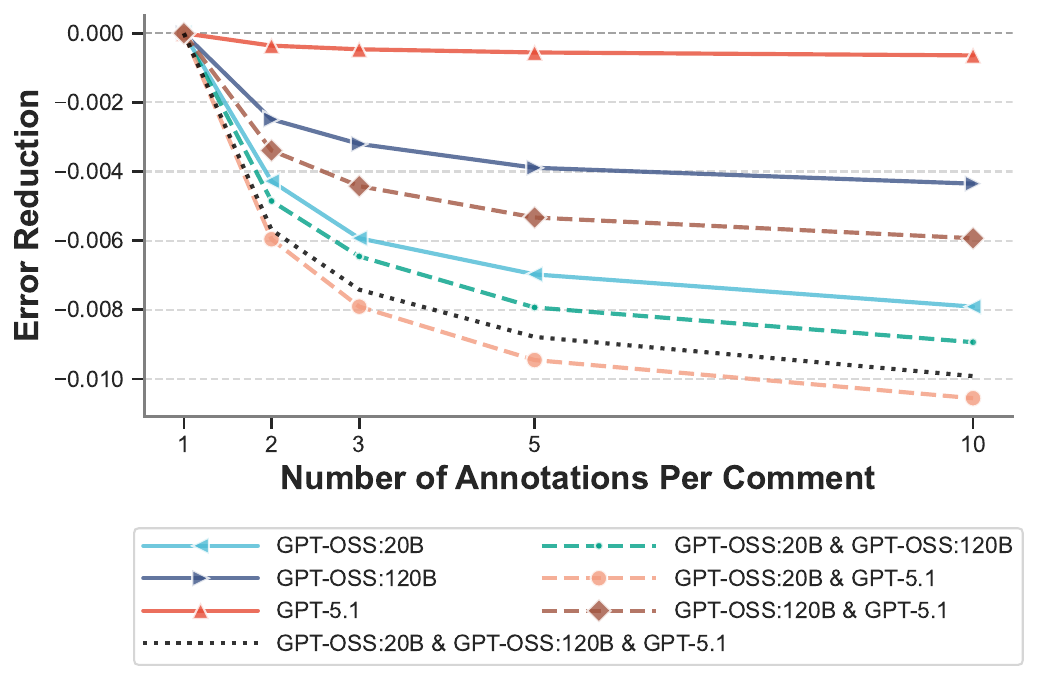}
        \caption{GPT}
    \end{subfigure}
    \hfill
    \begin{subfigure}[c]{0.32\textwidth}
        \centering
        \includegraphics[width=\linewidth]{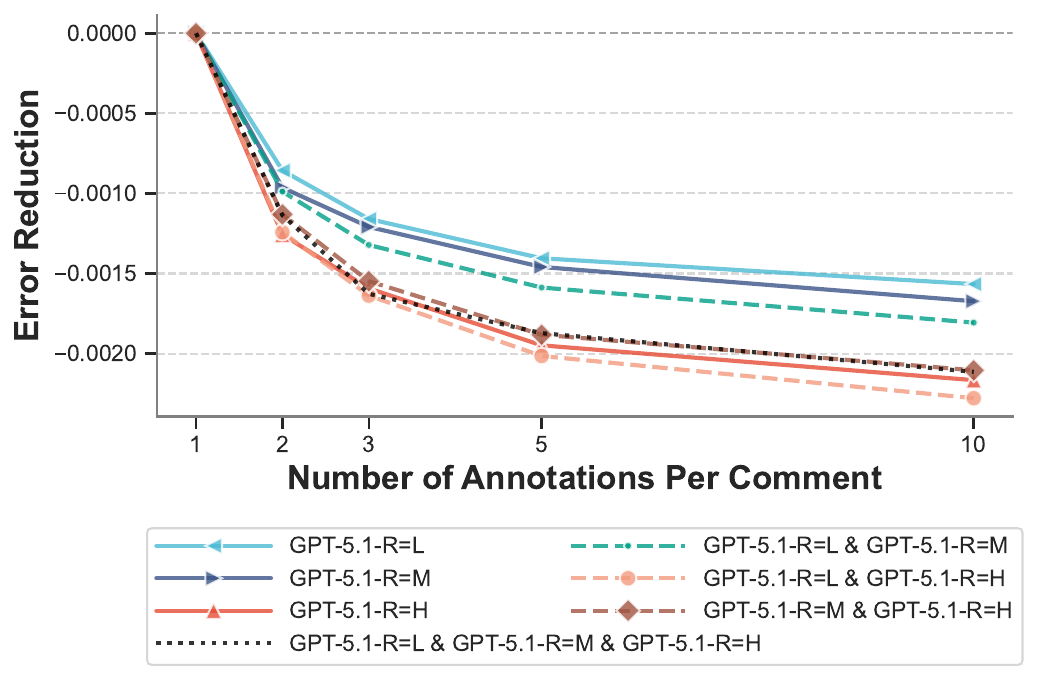}
        \caption{GPT-R}
    \end{subfigure}

    \caption{
    \textbf{Within-family model mixing} across sizes (female subgroup).
    Error reduction is measured relative to each model's $k{=}1$ MSE.
    Mixing models of different sizes within the same family yields consistent error reduction, similar in pattern to cross-family mixing.
    }
    \label{fig:influence_mixing_size}
\end{figure*}

\begin{figure*}[t]
    \begin{subfigure}[c]{0.32\textwidth}
        \centering
        \includegraphics[width=\linewidth]{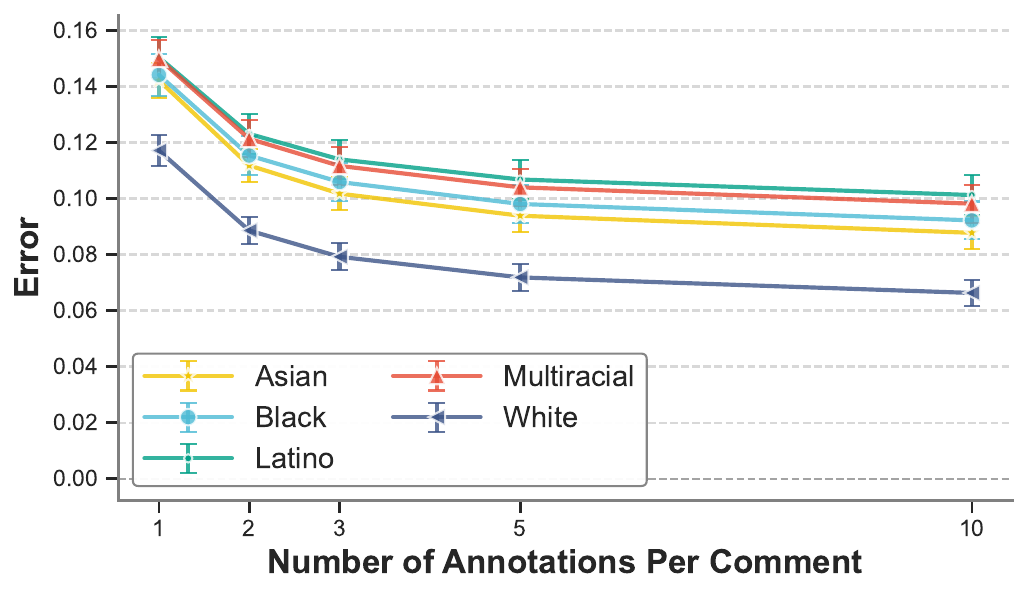}
        \caption{GPT-OSS:20B}
    \end{subfigure}
    \hfill
    \begin{subfigure}[c]{0.32\textwidth}
        \centering
        \includegraphics[width=\linewidth]{figures/DICE/mse_race_gpt-oss-120b.pdf}
        \caption{GPT-OSS:120B}
    \end{subfigure}
    \hfill
    \begin{subfigure}[c]{0.32\textwidth}
        \centering
        \includegraphics[width=\linewidth]{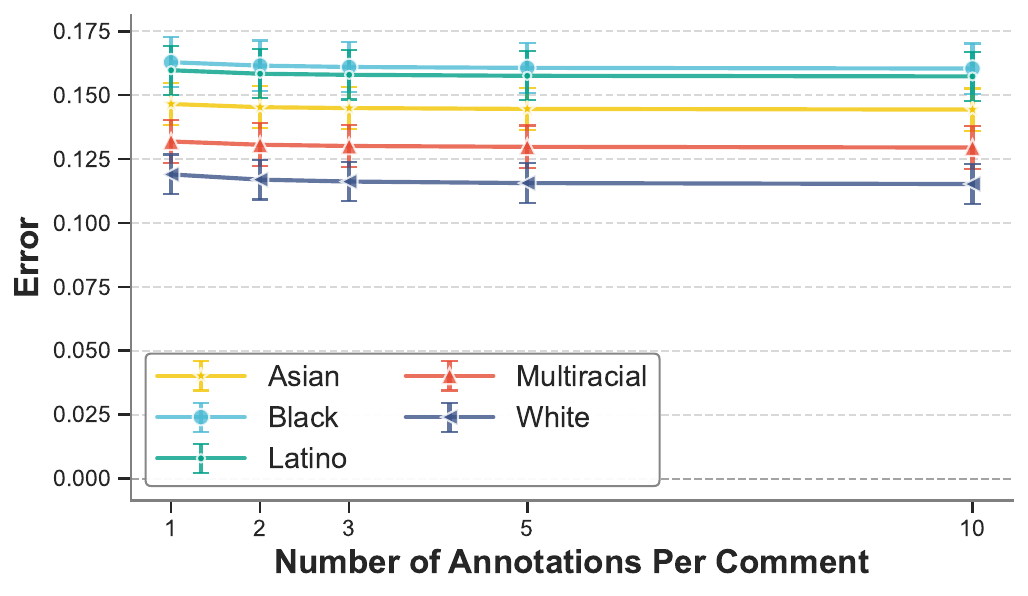}
        \caption{GPT-5.1}
    \end{subfigure}

    \vspace{0.6em}

    \begin{subfigure}[c]{0.32\textwidth}
        \centering
        \includegraphics[width=\linewidth]{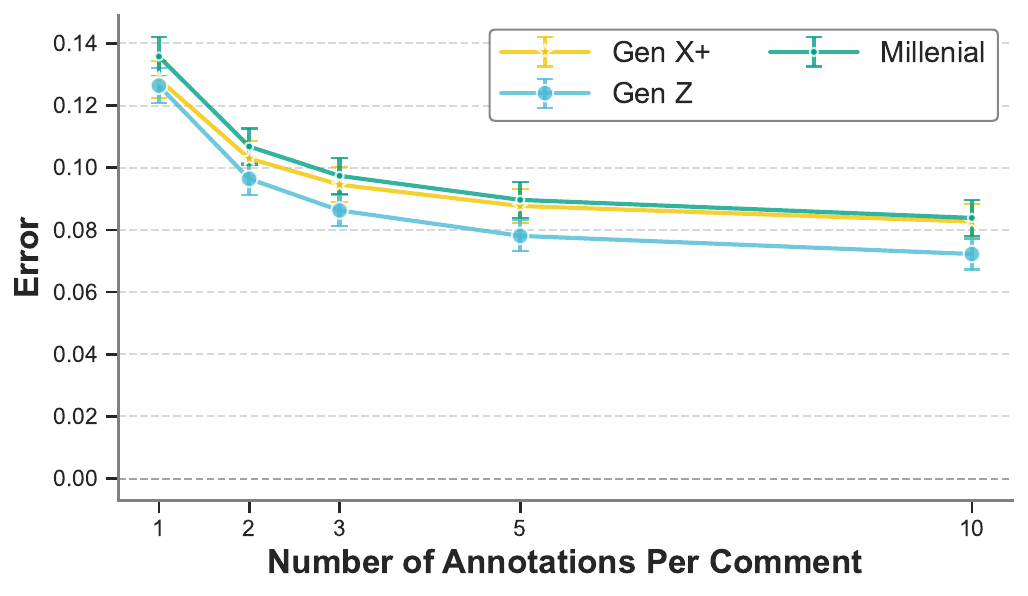}
        \caption{GPT-OSS:20B}
    \end{subfigure}
    \hfill
    \begin{subfigure}[c]{0.32\textwidth}
        \centering
        \includegraphics[width=\linewidth]{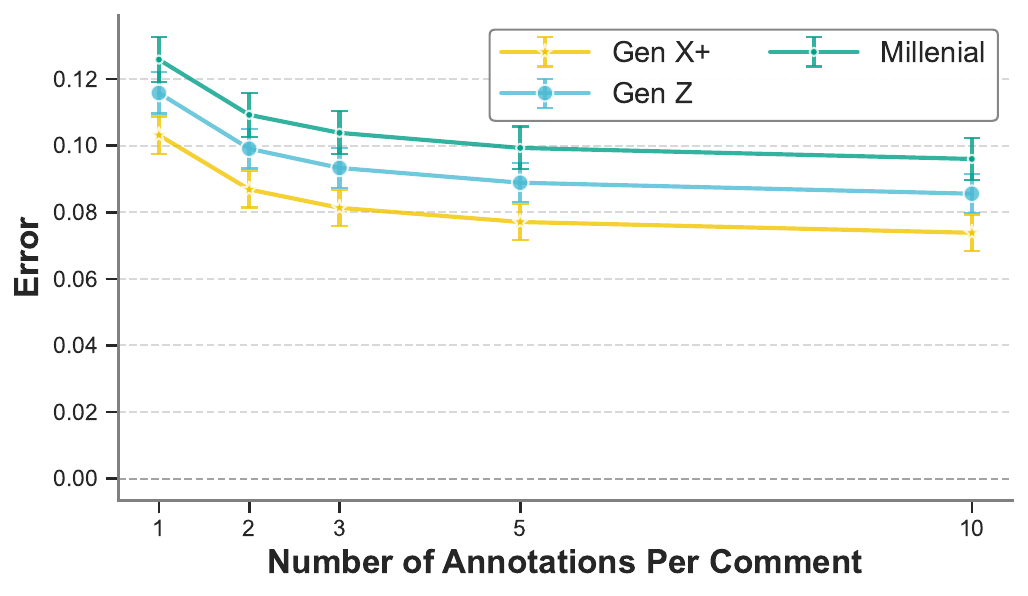}
        \caption{GPT-OSS:120B}
    \end{subfigure}
    \hfill
    \begin{subfigure}[c]{0.32\textwidth}
        \centering
        \includegraphics[width=\linewidth]{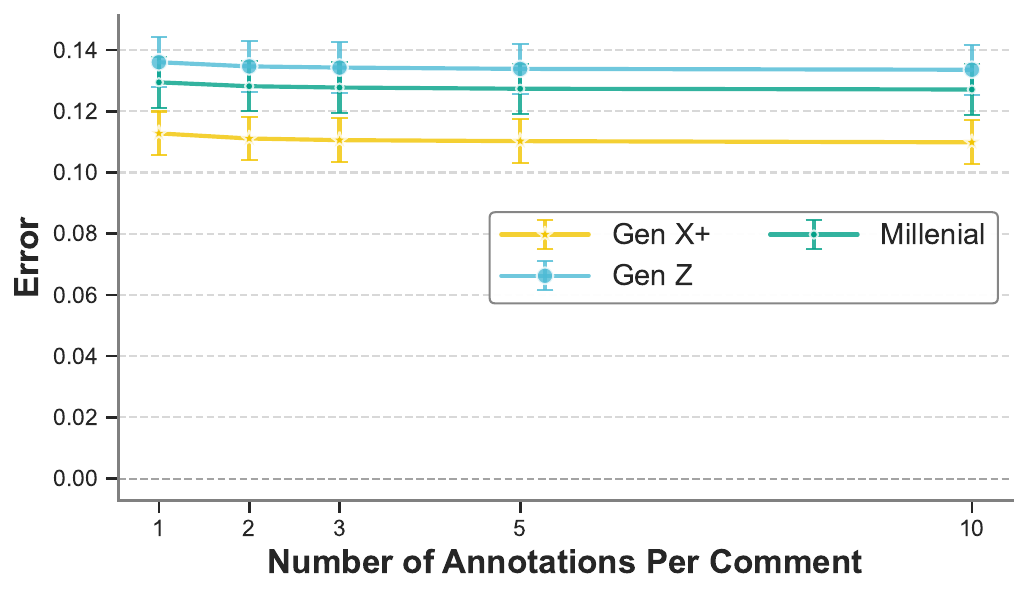}
        \caption{GPT-5.1}
    \end{subfigure}

    \caption{\textbf{LLM PT performance by subgroup prevalence} on DICES, for race (a--c) and age (d--f) subgroups.
    MSE generally deteriorates for less prevalent subgroups, with the pattern more pronounced for GPT-OSS:120B and GPT-5.1.}
    \label{fig:DICES_breadth}
\end{figure*}

\begin{figure*}[t]
    \centering
    \begin{subfigure}[c]{0.24\textwidth}
        \centering
        \includegraphics[width=\linewidth]{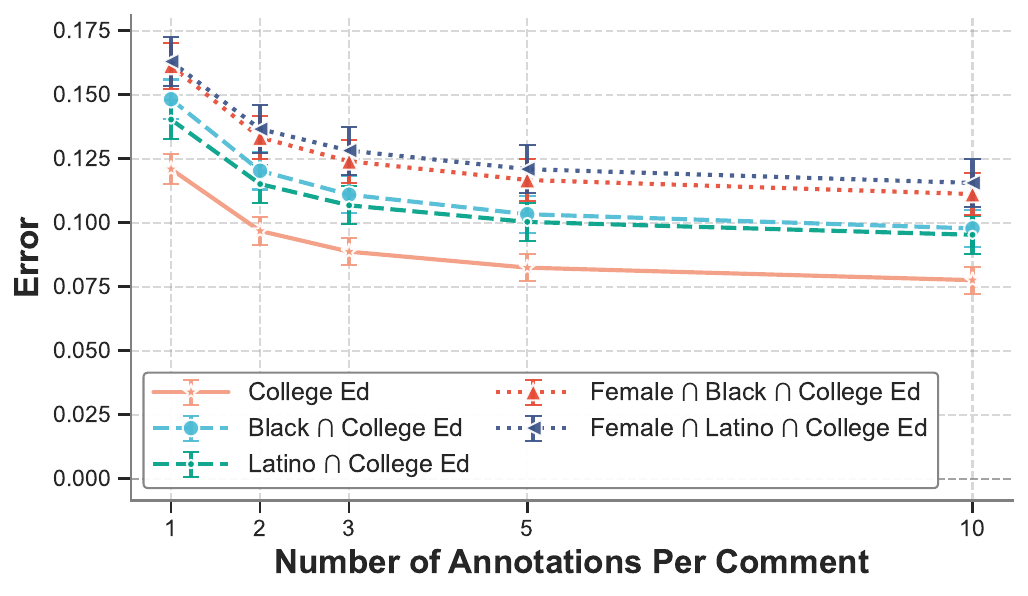}
        \caption{GPT-OSS:20B}
    \end{subfigure}
    \hfill
    \begin{subfigure}[c]{0.24\textwidth}
        \centering
        \includegraphics[width=\linewidth]{figures/DICE/mse_good_path_gpt-oss-120b.pdf}
        \caption{GPT-OSS:120B}
    \end{subfigure}
    \hfill
    \begin{subfigure}[c]{0.24\textwidth}
        \centering
        \includegraphics[width=\linewidth]{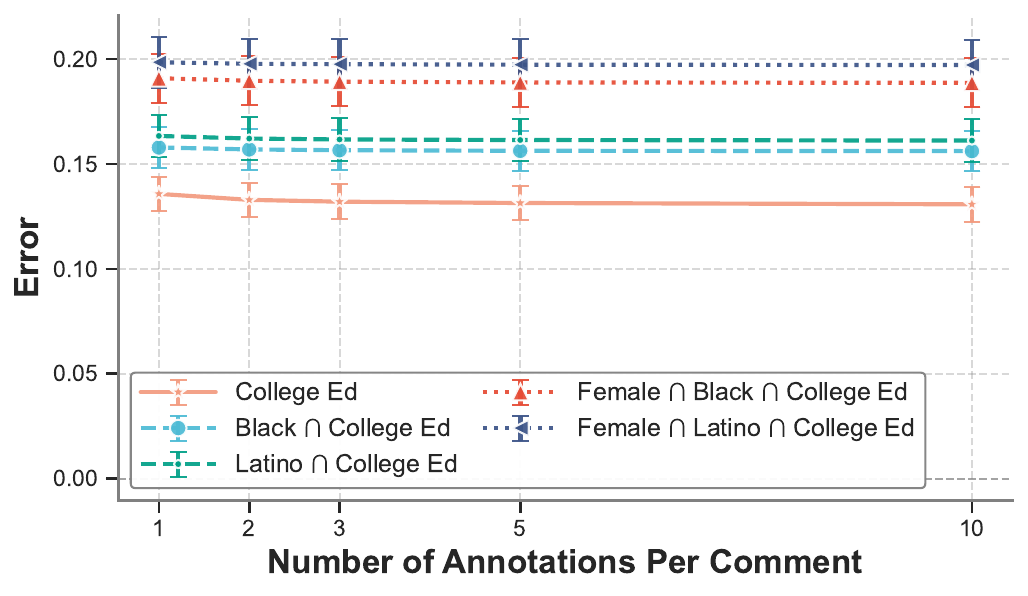}
        \caption{GPT-5.1}
    \end{subfigure}
    \hfill
    \begin{subfigure}[c]{0.24\textwidth}
        \centering
        \includegraphics[width=\linewidth]{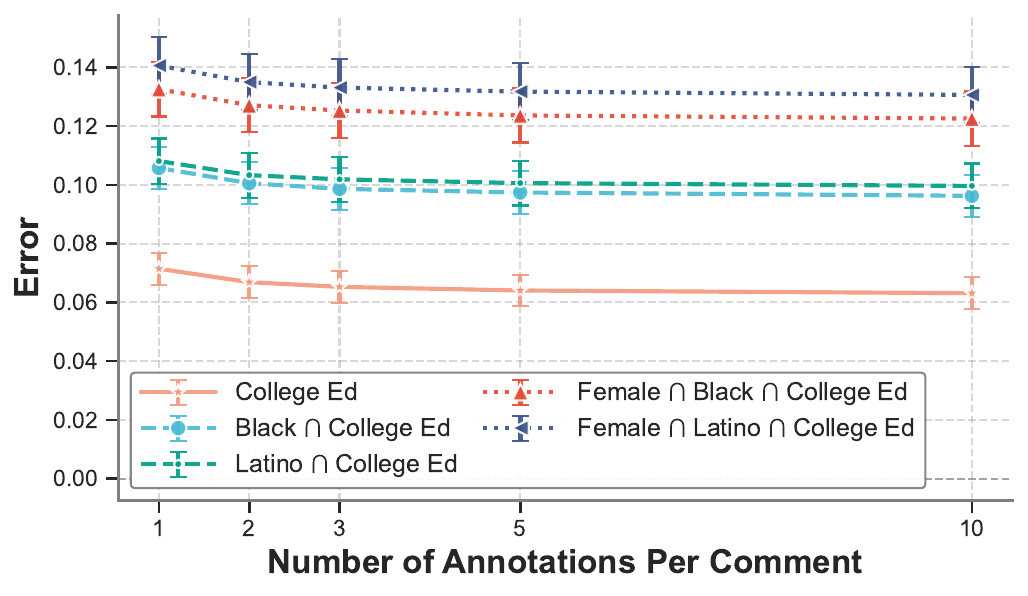}
        \caption{GPT-5.1-R=H}
    \end{subfigure}

    \vspace{0.6em}
    \begin{subfigure}[c]{0.24\textwidth}
        \centering
        \includegraphics[width=\linewidth]{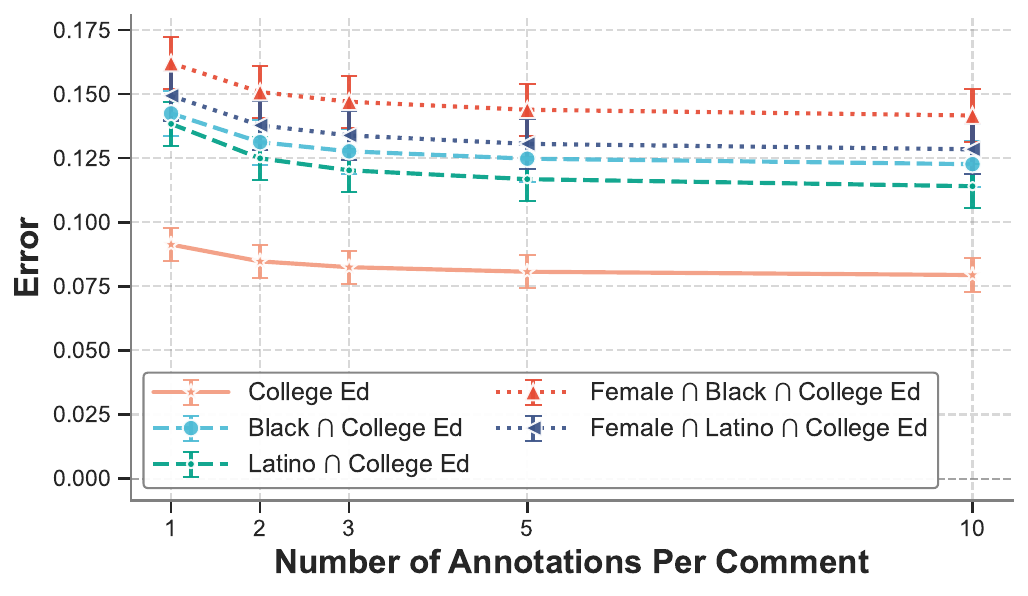}
        \caption{Gemma3:1B}
    \end{subfigure}
    \hfill
    \begin{subfigure}[c]{0.24\textwidth}
        \centering
        \includegraphics[width=\linewidth]{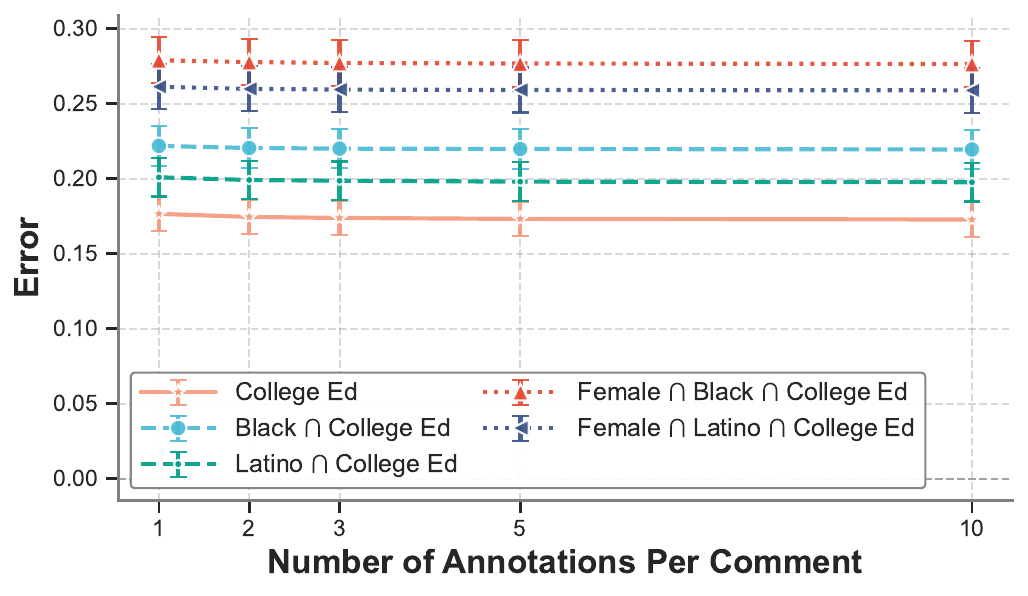}
        \caption{Gemma3:12B}
    \end{subfigure}
    \hfill
    \begin{subfigure}[c]{0.24\textwidth}
        \centering
        \includegraphics[width=\linewidth]{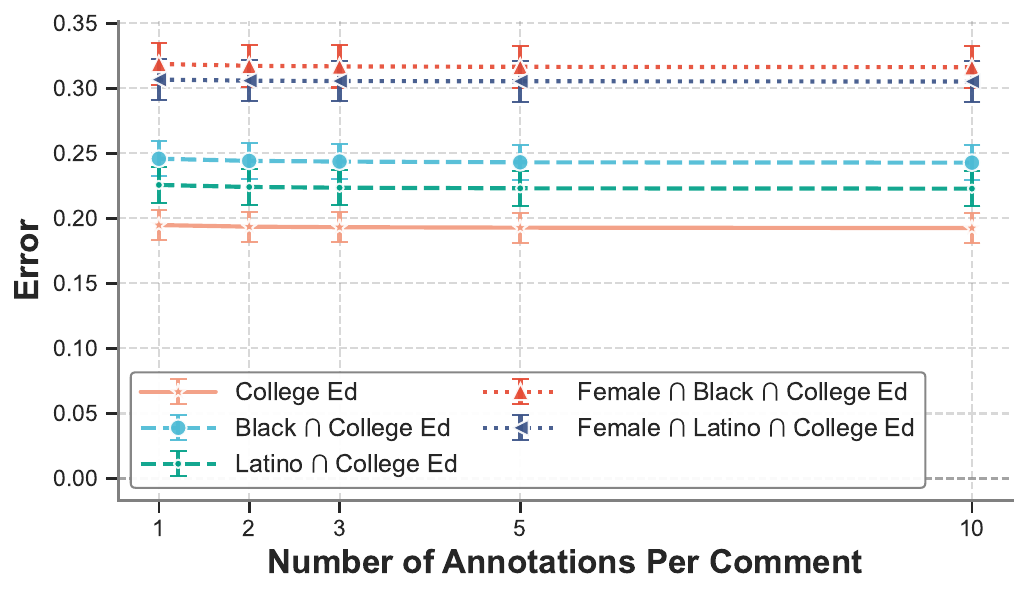}
        \caption{Gemma3:27B}
    \end{subfigure}
    \hfill
    \begin{subfigure}[c]{0.24\textwidth}
        \centering
        \includegraphics[width=\linewidth]{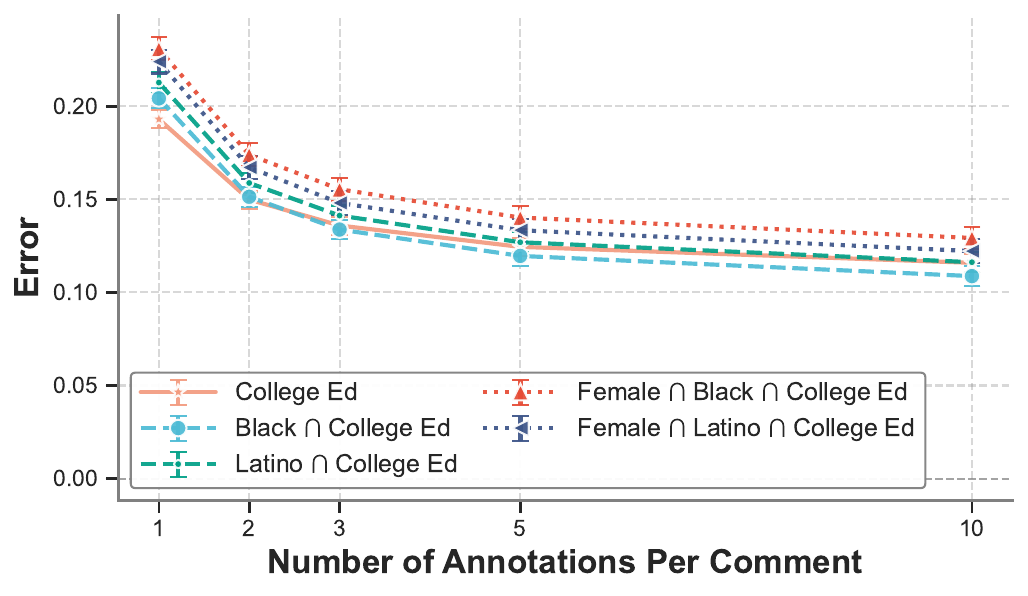}
        \caption{DeepSeek-R1:1.5B}
    \end{subfigure}

    \vspace{0.6em}
    \begin{subfigure}[c]{0.24\textwidth}
        \centering
        \includegraphics[width=\linewidth]{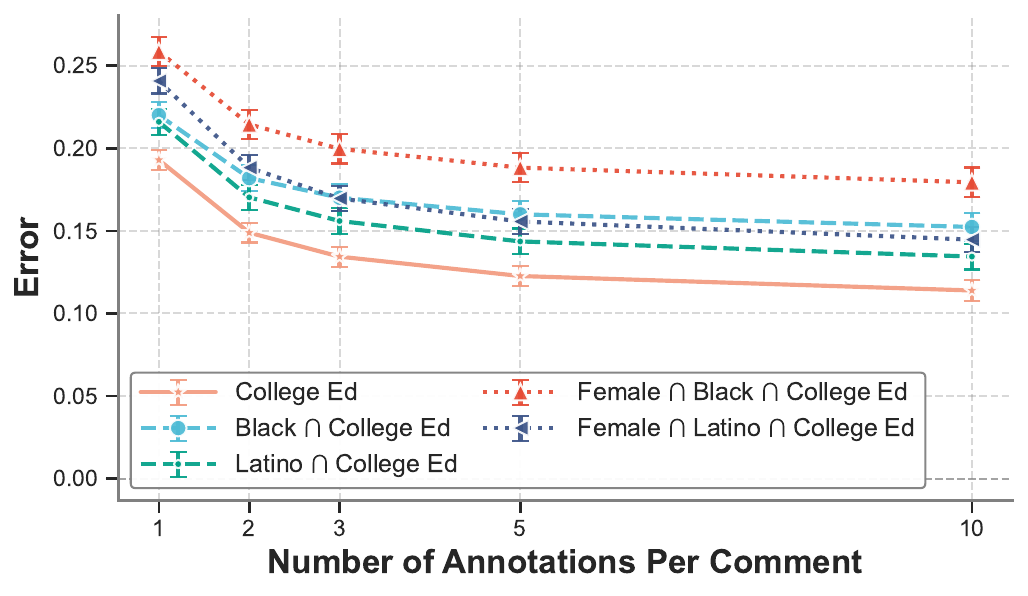}
        \caption{DeepSeek-R1:7B}
    \end{subfigure}
    \hfill
    \begin{subfigure}[c]{0.24\textwidth}
        \centering
        \includegraphics[width=\linewidth]{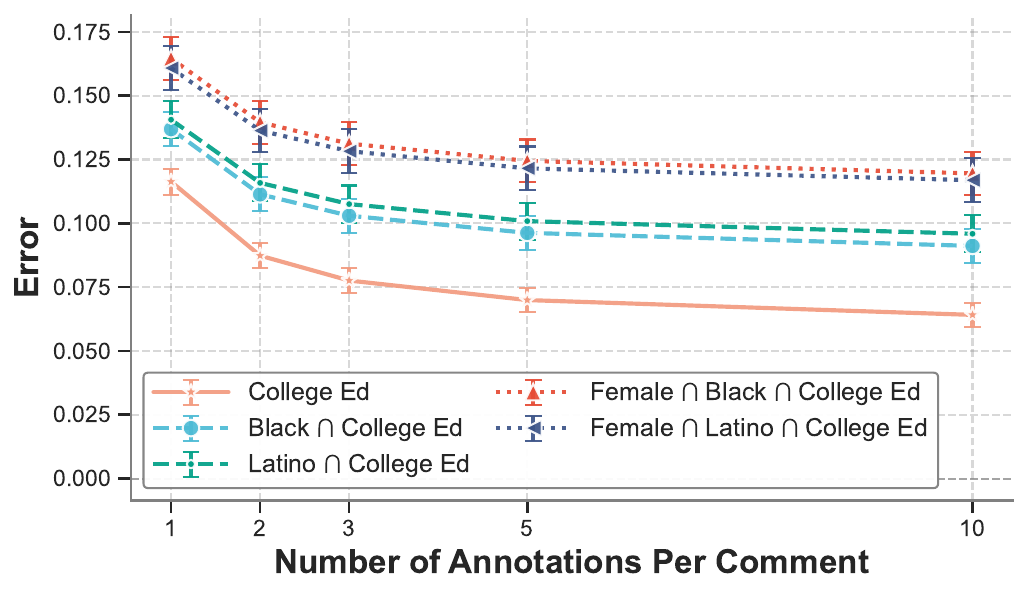}
        \caption{DeepSeek-R1:32B}
    \end{subfigure}
    \hfill
    \begin{subfigure}[c]{0.24\textwidth}
        \centering
        \includegraphics[width=\linewidth]{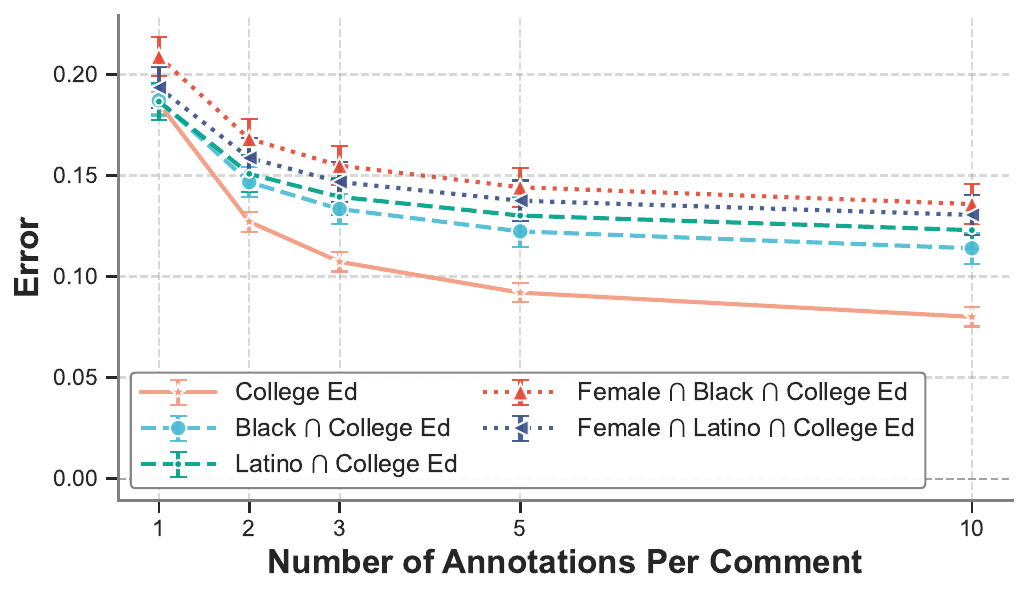}
        \caption{Qwen3:1.7B}
    \end{subfigure}
    \hfill
    \begin{subfigure}[c]{0.24\textwidth}
        \centering
        \includegraphics[width=\linewidth]{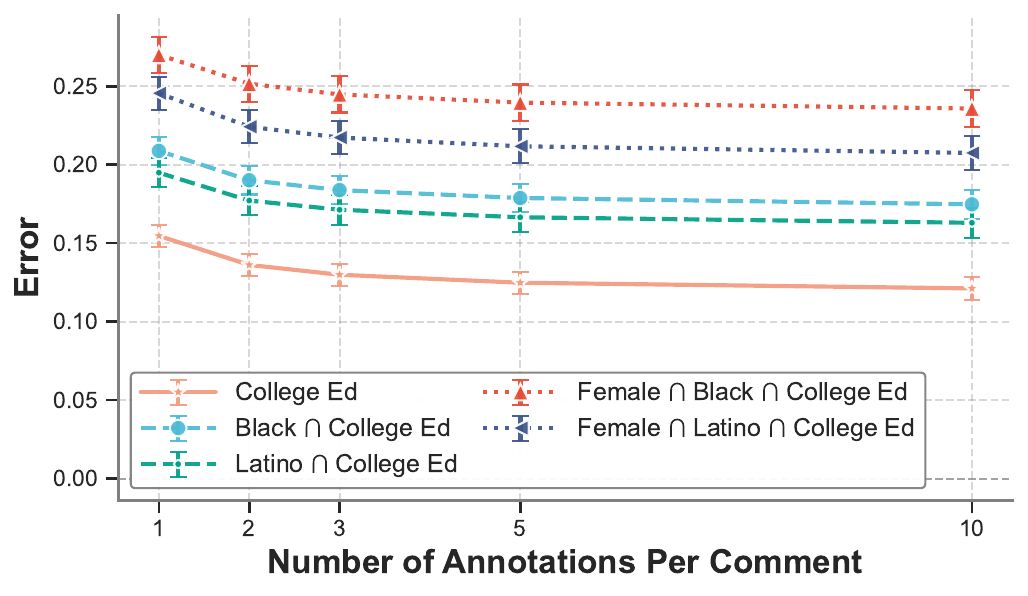}
        \caption{Qwen3:8B}
    \end{subfigure}

    \vspace{0.6em}
    \begin{subfigure}[c]{0.24\textwidth}
        \centering
        \includegraphics[width=\linewidth]{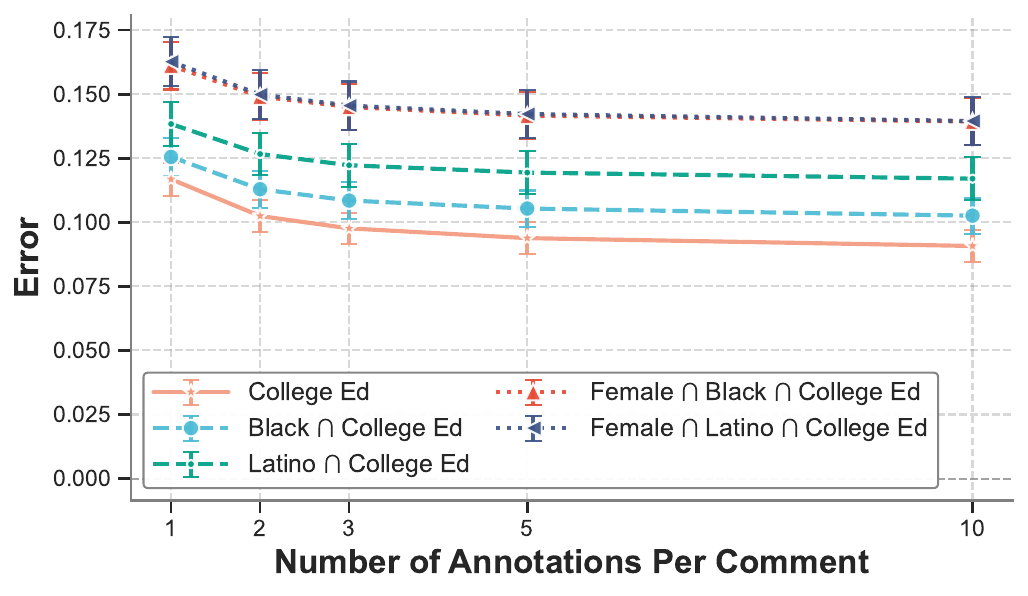}
        \caption{Qwen3:32B}
    \end{subfigure}
    \hfill
    \begin{subfigure}[c]{0.24\textwidth}
        \centering
        \includegraphics[width=\linewidth]{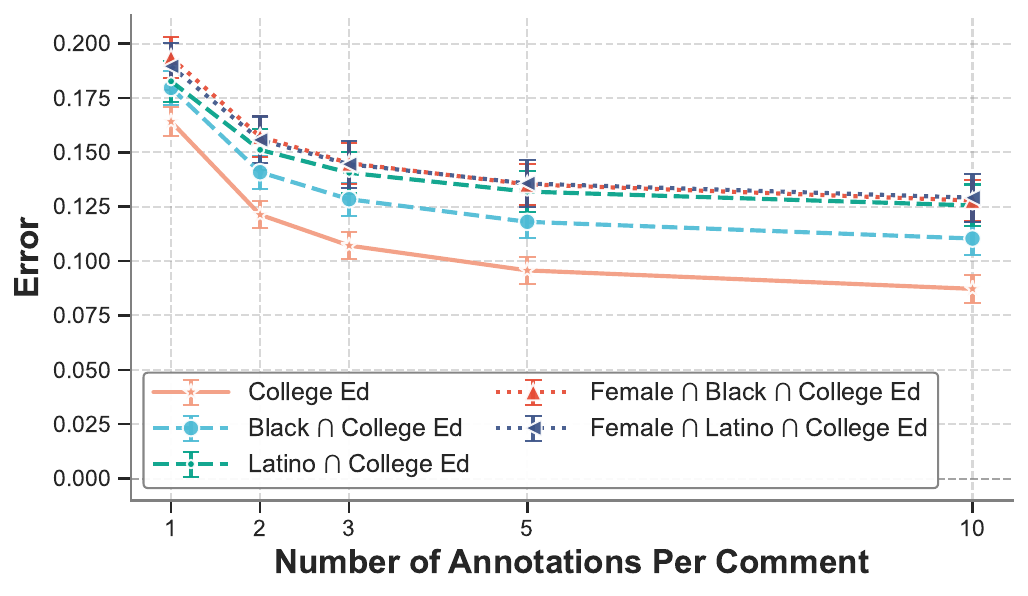}
        \caption{Qwen3-R:1.7B}
    \end{subfigure}
    \hfill
    \begin{subfigure}[c]{0.24\textwidth}
        \centering
        \includegraphics[width=\linewidth]{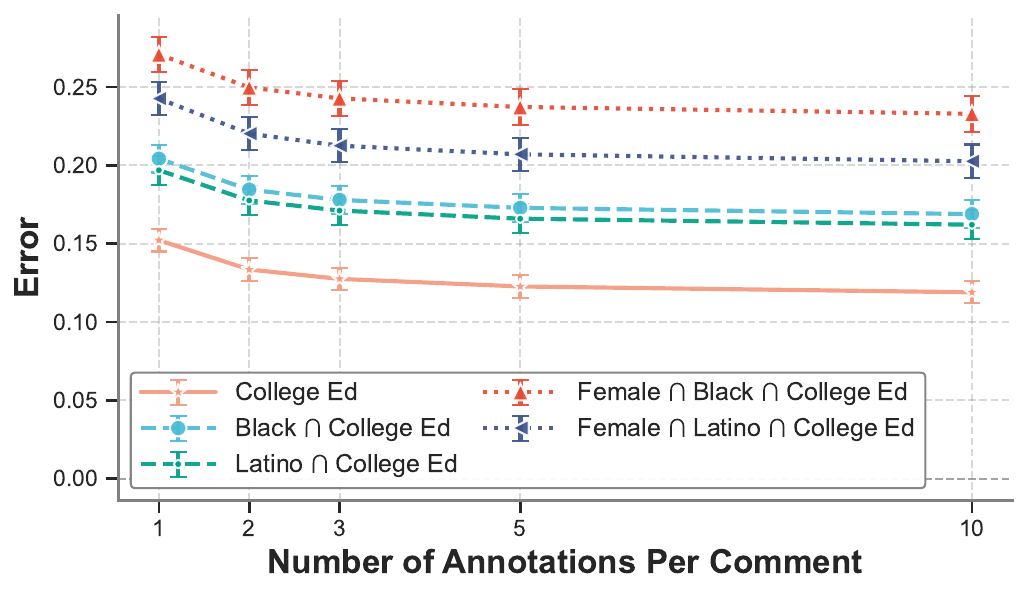}
        \caption{Qwen3-R:8B}
    \end{subfigure}
    \hfill
    \begin{subfigure}[c]{0.24\textwidth}
        \centering
        \includegraphics[width=\linewidth]{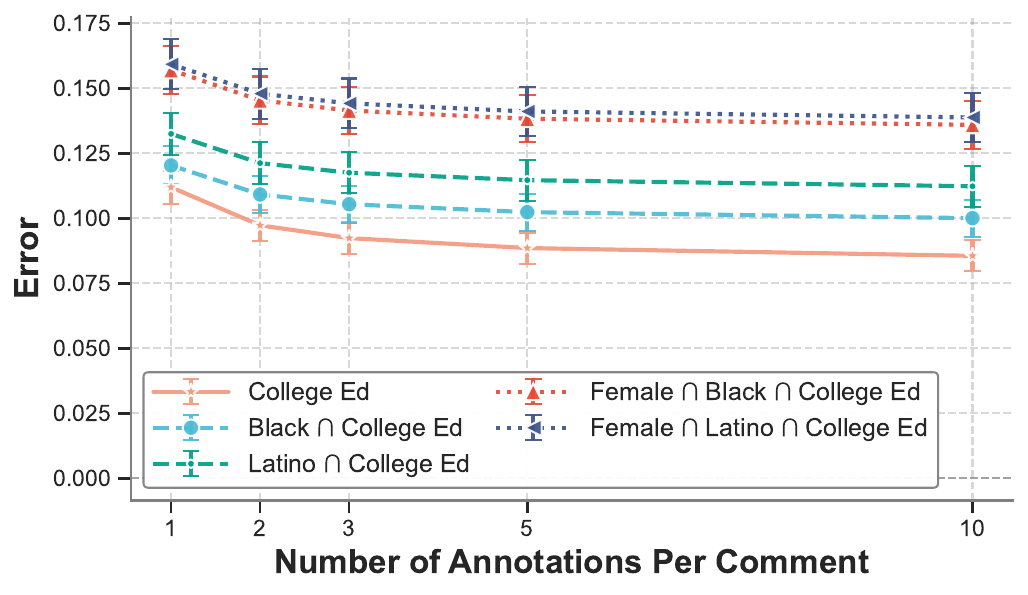}
        \caption{Qwen3-R:32B}
    \end{subfigure}

    \caption{\textbf{Specificity cascade (college-educated path)} on DICES.
    Target subgroups become progressively more specific: college-educated people $\to$ college-educated women $\to$ college-educated Black/Latino women.
    MSE increases with specificity consistently across all 16 models, validating H3.}
    \label{fig:dices_good_example}
\end{figure*}

\begin{figure*}[t]
    \centering
    \begin{subfigure}[c]{0.24\textwidth}
        \centering
        \includegraphics[width=\linewidth]{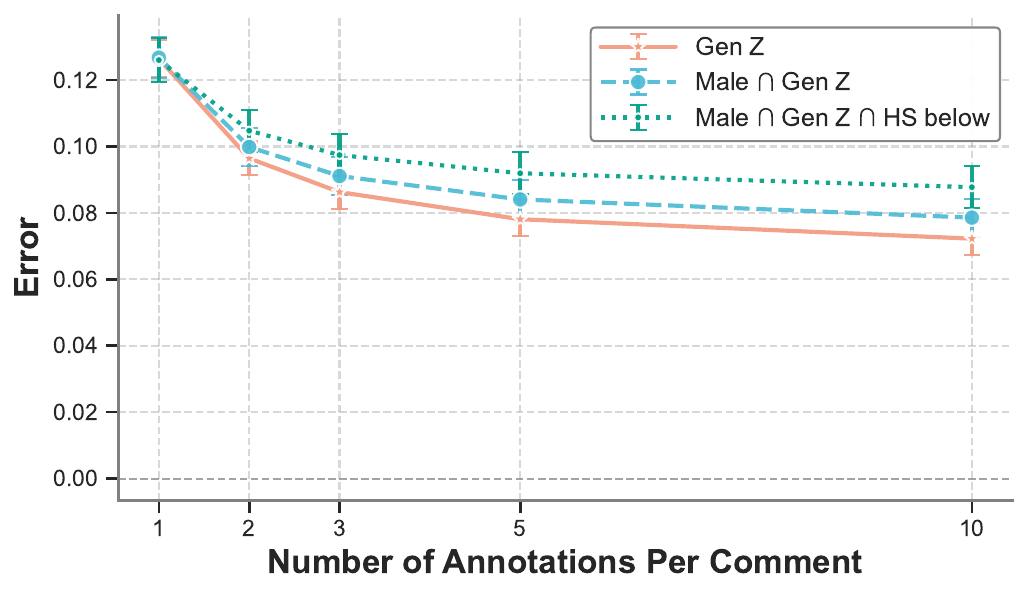}
        \caption{GPT-OSS:20B}
    \end{subfigure}
    \hfill
    \begin{subfigure}[c]{0.24\textwidth}
        \centering
        \includegraphics[width=\linewidth]{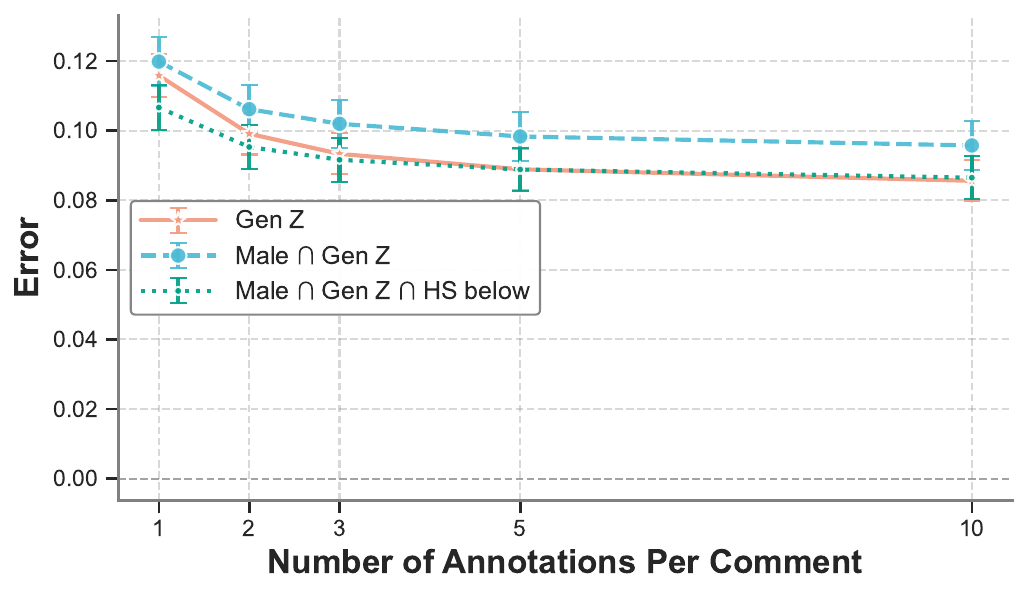}
        \caption{GPT-OSS:120B}
    \end{subfigure}
    \hfill
    \begin{subfigure}[c]{0.24\textwidth}
        \centering
        \includegraphics[width=\linewidth]{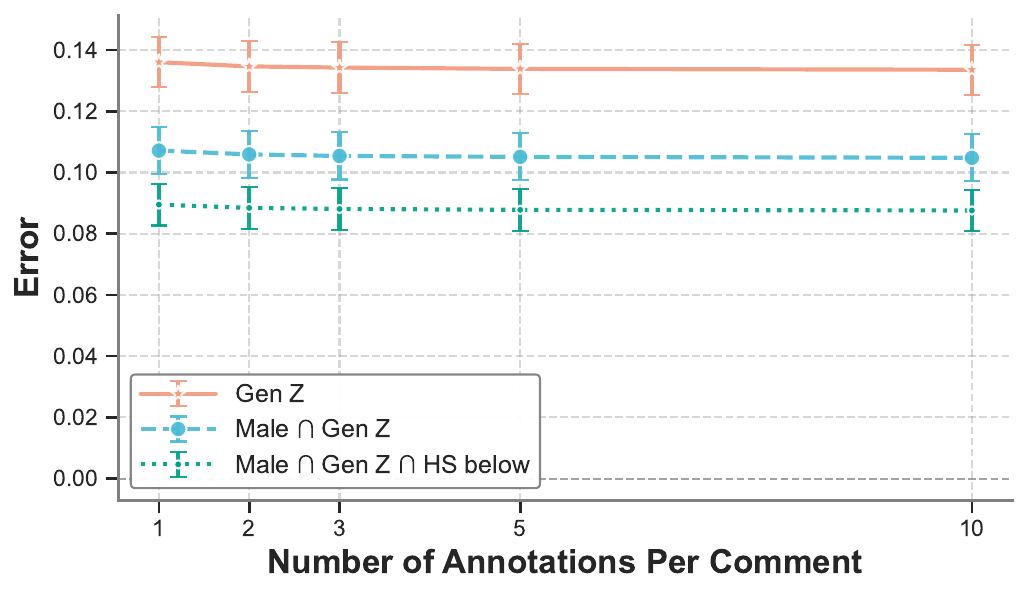}
        \caption{GPT-5.1}
    \end{subfigure}
    \hfill
    \begin{subfigure}[c]{0.24\textwidth}
        \centering
        \includegraphics[width=\linewidth]{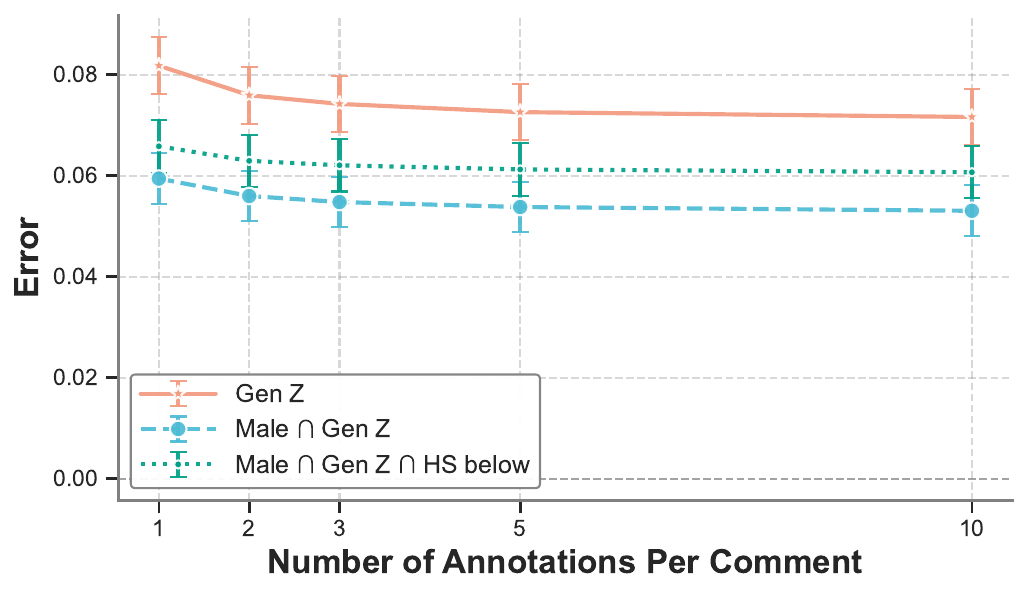}
        \caption{GPT-5.1-R=H}
    \end{subfigure}

    \vspace{0.6em}
    \begin{subfigure}[c]{0.24\textwidth}
        \centering
        \includegraphics[width=\linewidth]{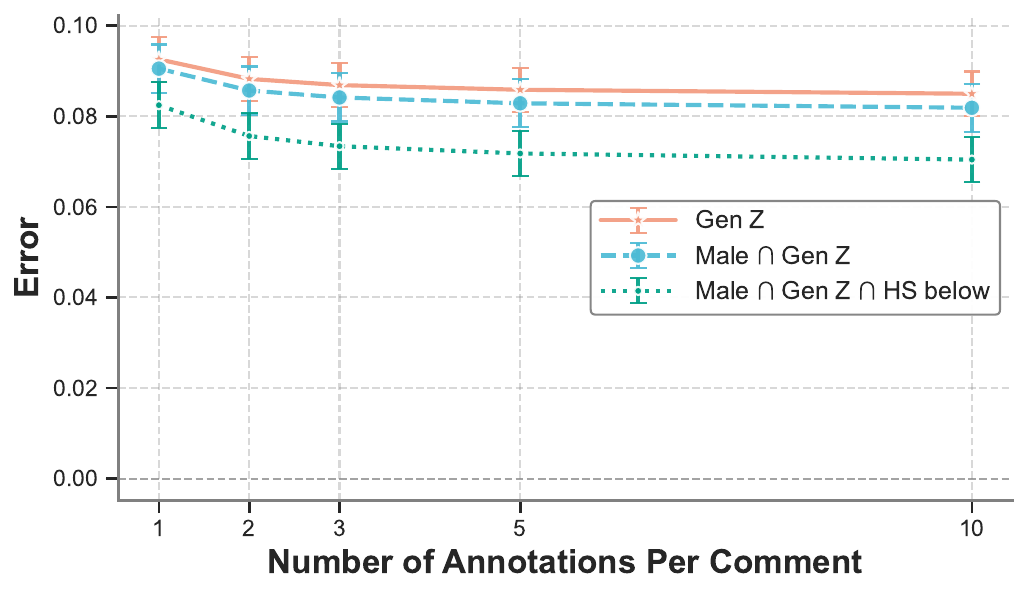}
        \caption{Gemma3:1B}
    \end{subfigure}
    \hfill
    \begin{subfigure}[c]{0.24\textwidth}
        \centering
        \includegraphics[width=\linewidth]{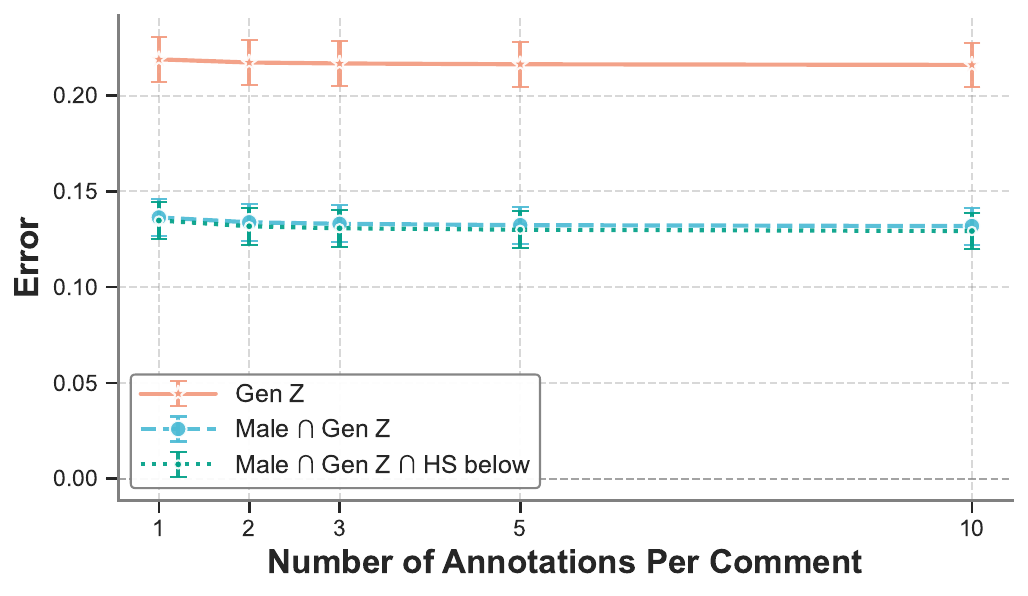}
        \caption{Gemma3:12B}
    \end{subfigure}
    \hfill
    \begin{subfigure}[c]{0.24\textwidth}
        \centering
        \includegraphics[width=\linewidth]{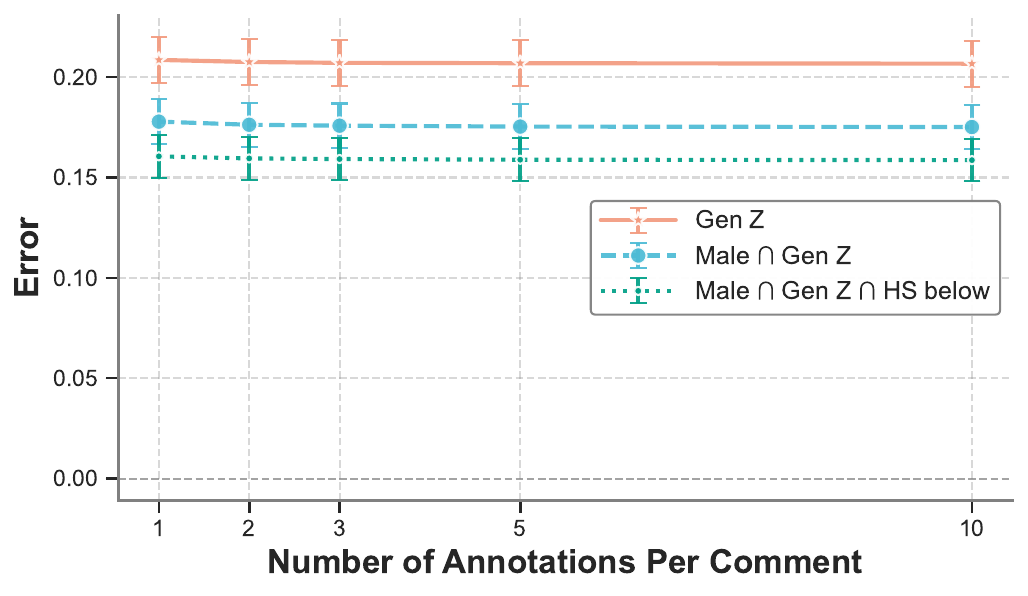}
        \caption{Gemma3:27B}
    \end{subfigure}
    \hfill
    \begin{subfigure}[c]{0.24\textwidth}
        \centering
        \includegraphics[width=\linewidth]{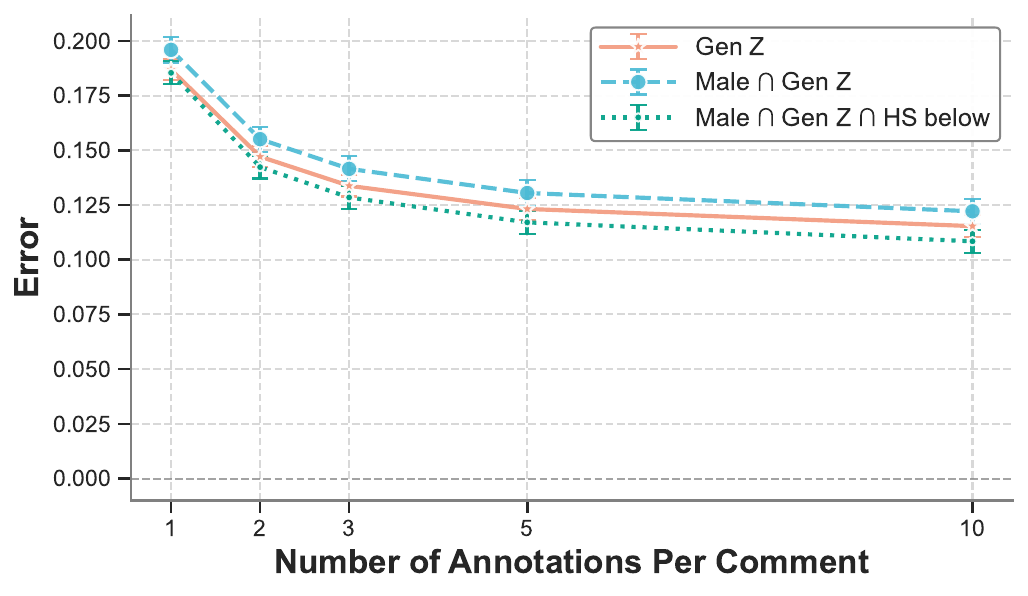}
        \caption{DeepSeek-R1:1.5B}
    \end{subfigure}

    \vspace{0.6em}
    \begin{subfigure}[c]{0.24\textwidth}
        \centering
        \includegraphics[width=\linewidth]{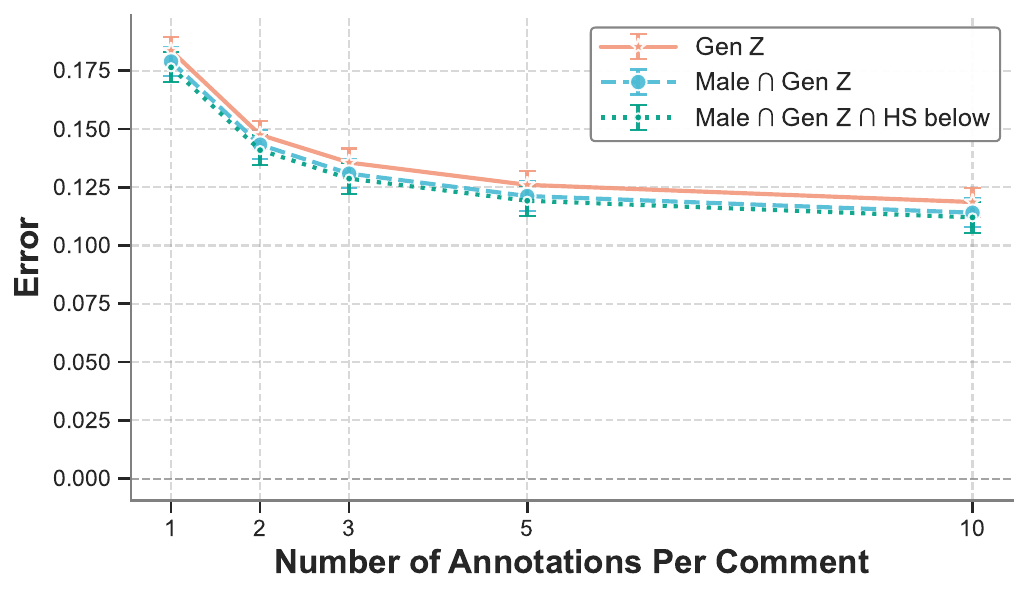}
        \caption{DeepSeek-R1:7B}
    \end{subfigure}
    \hfill
    \begin{subfigure}[c]{0.24\textwidth}
        \centering
        \includegraphics[width=\linewidth]{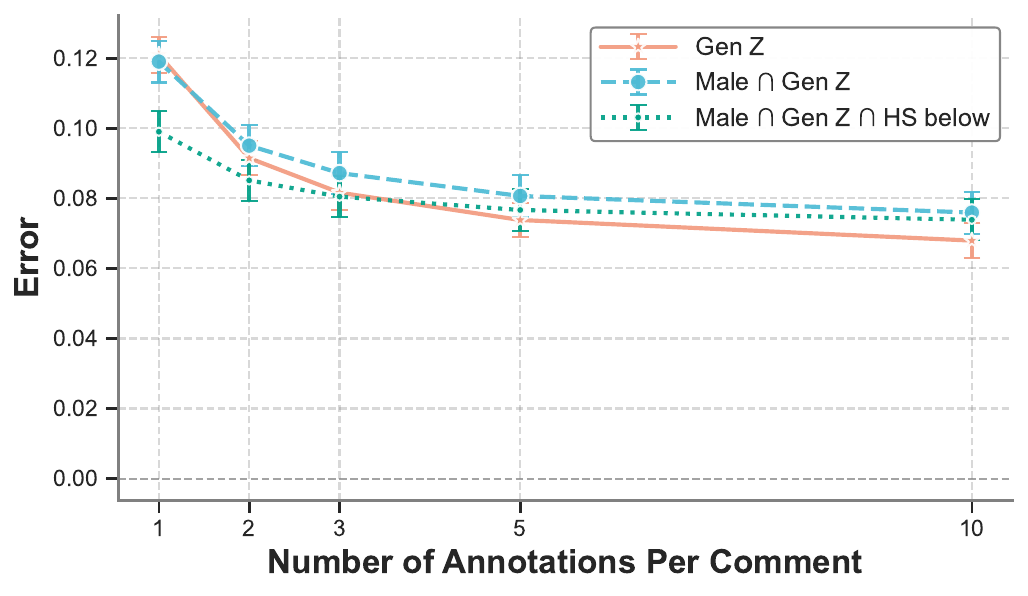}
        \caption{DeepSeek-R1:32B}
    \end{subfigure}
    \hfill
    \begin{subfigure}[c]{0.24\textwidth}
        \centering
        \includegraphics[width=\linewidth]{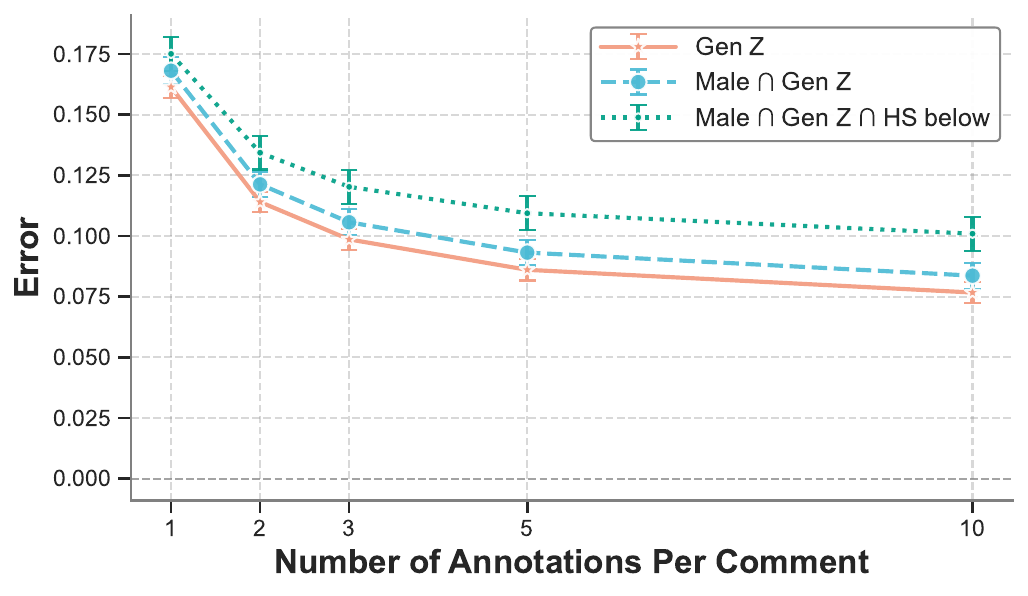}
        \caption{Qwen3:1.7B}
    \end{subfigure}
    \hfill
    \begin{subfigure}[c]{0.24\textwidth}
        \centering
        \includegraphics[width=\linewidth]{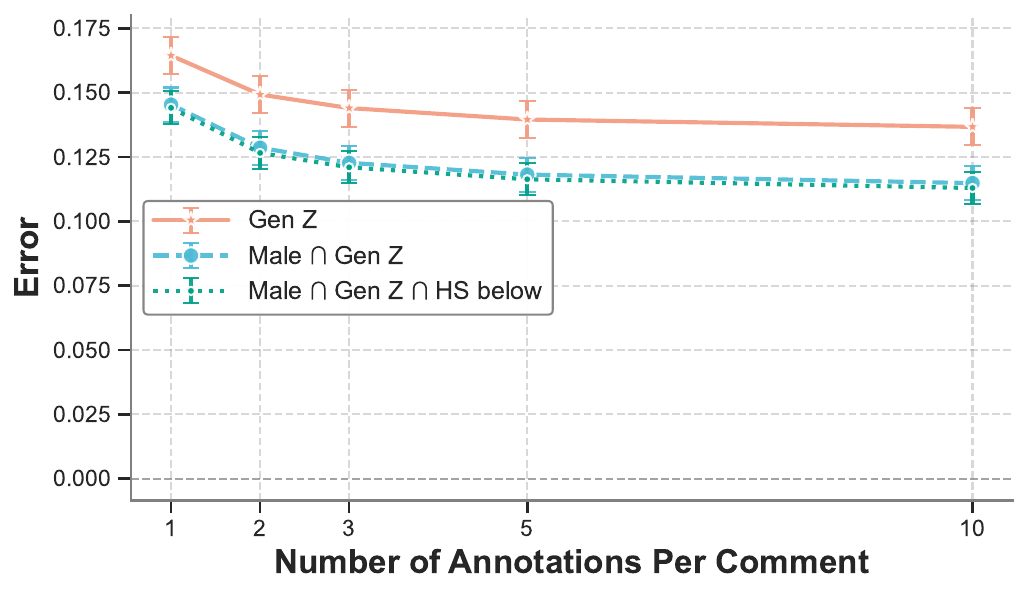}
        \caption{Qwen3:8B}
    \end{subfigure}

    \vspace{0.6em}
    \begin{subfigure}[c]{0.24\textwidth}
        \centering
        \includegraphics[width=\linewidth]{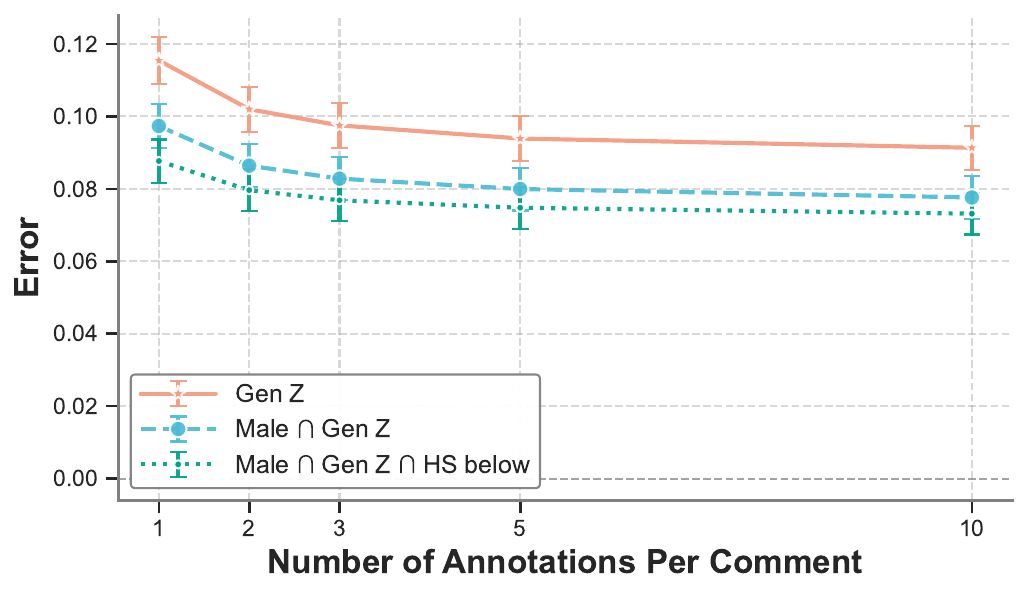}
        \caption{Qwen3:32B}
    \end{subfigure}
    \hfill
    \begin{subfigure}[c]{0.24\textwidth}
        \centering
        \includegraphics[width=\linewidth]{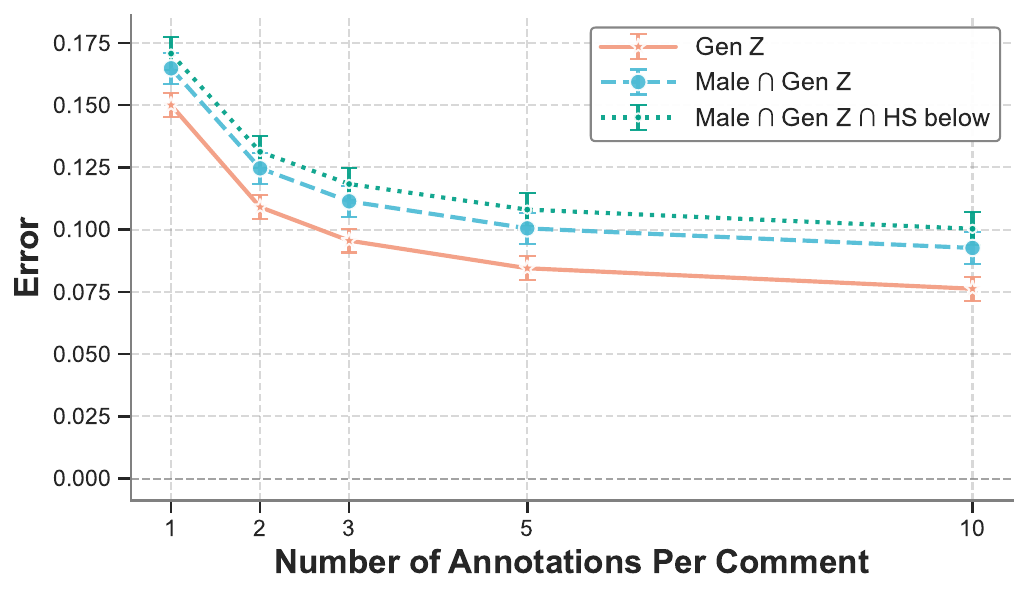}
        \caption{Qwen3-R:1.7B}
    \end{subfigure}
    \hfill
    \begin{subfigure}[c]{0.24\textwidth}
        \centering
        \includegraphics[width=\linewidth]{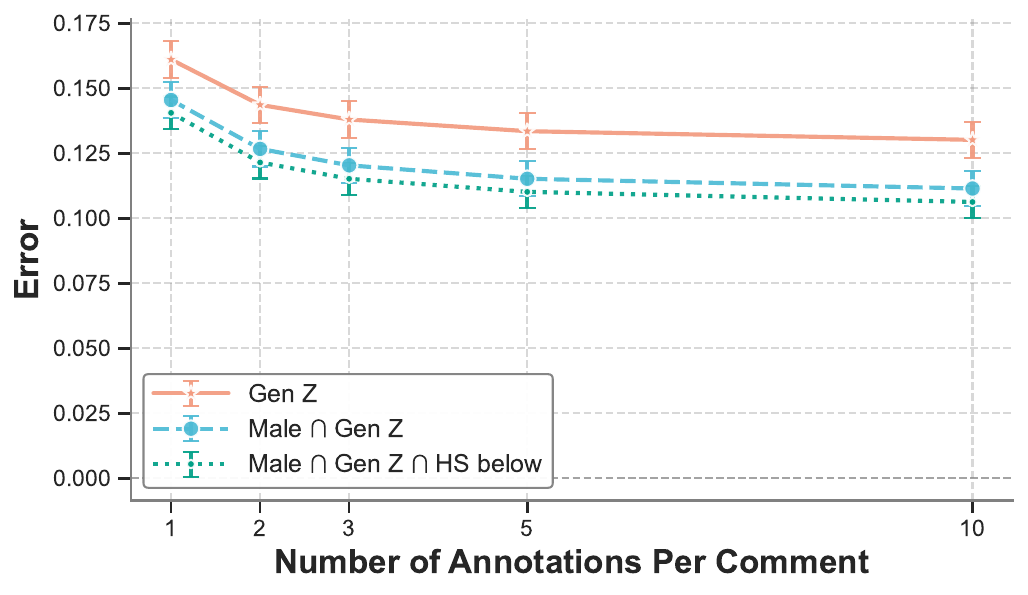}
        \caption{Qwen3-R:8B}
    \end{subfigure}
    \hfill
    \begin{subfigure}[c]{0.24\textwidth}
        \centering
        \includegraphics[width=\linewidth]{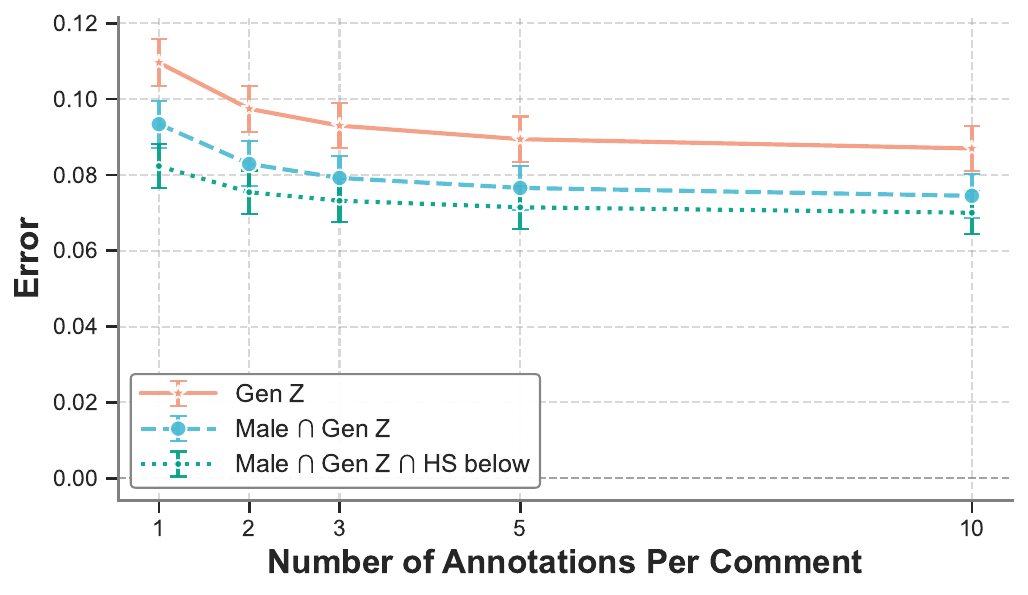}
        \caption{Qwen3-R:32B}
    \end{subfigure}
    \caption{\textbf{Specificity cascade (Gen~Z path)} on DICES.
    Target subgroups: Gen~Z people $\to$ Gen~Z men $\to$ Gen~Z men with $\le$high school education.
    Unlike the college-educated cascade (Figure~\ref{fig:dices_good_example}), several models show \textit{decreasing} error with specificity (c, e--g, l--m, o--p), indicating that demographic specificity is not always a reliable proxy for representation difficulty.}
    \label{fig:dices_bad_example}
\end{figure*}

\begin{figure*}[t]
    \begin{subfigure}[c]{0.49\textwidth}
        \centering
        \includegraphics[width=\linewidth]{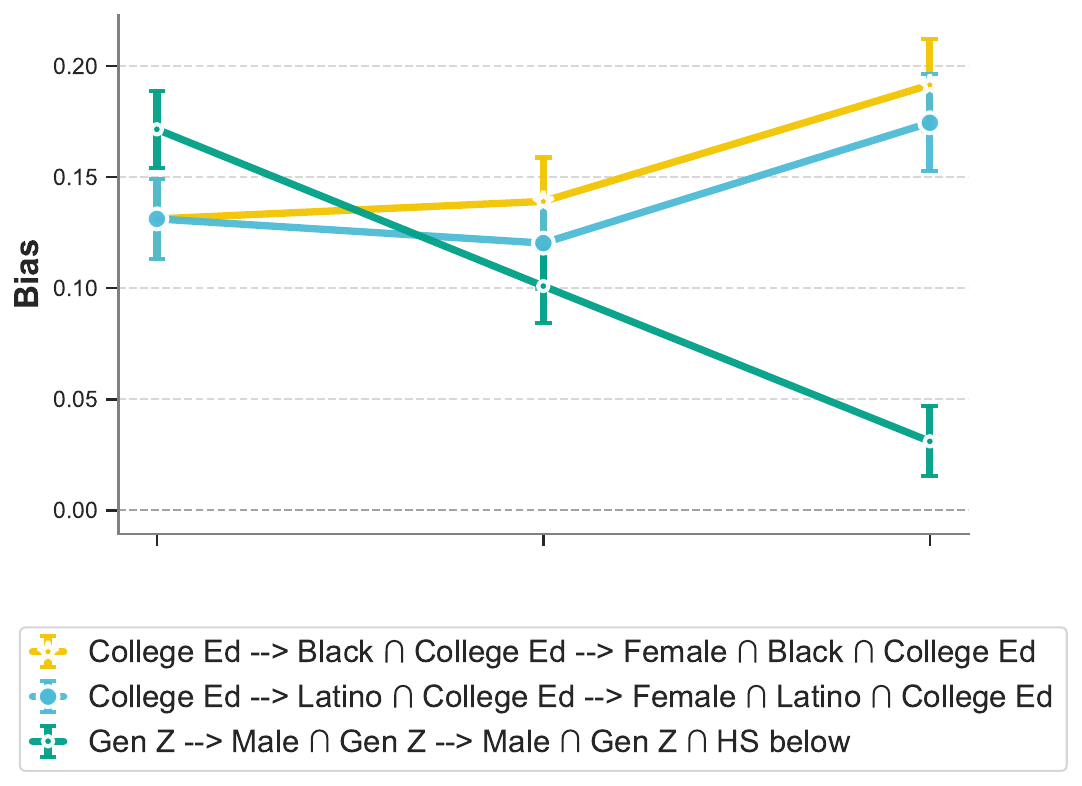}
        \caption{Bias: GPT-5.1}
    \end{subfigure}
    \hfill
    \begin{subfigure}[c]{0.49\textwidth}
        \centering
        \includegraphics[width=\linewidth]{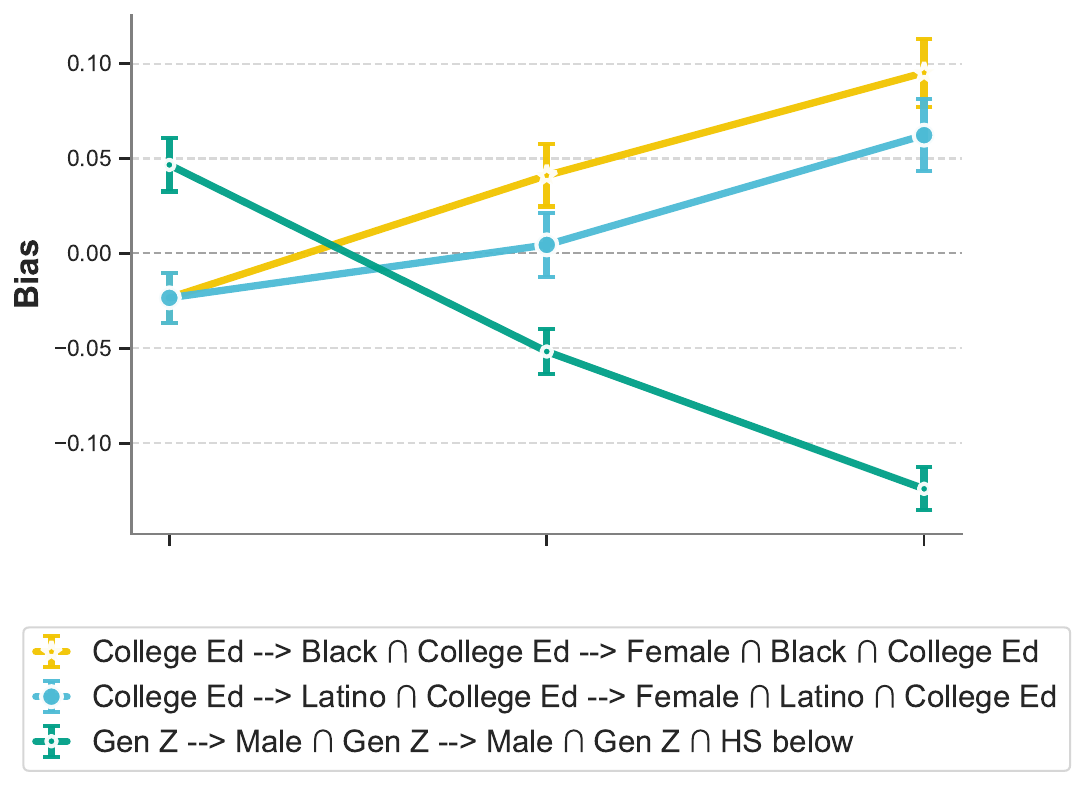}
        \caption{Bias: GPT-5.1-R=H}
    \end{subfigure}
    \vspace{0.6em}

    \begin{subfigure}[c]{0.49\textwidth}
        \centering
        \includegraphics[width=\linewidth]{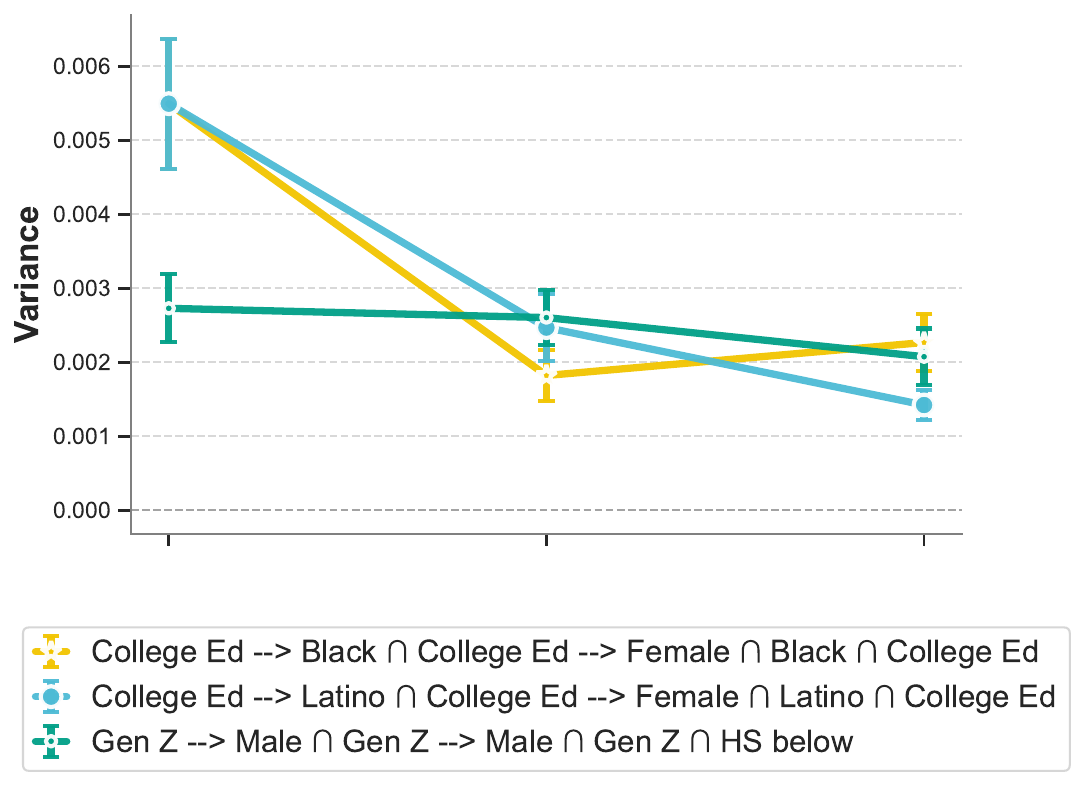}
        \caption{Variance: GPT-5.1}
    \end{subfigure}
    \hfill
    \begin{subfigure}[c]{0.49\textwidth}
        \centering
        \includegraphics[width=\linewidth]{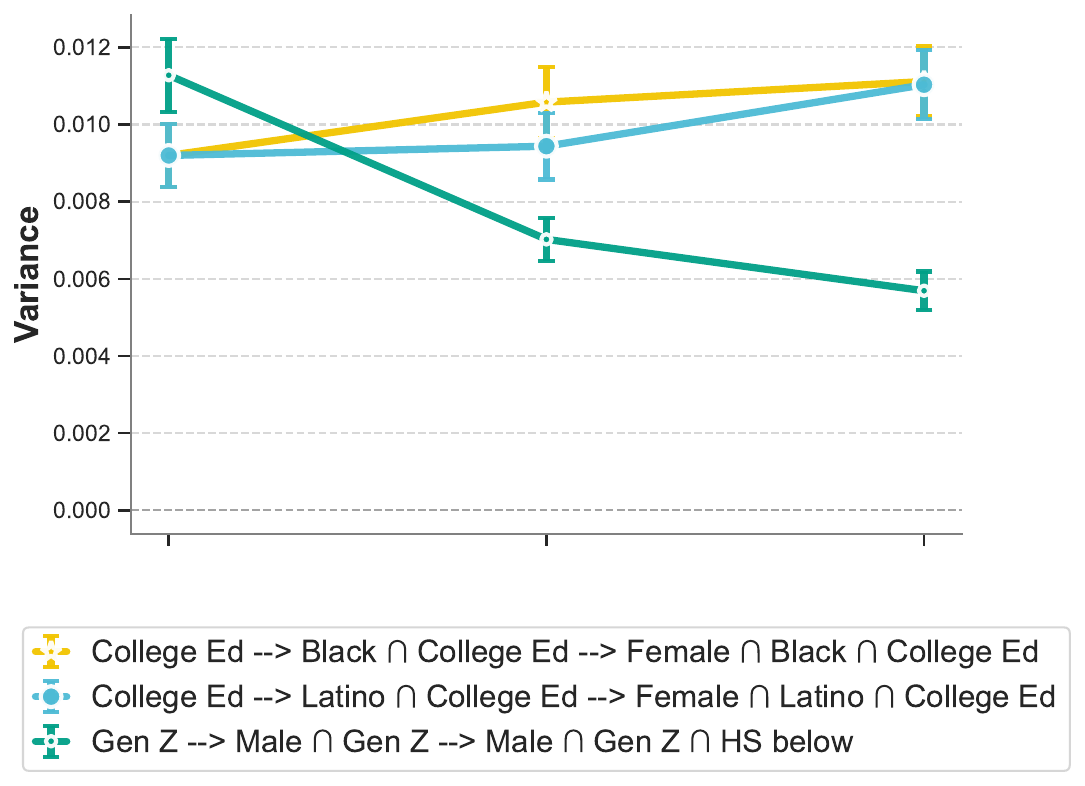}
        \caption{Variance: GPT-5.1-R=H}
    \end{subfigure}
    \caption{\textbf{Bias--variance diagnostic for specificity cascades} ($k{=}1$, GPT-5.1 $\pm$ reasoning).
    X-axis shows progressively more specific subgroups from both cascades (Figures~\ref{fig:dices_good_example}--\ref{fig:dices_bad_example}).
    For the college-educated cascade (left side of each panel), error growth is driven by increasing bias, consistent with representation mismatch.
    For the Gen~Z cascade (right side), both bias and variance \textit{decrease}, suggesting the more specific group is better represented in training data.}
    \label{fig:dices_bias_var}
\end{figure*}

\begin{figure*}[t]
    \centering
    \begin{subfigure}[c]{0.48\textwidth}
        \centering
        \includegraphics[width=\linewidth]{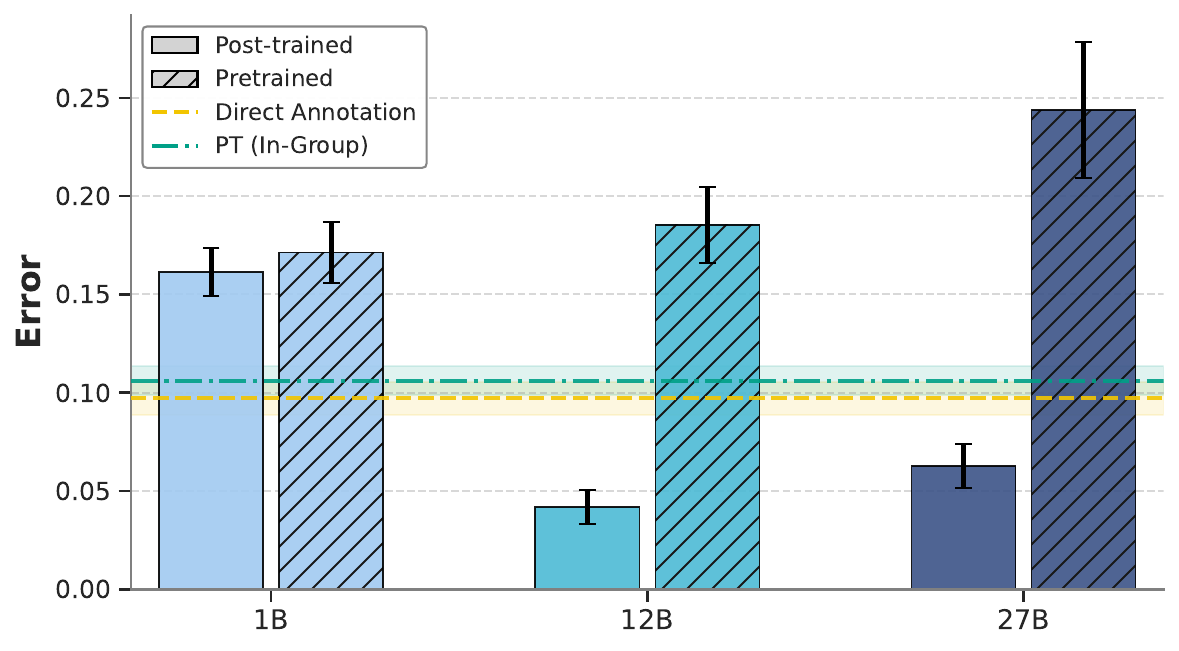}
    \end{subfigure}
    \hfill
    \begin{subfigure}[c]{0.48\textwidth}
        \centering
        \includegraphics[width=\linewidth]{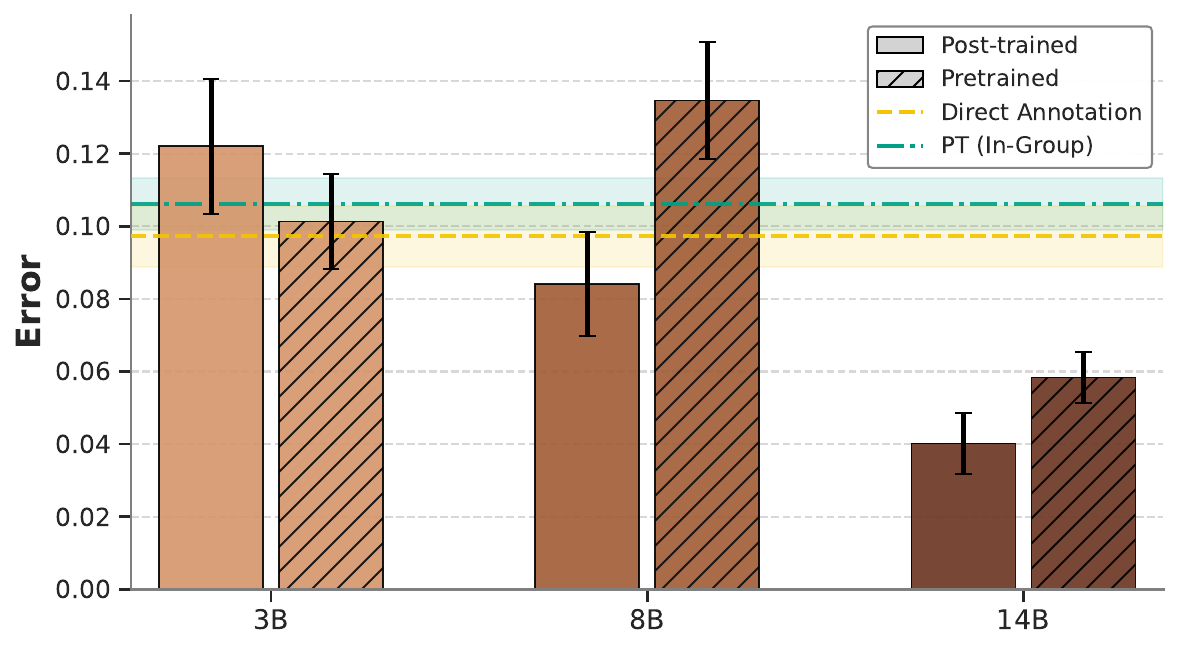}
    \end{subfigure}

    \vspace{0.4em}

    \begin{subfigure}[c]{0.48\textwidth}
        \centering
        \includegraphics[width=\linewidth]{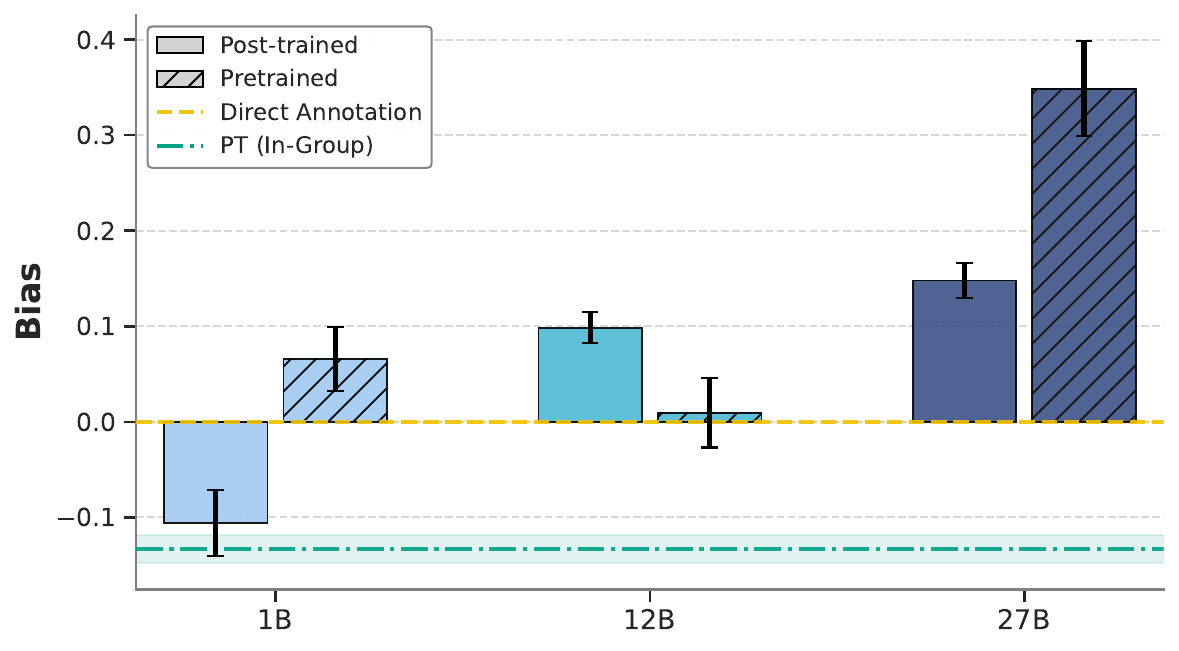}
    \end{subfigure}
    \hfill
    \begin{subfigure}[c]{0.48\textwidth}
        \centering
        \includegraphics[width=\linewidth]{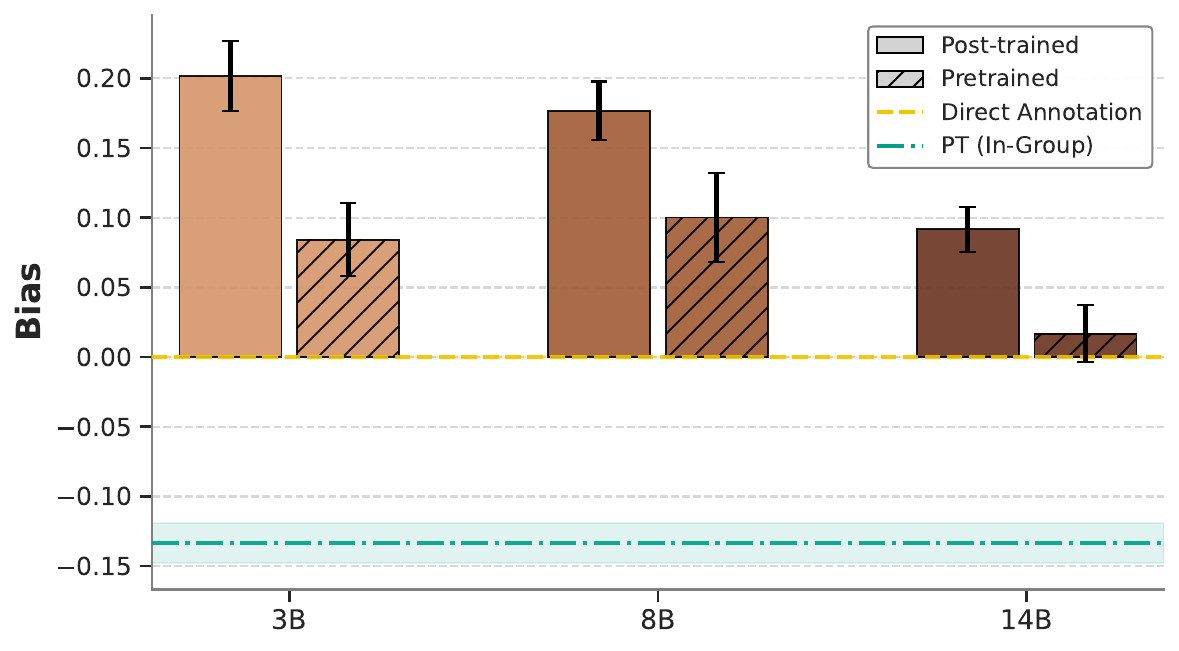}
    \end{subfigure}

    \vspace{0.4em}

    \begin{subfigure}[c]{0.48\textwidth}
        \centering
        \includegraphics[width=\linewidth]{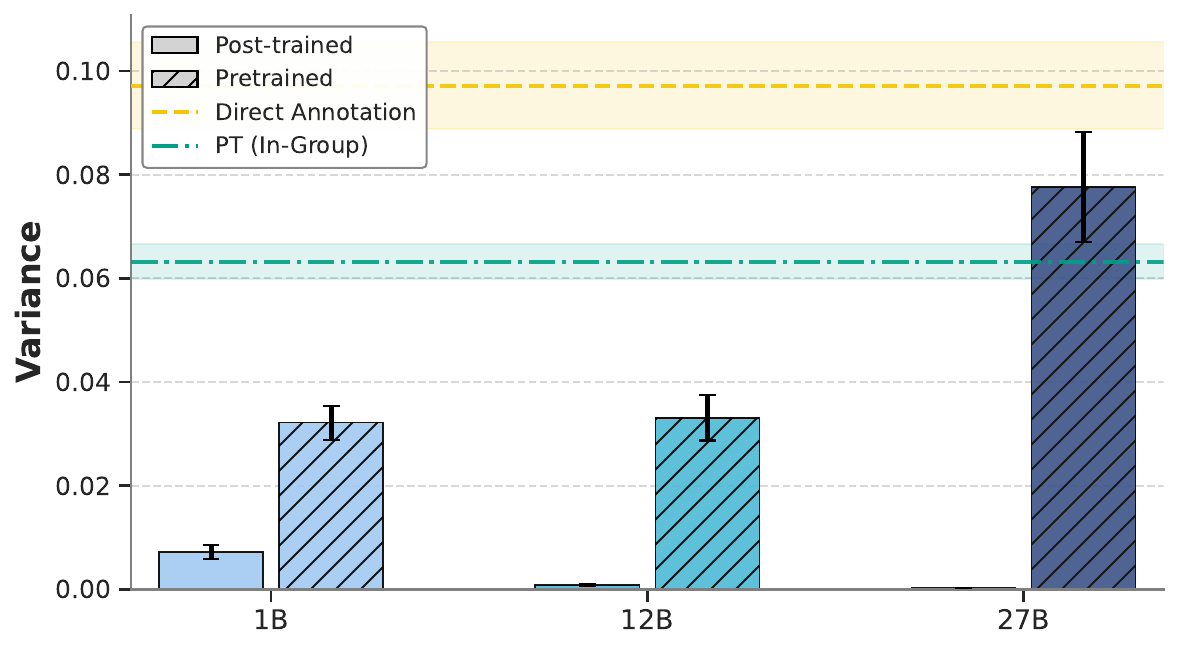}
        \caption{Gemma 3}
    \end{subfigure}
    \hfill
    \begin{subfigure}[c]{0.48\textwidth}
        \centering
        \includegraphics[width=\linewidth]{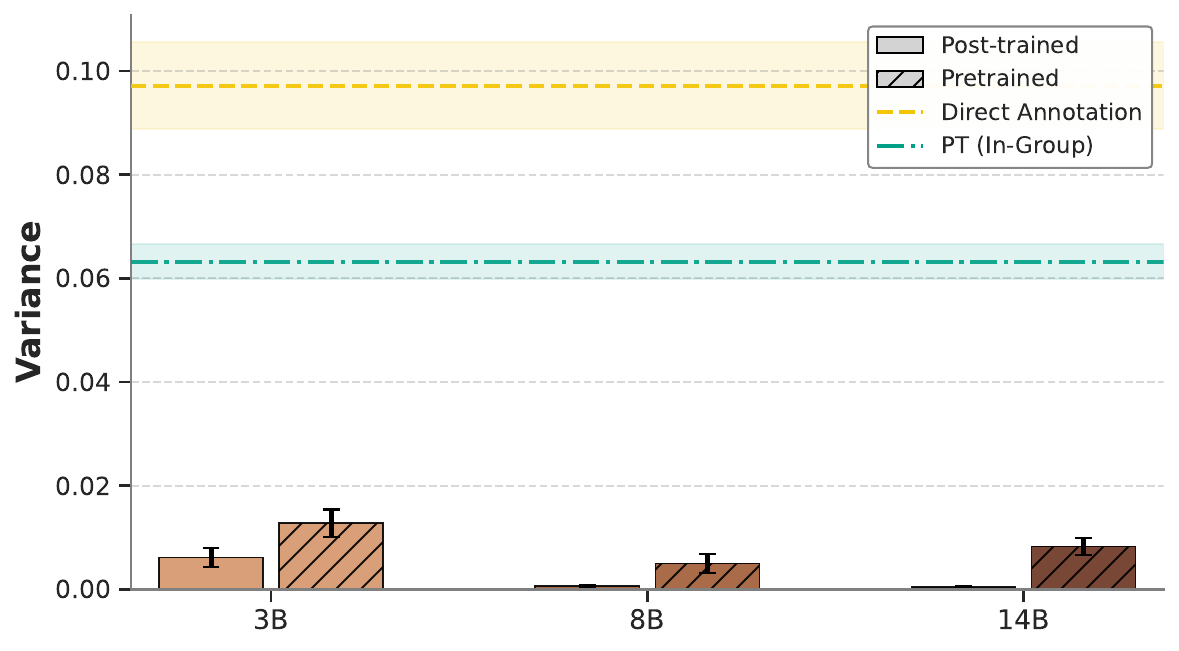}
        \caption{Ministral 3}
    \end{subfigure}

    \caption{\textbf{Pretrained vs.\ post-trained models} ($k{=}1$, female subgroup).
    Post-trained models exhibit dramatically lower variance, while pretrained models show lower absolute bias at matched sizes.
    }
    \label{fig:pt_it_panel}
\end{figure*}

\section{Differential Perspective-Taking}
\label{sec:differential_pt}

A natural concern with evaluating $\hat f(x,g)$ against $f^\ast(x,g)$ is that strong global priors may mask a lack of genuine group sensitivity.
An estimator may achieve low MSE by learning reasonable base rates without meaningfully differentiating between groups on a per-item basis.
To probe whether an estimator truly conditions on group identity, we introduce a \textit{differential perspective-taking} (DPT) diagnostic that isolates group-specific sensitivity from generic annotation skill.

\subsection{Formulation}
\label{sec:dpt_formulation}

For two groups $g_1,g_2\in\mathcal{G}$ and an item $x$, define the \textit{ground-truth disagreement}
$\Delta^\ast(x; g_1,g_2) \triangleq f^\ast(x,g_1)-f^\ast(x,g_2)$
and the corresponding \textit{estimated disagreement}
$\hat\Delta_A(x; g_1,g_2) \triangleq \hat f_A(x,g_1)-\hat f_A(x,g_2)$,
where $A\in\{H,L\}$ indexes human or LLM-based PT.
Unlike absolute estimates, $\Delta^\ast$ isolates \textit{relative subgroup movement}: an estimator relying on a group-invariant prior yields $\hat\Delta \approx 0$ regardless of true disagreement.

We quantify alignment using Pearson correlation $\rho$ between $\Delta^\ast$ and $\hat\Delta$ across items (capturing direction and relative magnitude) and \textit{directional accuracy} (the fraction of items where $\mathrm{sign}(\hat\Delta)=\mathrm{sign}(\Delta^\ast)$).
Strong alignment is evidence that the estimator possesses a Wide Lens capable of modeling relative subgroup structure; systematic attenuation of $\hat\Delta$ toward zero indicates representational collapse across groups.
Bootstrap 95\% CIs are computed with 2{,}000 resamples, and Fisher $z$-tests compare human vs.\ LLM correlations.

\subsection{Empirical Results}
\label{sec:dpt_results}

We evaluate DPT on all three group pairs in the toxicity detection dataset (Female$\leftrightarrow$Male, Female$\leftrightarrow$Non-binary, Male$\leftrightarrow$Non-binary), comparing 19 LLMs against both in-group and out-group human PT.
Figure~\ref{fig:dpt_scatter_panel} shows scatter plots of $\hat\Delta$ vs.\ $\Delta^\ast$ for representative estimators; Figure~\ref{fig:dpt_scaling} shows DPT ability as a function of model scale.

\paragraph{Female vs.\ Male: Humans and LLMs are comparable.}
This pair exhibits the strongest DPT signal.
Human PT (In-Group) achieves $\rho = 0.265$ [0.093, 0.428] with directional accuracy 64.9\%.
The best LLMs---GPT-5.1-R=M ($\rho = 0.343$) and Qwen3-R:32B ($\rho = 0.313$)---exceed Human PT (In-Group) in correlation, but no difference reaches significance (Fisher $z$-tests, all $p > 0.05$).
Interestingly, Human PT (\textit{Out-Group}) achieves the highest overall correlation ($\rho = 0.353$ [0.204, 0.492]).

\paragraph{\textit{Male vs.\ Non-binary: A genuine human advantage.}}
Human PT (In-Group) achieves $\rho = 0.312$ [0.133, 0.487] with DA = 67.7\%.
Crucially, humans \textit{significantly outperform} multiple LLMs: DeepSeek-R1:32B ($\rho = 0.053$, $p = 0.020$), GPT-5.1-R=H ($\rho = 0.075$, $p = 0.030$), and Qwen3-R:32B ($\rho = 0.071$, $p = 0.028$).
Only GPT-5.1 at low-to-moderate reasoning effort achieves comparable correlations.
This result exposes a concrete boundary condition: non-binary perspectives are lower-prevalence in training data, and the Male$\leftrightarrow$Non-binary contrast isolates a representation gap that most LLMs cannot bridge.

\paragraph{Female vs.\ Non-binary: A null pair.}
Neither humans ($\rho = 0.078$) nor any LLM ($|\rho| < 0.25$, all CIs spanning zero) achieves meaningful DPT.
The ground-truth differentials have very low variance ($\sigma(\Delta^\ast) \approx 0.032$), indicating that these groups largely agree---validating DPT as a genuine diagnostic: when groups perceive items similarly, no estimator can or should differentiate them.

\begin{figure*}[t]
    \centering
    \includegraphics[width=\linewidth]{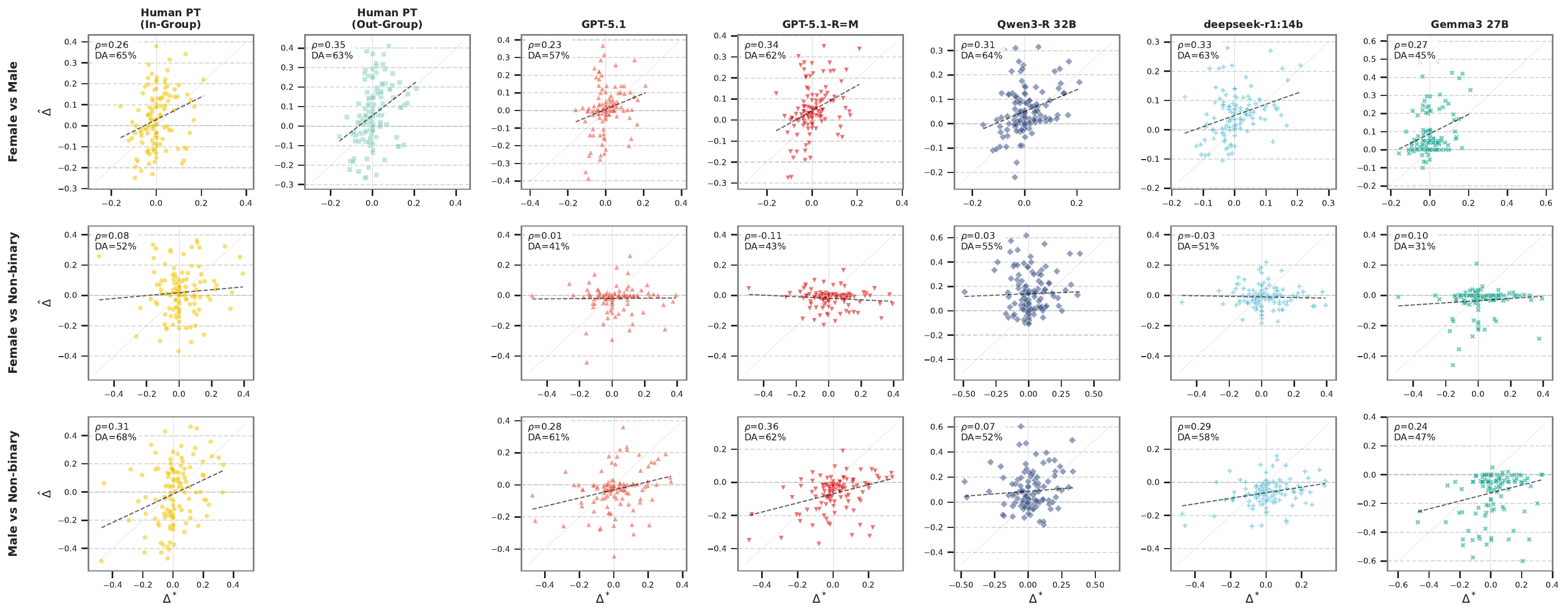}
    \caption{\textbf{Differential PT scatter plots}: per-item estimated disagreement $\hat\Delta$ vs.\ ground-truth disagreement $\Delta^\ast$ for representative estimators (columns) across the three gender-group pairs (rows). Dashed lines are OLS fits.}
    \label{fig:dpt_scatter_panel}
\end{figure*}

\begin{figure}[t]
    \centering
    \includegraphics[width=\linewidth]{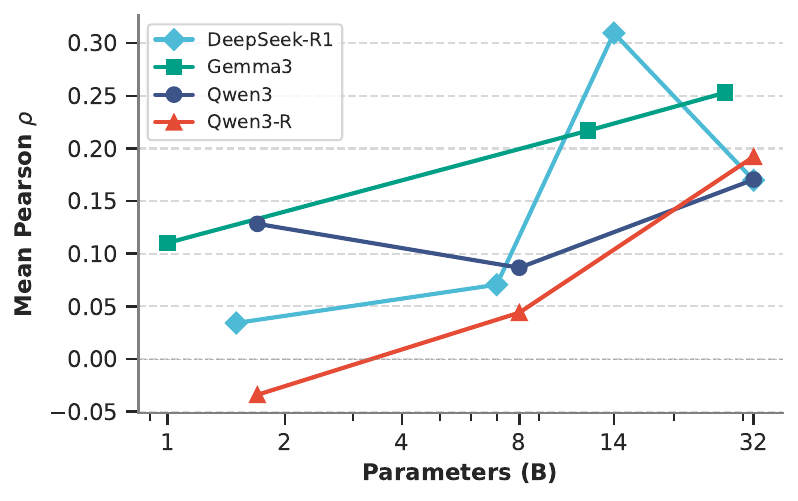}
    \caption{\textbf{DPT ability vs.\ model scale}: mean Pearson $\rho$ across the two informative group pairs (F$\leftrightarrow$M, M$\leftrightarrow$NB).}
    \label{fig:dpt_scaling}
\end{figure}

\paragraph{Scaling and reasoning.}
DPT ability emerges at approximately 8--14B parameters: models below 2B universally fail (mean $\rho \approx 0$; Figure~\ref{fig:dpt_scaling}).
However, scaling is non-monotonic. For example, DeepSeek-R1 peaks at 14B ($\rho = 0.31$) then regresses at 32B ($\rho = 0.17$), echoing the non-linearities observed in standard PT.
Reasoning shows diminishing and eventually negative returns: GPT-5.1 achieves its best DPT at moderate effort (R=M, $\rho = 0.343$ on F$\leftrightarrow$M) but degrades at high effort (R=H, $\rho = 0.199$), consistent with the criterion-drift mechanism (Appendix~\ref{app:reasoning_traces}).

\paragraph{Key takeaway.}
DPT complements standard PT evaluation by isolating group-specific sensitivity.
While LLMs generally excel at absolute estimation of $f^\ast(x,g)$, differential sensitivity to per-item group differences remains a domain where human PT retains a genuine advantage, particularly for lower-prevalence groups.
This provides a concrete boundary condition separating regimes where LLMs serve as frontline estimators from those where human judgment remains essential.

\end{document}